\documentclass{article}
\PassOptionsToPackage{numbers,sort&compress}{natbib}
\usepackage[main,final]{neurips_2026}
\renewcommand{\acksection}{\section*{Acknowledgements}}

\usepackage{graphicx}
\usepackage{subfigure}
\usepackage{multirow}
\usepackage{booktabs} 
\usepackage{setspace}
\usepackage{longtable}
\usepackage{array}
\usepackage{tabularx}
\usepackage{lipsum} 
\usepackage{wrapfig} 
\usepackage{booktabs} 
\usepackage{bbm}
\usepackage{amsmath}
\usepackage{amssymb}
\usepackage{mathtools}
\usepackage{amsthm}
\usepackage{makecell}
\usepackage{color}
\usepackage{xcolor}         
\usepackage{colortbl}
\usepackage{amsmath}
\usepackage{amssymb}
\usepackage{mathtools}
\usepackage{amsthm}
\usepackage{algorithm}
\usepackage{algorithmic}
\usepackage{amsfonts}
\usepackage[utf8]{inputenc} 
\usepackage[T1]{fontenc}    
\usepackage[hidelinks]{hyperref}
\definecolor{citeColor}{RGB}{0,20,115}
\hypersetup{colorlinks,linkcolor={citeColor},citecolor={citeColor},urlcolor={citeColor}}

\usepackage{url}            
\usepackage{booktabs}       
\usepackage{amsfonts}       
\usepackage{nicefrac}       
\usepackage{microtype}      

\usepackage{amsmath}
\usepackage{array}
\usepackage{amsthm}
\usepackage{amsfonts}
\usepackage{enumerate}
\usepackage{subfigure}
\usepackage{bm}
\usepackage{tabularx}
\usepackage{enumitem}
\usepackage{multirow}
\usepackage{nicefrac}
\usepackage{booktabs}
\usepackage{bbding}
\usepackage{caption}
\usepackage{graphicx}
\usepackage{bbding}
\usepackage{wrapfig}
\usepackage{lipsum}
\usepackage{amssymb}
\usepackage{multibib}
\usepackage[most]{tcolorbox}
\usepackage[normalem]{ulem}
\usepackage{pifont}
\usepackage{varwidth}
\usepackage{tikz}
\usetikzlibrary{calc}
\usetikzlibrary{positioning, shapes, fit, backgrounds, decorations.pathreplacing}

\usepackage[capitalize,noabbrev]{cleveref}
\usepackage[disable,textsize=tiny]{todonotes}

\usepackage{etoc}
\etocdepthtag.toc{mtchapter}
\etocsettagdepth{mtchapter}{subsection}
\etocsettagdepth{mtappendix}{none}

\theoremstyle{plain}
\newtheorem{theorem}{Theorem}[section]
\newtheorem{proposition}[theorem]{Proposition}

\newtheorem{corollary}[theorem]{Corollary}
\theoremstyle{definition}

\theoremstyle{remark}
\newtheorem{remark}[theorem]{Remark}

\definecolor{LightGray}{HTML}{F0F1F2} 

\newcolumntype{L}[1]{>{\raggedright\let\newline\\\arraybackslash\hspace{0pt}}m{#1}}
\newcolumntype{C}[1]{>{\centering\let\newline  \\\arraybackslash\hspace{0pt}}m{#1}}%
\newcolumntype{R}[1]{>{\raggedleft\let\newline \\\arraybackslash\hspace{0pt}}m{#1}}

\definecolor{shadecolor}{rgb}{0.92,0.92,0.92}
\definecolor{grey}{rgb}{0.5, 0.5, 0.5}

\newcommand{\greenbox}[1]{\cellcolor[HTML]{009901}}

\definecolor{case_study_red}{RGB}{102, 0, 0}
\definecolor{regimeHelpful}{HTML}{E4EFF8}
\definecolor{regimeNoise}{HTML}{F5DDDA}
\definecolor{rowAlt}{HTML}{F6F7F5}
\definecolor{hiBad}{HTML}{B85C55}
\definecolor{hiWarn}{HTML}{D3922D}
\definecolor{cellGood}{HTML}{C5DCF2}
\definecolor{cellGoodL}{HTML}{DCEAF7}
\definecolor{cellBad}{HTML}{EFCFCB}
\definecolor{cellBadL}{HTML}{F7E4E2}

\newcommand{\cellbest}[1]{\cellcolor{cellGood}\textbf{#1}}
\newcommand{\cellsec}[1]{\cellcolor{cellGoodL}\underline{#1}}

\definecolor{rewardflowBlue}{HTML}{546AF5}
\newcommand{\mainresultcolors}{%
  \colorlet{regimeHelpful}{rewardflowBlue!10}%
  \colorlet{cellGood}{rewardflowBlue}%
  \colorlet{cellGoodL}{rewardflowBlue!35}%
  \renewcommand{\cellbest}[1]{\cellcolor{cellGood}\textcolor{white}{\textbf{##1}}}%
}
\definecolor{case_study_yellow}{RGB}{180, 180, 0} 
\definecolor{case_study_green}{RGB}{0, 102, 0}
\definecolor{case_study_cyan}{RGB}{0,180,180}

\definecolor{outcome_reward}{HTML}{173E64}
\definecolor{intermediate_reward_1}{HTML}{4F95D9}
\definecolor{intermediate_reward_2}{HTML}{DCEAF7}

\newcommand{\redbox}[1]{\cellcolor[HTML]{FE0000}}

\usepackage{listings}
\usepackage{listingsutf8}
\lstnewenvironment{longprompt}[1][]
{\lstset{basicstyle={\small},prebreak={},postbreak=\mbox{\hspace{-2.25 em}},backgroundcolor=\color{white},frame=single,framerule=1pt,#1}}
{}

\usepackage{varwidth}
\usepackage{tikz}
\usetikzlibrary{calc}
\usetikzlibrary{positioning, shapes, fit, backgrounds, decorations.pathreplacing}

\usepackage[capitalize,noabbrev]{cleveref}

\crefname{figure}{Fig.}{Figs.}
\Crefname{figure}{Fig.}{Figs.}
\crefname{table}{Tab.}{Tabs.}
\Crefname{table}{Tab.}{Tabs.}

\usepackage{etoc}
\etocdepthtag.toc{mtchapter}
\etocsettagdepth{mtchapter}{subsection}
\etocsettagdepth{mtappendix}{none}

\usepackage[disable,textsize=tiny]{todonotes}

\title{RewardFlow: Can Action-Level Reward Be\\\emph{Estimated} from State-Graph Topology in Agentic RL?}

\makeatletter
\renewcommand*{\@fnsymbol}[1]{\ensuremath{\ifcase#1\or \dagger\or \ddagger\or
		\mathsection\or \mathparagraph\or \|\or **\or \dagger\dagger
		\or \ddagger\ddagger \else\@ctrerr\fi}}
\makeatother
\author{%
  \textbf{Xiao Feng}$^{1}$ \quad
  \textbf{Bo Han}$^{1}\thanks{
		Correspondence to Bo Han (bhanml@comp.hkbu.edu.hk).}$ \quad
  \textbf{Zhanke Zhou}$^{1}$ \quad 
  \textbf{Jiaqi Fan}$^{2}$ \quad \\
  \textbf{Jiangchao Yao}$^{3}$ \quad 
  \textbf{Ka Ho Li}$^{2}$ \quad
  \textbf{Dahai Yu}$^{2}$ \quad
  \textbf{Michael K. Ng}$^{1}$\quad \\
$^{1}$Hong Kong Baptist University \quad
$^{2}$TCL Corporate Research (HK) Co., Ltd \quad \\
$^{3}$CMIC, Shanghai Jiao Tong University
\quad \\
}

\newcommand{\andrew}[1]{}
\newcommand{\zk}[1]{}

\ifdefined\ShowAuthorComments

  \renewcommand{\andrew}[1]{\textcolor{violet}{\textbf{[Andrew:} #1\textbf{]}}}
  \renewcommand{\zk}[1]{\textcolor{blue!70!black}{\textbf{[ZK:} #1\textbf{]}}}
\fi

\begin{document}

\maketitle











\begin{abstract}

In agentic tasks such as household assistance and web shopping, large language model (LLM) agents take multiple actions before receiving a final outcome indicating success or failure.
For reinforcement learning (RL), this outcome obscures individual actions' contributions: a shopping trajectory may find the right product but fail to complete the purchase.
To address this credit-assignment problem, we introduce RewardFlow, which connects multiple trajectories for the same task through \emph{shared states} (states grouped by task-relevant equivalence or semantic similarity).
The resulting directed \emph{state graph} links observed trajectory segments. Therefore, an action in a failed trajectory may reach a state from which another trajectory succeeds.
It propagates reward signals backward from states where the task succeeds, assigning \emph{state potentials} that decrease with shortest-path distance to success in the graph.
\emph{Potential differences} then reward actions that reduce this distance, bringing the agent closer to success in the graph.
These estimated rewards yield action-level advantages that complement trajectory-level advantages from outcome rewards during policy optimization, without a learned process reward model or manual action-level annotations.
Experiments on ALFWorld, WebShop, Sokoban, and search tasks show gains over RL baselines across model scales.
With 3B policies, graph construction and reward estimation take $1.47$-$2.39$ seconds per training step on ALFWorld, WebShop, and Sokoban (Fig.~\ref{fig:rewardflow_teaser}).
Code is available at \url{https://github.com/tmlr-group/RewardFlow}.

\end{abstract}

\begin{center}
\begin{minipage}{\linewidth}
    \centering
    \includegraphics[width=\linewidth]{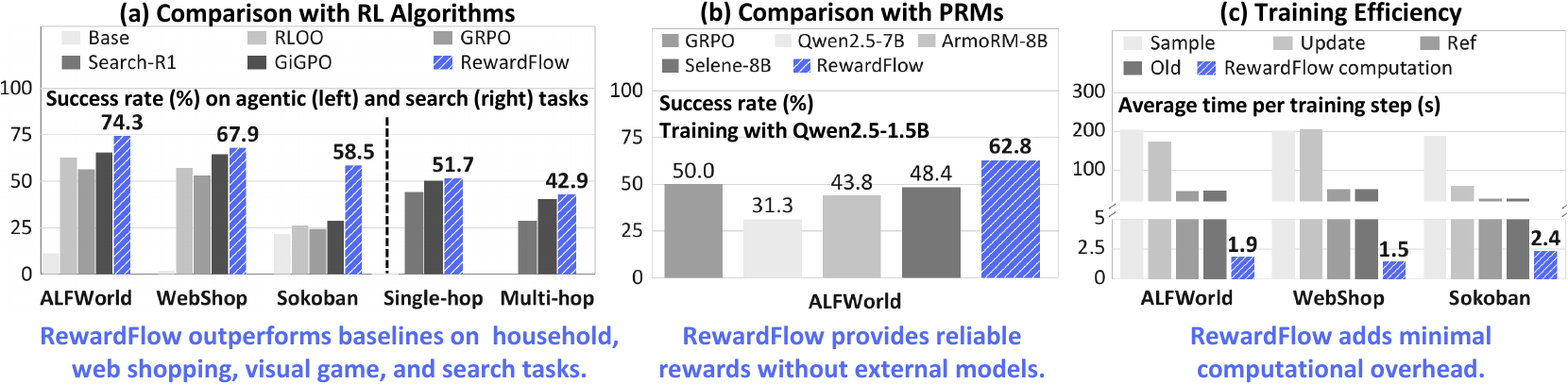}
    \captionsetup{font=small,skip=5pt,hypcap=false}
    \captionof{figure}{\textbf{RewardFlow improves agent learning.}
    \textbf{(a)} Averaged across evaluated model scales, \mbox{RewardFlow} achieves higher success rates on ALFWorld, WebShop, and Sokoban and higher search accuracy than the RL baselines.
    \textbf{(b)} On ALFWorld, Qwen2.5-1.5B trained with RewardFlow achieves a $62.8\%$ success rate, compared with $50.0\%$ for GRPO and $48.4\%$ for GRPO with Selene-8B, the strongest evaluated PRM baseline.
    \textbf{(c)} Graph construction and reward estimation take $1.47$-$2.39$ s per training step, using Qwen2.5-3B on ALFWorld and WebShop and Qwen2.5-VL-3B on Sokoban.}
    \label{fig:rewardflow_teaser}
\end{minipage}
\end{center}

\begin{figure*}[t!]
    \centering
    \includegraphics[width=\linewidth]{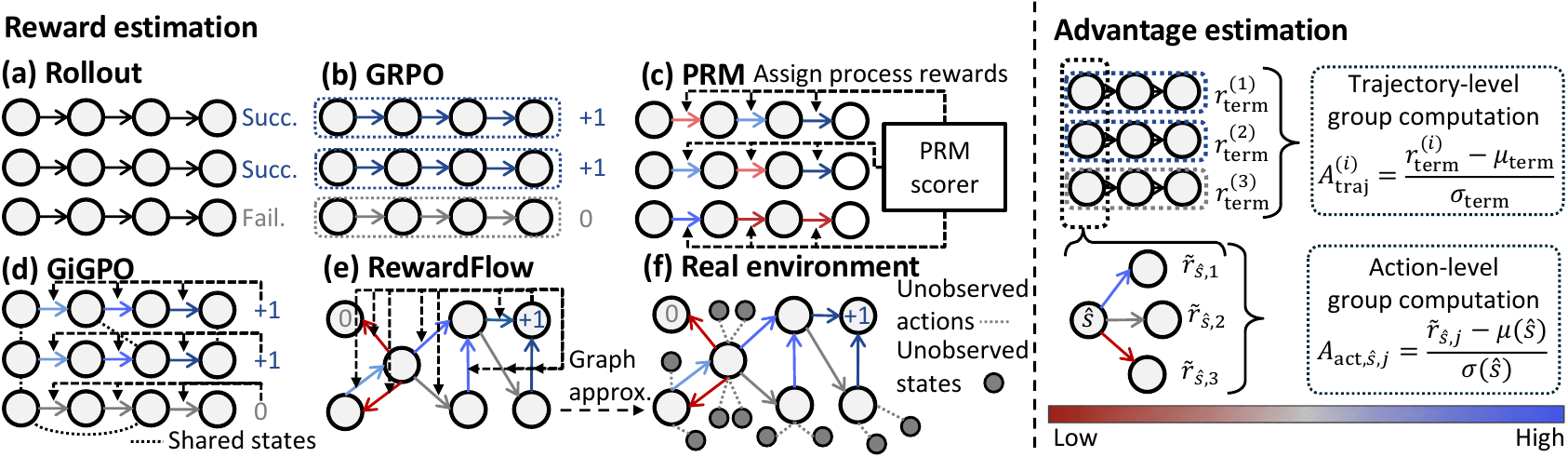}
    \caption{\textbf{From trajectory outcomes to graph-derived action rewards.} Given sampled trajectories (a), GRPO (b) uses terminal rewards, PRMs (c) score actions, and GiGPO (d) compares trajectory-local returns at shared states. RewardFlow (e) connects observed transitions into a graph that partially captures the environment's transition structure (f). Shared states link actions to successful continuations from other trajectories. Graph distances to success define potentials, whose differences yield action rewards. Right: Action and terminal rewards yield action- and trajectory-level advantages. The index $j$ identifies action occurrences at graph node $\hat s$.
    }
    \label{fig:rewardflow_overview}
    \vspace{-12pt}
\end{figure*}

\section{Introduction}
\label{sec:introduction}
In agentic tasks such as household assistance~\citep{shridhar2021alfworld}, web shopping~\citep{yao2022webshop}, and tool-assisted reasoning~\citep{jiang2025verltool,zhou2025alphaapollo}, large language model (LLM) agents interact with environments through sequences of actions.
Reinforcement learning (RL) often trains these agents using a reward available only when a trajectory ends, indicating success or failure.
Yet an action's contribution to task progress can differ from the trajectory's final outcome: a shopping agent may find the right product but later select the wrong option and fail to complete the requested purchase.
Conversely, a successful trajectory may include avoidable detours.
This mismatch creates an action-level credit-assignment problem: outcome rewards alone provide limited guidance on which actions to reinforce or discourage.

One approach to this credit-assignment problem is to use process reward models (PRMs) to score individual actions~\citep{lightman2023let,wang2025spa}.
However, a fixed scorer can become unreliable as the policy changes~\citep{fluri2025perils} and be exploited through reward hacking~\citep{cheng2025stop}.
An alternative is to derive feedback directly from sampled interactions and terminal rewards.
Group-in-Group Policy Optimization (GiGPO)~\citep{feng2025group} compares actions sampled at matching states across trajectories.
However, each action is scored using discounted rewards from that action onward in the same trajectory.
Later mistakes may therefore lower the score of an earlier action that made progress.
This raises a question: \emph{Can action-level rewards be estimated from the state-transition structure revealed collectively by sampled trajectories?}

To answer this question, we connect observed transitions across trajectories through \emph{shared states}.
These states represent task-relevant situations identified as equivalent or semantically similar using observations and relevant interaction history.
Such states recur in trajectories sampled for the same task (Fig.~\ref{fig: total_vs_unique}, left), allowing their transitions to form a directed \emph{state graph}.
In the shopping example, two trajectories may reach the same product page with the same selected options. From that state, one later fails while the other completes the purchase.
The graph connects the first trajectory's route to this shared state with the second's successful continuation, using only observed transitions.
The graph's \emph{topology} therefore encodes reachability to observed successful states and the minimum number of transitions required to reach them.

Building on this insight, we introduce RewardFlow, which estimates \emph{action-level rewards from state-graph topology} for agentic RL (Fig.~\ref{fig:rewardflow_overview}).
For each task, RewardFlow constructs the graph from the current group of sampled trajectories and marks successful terminal nodes using outcome rewards.
Reward signals propagate backward through the graph to assign \emph{state potentials} that decrease with shortest-path distance to observed success.
\emph{Potential differences} then reward actions that bring the agent closer to success in the graph and penalize moves farther away.
Comparing these rewards at shared states gives action-level advantages that complement outcome-based trajectory-level advantages during policy optimization, without a learned PRM or manual action-level annotations.

Empirically, we evaluate RewardFlow across four task families: household tasks (ALFWorld), web shopping (WebShop), visual puzzles (Sokoban), and search tasks.
Averaged across evaluated model scales, it achieves absolute success-rate improvements over GiGPO of $6.2\%$ for ALFWorld and WebShop combined and $29.7\%$ on Sokoban.
It also achieves higher average search accuracy than GiGPO, with both methods using the same encoder for state grouping.
A separate ALFWorld-$3$B control supports the contribution of the reward estimator: RewardFlow achieves an absolute success-rate improvement of $4.2\%$ over GiGPO after state preprocessing and invalid-action filtering (Sec.~\ref{ssec:ablations}).
It also outperforms the evaluated PRM-based baselines on ALFWorld. Graph processing takes $1.47$-$2.39$ seconds per training step with $3$B policies on ALFWorld, WebShop, and Sokoban (Fig.~\ref{fig:rewardflow_teaser}).

Our contributions are:
\begin{itemize}[leftmargin=*, itemsep=4pt, parsep=0pt, topsep=0pt]
\item \textbf{State graphs for cross-trajectory analysis.} We connect observed transitions through shared states to construct directed graphs from sampled trajectories. These graphs reveal paths to success across trajectories and characterize each state's reachability and distance to observed success (Sec.~\ref{sec: graph_modeling}).
\item \textbf{Topology-aware action-level reward estimation.} We derive action rewards from potential differences for policy optimization. We analyze approximation error under incomplete graph coverage and conditions for cumulative raw rewards to favor successful trajectories (Sec.~\ref{sec: method}).
\item \textbf{Controlled empirical evaluation.} Experiments show gains over RL baselines across four task families and multiple model scales. A matched preprocessing control and advantage ablations support the contribution of graph-derived rewards to policy learning (Sec.~\ref{sec: experiments}).
\end{itemize}

\section{Preliminaries}
\label{sec: preliminary}

\paragraph{Problem formulation.}
Consider an agentic setting where an LLM agent interacts with an environment to solve a task $x\sim P(X)$. We model these environments as partially observable Markov decision processes (POMDPs)~\citep{kaelbling1998planning}. Since the current observation may not fully reveal the environment state, we use a state representation $s_t\in\mathcal S$ that combines the task description, current observation, and relevant interaction history. At turn $t$, the policy $\pi_\theta$ generates a token sequence $o_t$ conditioned on $s_t$, with $o_{t,k}$ denoting its $k$-th token. An action $a_t\in\mathcal A$ is parsed from $o_t$ and executed in the environment. The environment returns feedback used to update the representation to $s_{t+1}$. Repeating this interaction yields a trajectory $\tau=(s_0,a_0,\ldots,a_{T-1},s_T)$, where $T$ is the number of action turns. In many agentic tasks, rewards are available only when the outcome is known: terminal rewards indicate success or failure, while intermediate actions receive zero task reward.

\paragraph{Group-based RL.}
For each task, the policy samples a group of $G$ trajectories $\{\tau^{(i)}\}_{i=1}^G$, where trajectory $i$ has $T_i$ action turns and a binary terminal outcome $r_{\mathrm{term}}^{(i)}\in\{0,1\}$. Group-based methods such as GRPO~\citep{shao2024deepseekmath} and RLOO~\citep{ahmadian2024back,kool2019buy} use these outcomes to compute trajectory-level advantages for policy optimization. An outcome-based advantage is shared by all generated tokens in a trajectory, without distinguishing useful, unnecessary, and counterproductive actions within it. GiGPO~\citep{feng2025group} further groups actions at shared states and computes action-level advantages from their trajectory-local returns. However, trajectory-local returns can undervalue useful actions in failed trajectories and overvalue unnecessary or counterproductive actions in successful ones. Therefore, our goal is to estimate action-level rewards $\tilde r_t^{(i)}$ from the state-transition structure revealed across sampled trajectories. Together with trajectory-level terminal rewards, they yield complementary advantages to improve agentic RL without action-level annotations or a learned reward model.

\section{State Graphs for Cross-Trajectory Analysis}
\label{sec: graph_modeling}

\begin{figure}[t]
    \centering
    \includegraphics[width=\linewidth]{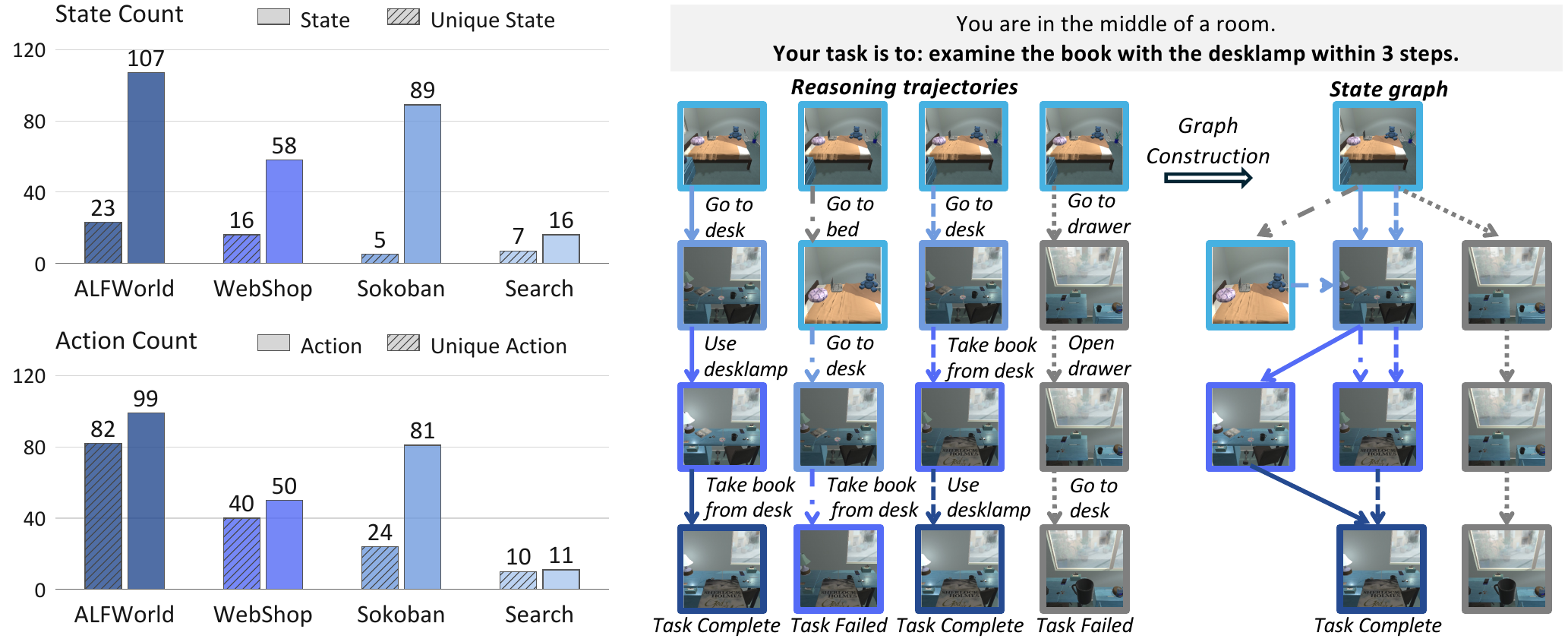}
    \caption{\textbf{Shared states connect sampled trajectories across environments.} \textbf{(left)} Per-task average counts of total and unique states and actions, computed from 128 trajectories sampled across 8 tasks using Qwen2.5-(VL)-3B-Instruct on ALFWorld, WebShop, Sokoban, and Search, reveal repeated visits to states and actions. \textbf{(right)} Successful and failed trajectories can reach the same states; merging these shared states connects their observed paths into a single state graph.}
    \label{fig:graph_construction}
    \label{fig: total_vs_unique}
\end{figure}

To estimate action-level rewards from paths observed across trajectories, RewardFlow connects sampled trajectories through shared states into a \emph{state graph} (Fig.~\ref{fig:method-overview}, Steps~1-2). First, we describe how to identify shared states. Then, we construct the graph from observed state transitions. Finally, we analyze reachability and shortest-path distance to observed success and explain how shared nodes connect these paths.

\subsection{State Identification across Trajectories}
\label{ssec: rethinking_states_in_agentic_reasoning}

Across environments, the gap between total and unique counts of states and actions shows repeated visits in sampled trajectories (Fig.~\ref{fig: total_vs_unique}, left). Trajectories sampled for the same task can reach shared intermediate states, even when their subsequent actions and final outcomes differ (Fig.~\ref{fig:graph_construction}, right). Identifying these meeting points lets one trajectory inform another about actions and paths observed from a shared state, even when their final outcomes differ.

We group task-relevant situations identified as equivalent or semantically similar through a map $f:\mathcal S\to\mathcal V$: states assigned to the same graph node $\hat s_v$ are treated as shared states. Under partial observability, state identification uses task-relevant observations and interaction history. Task-specific state preprocessing is detailed in Appendix~\ref{app: state preprocessing}. For example, $f(s_1^{(1)})=f(s_2^{(2)})=f(s_t^{(i)})=\hat s_v$ joins visits from three trajectories. If a failed trajectory reaches this node and another trajectory reaches success from it, the graph connects the first trajectory's path to the node with the second trajectory's path to success. This cross-trajectory path provides evidence for evaluating actions in the failed trajectory beyond its own terminal outcome.

\subsection{Constructing the State Graph}
\label{ssec: state_graph_construction}

Shared states form graph nodes, and observed state transitions form directed edges labeled by their actions. We include environment-valid state transitions and exclude invalid actions before constructing paths (Fig.~\ref{fig: invalid_action}, left). Appendix~\ref{app: invalid action filtering} gives the validity rules.

Let $\mathcal I_{\mathrm{valid}}\subseteq\{(i,t):1\leq i\leq G,\ 0\leq t<T_i\}$ contain the indices of environment-valid state transitions. The directed state graph $\mathcal G=(\mathcal V,\mathcal A_{\mathcal G},\mathcal E)$ is defined by
\begin{equation}
\begin{aligned}
\mathcal A_{\mathcal G} = \{a_t^{(i)}:(i,t)\in\mathcal I_{\mathrm{valid}}\}\subseteq\mathcal A,\qquad
\mathcal E = \bigl\{\bigl(f(s_t^{(i)}),a_t^{(i)},f(s_{t+1}^{(i)})\bigr):(i,t)\in\mathcal I_{\mathrm{valid}}\bigr\}.
\end{aligned}
\label{eq:induced_graph}
\end{equation}
Thus $\mathcal V=\bigcup_{i=1}^{G}\,\bigcup_{t=0}^{T_i}\{f(s_t^{(i)})\}$ includes all sampled states after applying $f$, including terminal states, while $\mathcal A_{\mathcal G}$ contains only labels of retained state transitions. We mark successful terminal nodes using the trajectories' terminal outcomes; these nodes anchor reward propagation in Sec.~\ref{ssec: process_reward_shaping}.

Each edge $(\hat s_v,a,\hat s_w)\in\mathcal E$ records an observed state transition, allowing paths to continue across trajectories at shared states. Appendix~\hyperref[app:graph-diagnostics]{\ref*{app: state preprocessing}} provides further examples and statistics.

\subsection{Analyzing State-Graph Structure}
\label{ssec: understanding_the_state_graph}

The state graph connects intermediate states to task completion through environment-valid state transitions. Two complementary properties reveal this structure (Fig.~\ref{fig: graph_analysis}, right):
\begin{itemize}[leftmargin=*, itemsep=3pt, parsep=0pt, topsep=3pt]
    \item \textbf{Reachability.} A state is success-reachable when it has a path to an observed successful terminal. The shortest path measures its distance to success and supplies the basis for action-level reward estimation. Bright blue branches lack an observed path to success.
    \item \textbf{Centrality.} Incoming and outgoing state transitions show how a state connects different routes. The desk state links paths involving the book and desklamp, allowing a path entering this shared state to follow actions observed in another trajectory.
\end{itemize}

By connecting sampled states to observed success, these analyses help assess each state's contribution to task completion. Therefore, the resulting state graph enables action-level reward estimation across trajectories, providing more reliable feedback than trajectory-local estimation.


\begin{figure*}[t]
    \centering
    \includegraphics[width=\linewidth]{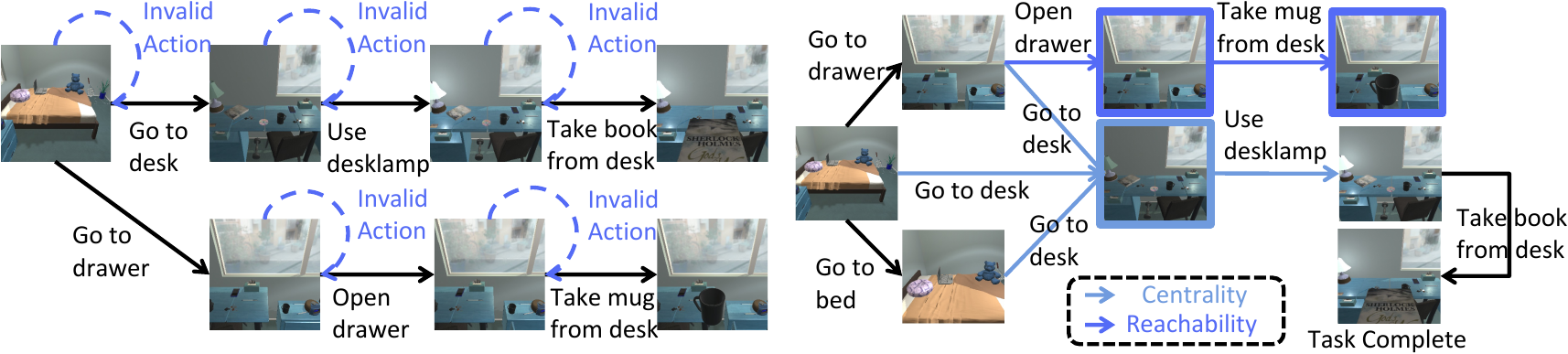}
    \caption{\textbf{State graph construction and structure.} \textbf{(left)} Environment-valid state transitions (black arrows) define the graph; blue dashed loops indicate invalid actions excluded from path construction. \textbf{(right)} The graph reveals reachability and centrality: bright blue branches lack an observed path to success, while the light blue desk state links incoming and outgoing state transitions.}
    \label{fig: invalid_action}
    \label{fig: graph_analysis}
\end{figure*}

\section{Reward Estimation from State Graphs}
\label{sec: method}

Building on the state graph, RewardFlow estimates action-level rewards to complement trajectory-level supervision in agentic RL (Fig.~\ref{fig:method-overview}, Steps~3-4). We first present action-level reward estimation, then describe how action- and trajectory-level advantages are combined for policy optimization. Finally, we analyze approximation error and trajectory-level reward ordering.

\subsection{Action-Level Reward Estimation}
\label{ssec: process_reward_shaping}

\paragraph{State potentials.}
Successful terminals anchor reward estimation in observed task completion. Let $\mathcal V_{\mathrm{succ}}\subseteq\mathcal V$ be these states in the graph. Starting from these anchors, we propagate a reward signal backward along incoming state transitions, so that states leading to observed success receive a potential. Each additional step away from success attenuates the signal by a factor of $\gamma$; when several paths are available, the shortest path determines the potential. We compute these distances using reverse breadth-first search (BFS)~\citep{Moore-bfs} and define the potential as
\begin{equation}
\begin{aligned}
d_{\mathcal G}(\hat s_v)&:=\min_{\hat s_w\in\mathcal V_{\mathrm{succ}}}
\operatorname{dist}_{\mathrm{hop}}(\hat s_v\leadsto\hat s_w), \quad
\Phi_{\mathcal G}(\hat s_v):=
\begin{cases}
\gamma^{d_{\mathcal G}(\hat s_v)}, & d_{\mathcal G}(\hat s_v)<\infty,\\
0, & d_{\mathcal G}(\hat s_v)=\infty,
\end{cases}\; \gamma\in(0,1).
\end{aligned}
\label{eq:reward_propagation}
\end{equation}
We set $d_{\mathcal G}=\infty$ when no path reaches observed success, including when $\mathcal V_{\mathrm{succ}}$ is empty. Successful terminals have potential $1$; states without such a path have potential $0$. Shorter paths receive higher potentials, including paths assembled from different trajectories. Each state's potential thus measures its proximity to observed success using experience pooled across trajectories. We further compare alternative graph propagation rules for potential estimation in Appendices~\ref{propagation strategy} and~\ref{app: bfs vs ppr}.

\paragraph{Action-level rewards.}
Let $\phi_t^{(i)}$ denote the graph potential at state $s_t^{(i)}$ in trajectory $i$. An environment-valid action receives the potential change along its observed state transition:
\begin{equation}
\label{eq:reward_difference}
\phi_t^{(i)}:=\Phi_{\mathcal G}\bigl(f(s_t^{(i)})\bigr),\qquad
\tilde r_t^{(i)}:=\phi_{t+1}^{(i)}-\phi_t^{(i)}.
\end{equation}
Positive rewards recognize progress toward success in the state graph, while negative rewards penalize detours that increase the distance to success. An action in a failed trajectory can therefore receive a positive reward through a route discovered by another trajectory. Appendix~\ref{app:gigpo-topology-example} illustrates this effect with a worked example comparing RewardFlow's cross-trajectory reward estimation with GiGPO's trajectory-local estimates.

\begin{figure}[t]
  \centering
  \includegraphics[width=\linewidth]{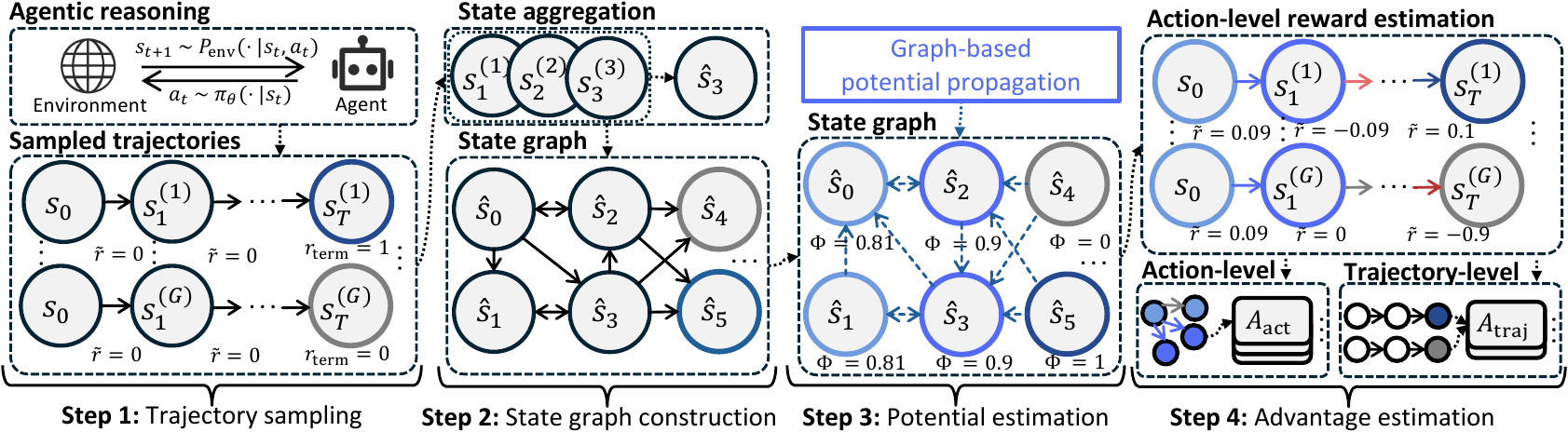}
  \vspace{-12pt}
  \caption{\textbf{Action-level reward estimation with RewardFlow.} For each task, policy $\pi_\theta$ samples $G$ trajectories (Step~1). Shared states become nodes $\hat s_v$ connecting observed state transitions across trajectories (Step~2). Potentials $\Phi$ decay with shortest-path distance to successful terminals ($\gamma=0.9$ here; Step~3). Their differences define action rewards $\tilde r$, which complement terminal outcomes $r_{\mathrm{term}}$ through action- and trajectory-level advantages $A_{\mathrm{act}}$ and $A_{\mathrm{traj}}$ in policy optimization (Step~4).}
  \label{fig:method-overview}
  \vspace{-12pt}
\end{figure}

\subsection{Complementary Advantage Estimation}
\label{ssec: advantage estimation}

\paragraph{Action-level advantage.}
Comparing rewards at the same graph node favors actions making greater graph-estimated progress relative to alternatives from that state. Let $\mathcal I_v=\{(i,t)\in\mathcal I_{\mathrm{valid}}:f(s_t^{(i)})=\hat s_v\}$ index these occurrences, counting repeated visits separately. Let $\mu(\hat s_v)$ be their mean reward and $\sigma(\hat s_v)>0$ be a normalization scale obtained by applying $F_{\mathrm{norm}}$ to the reward multiset, such as the standard deviation plus $\varepsilon_{\mathrm{norm}}$. The action-level advantage is
\begin{equation}
A_{\mathrm{act},t}^{(i)}
:=(\tilde r_t^{(i)}-\mu(f(s_t^{(i)})))/{\sigma(f(s_t^{(i)}))}.
\label{eq:state_advantage}
\end{equation}
For a singleton group, we set $\mu$ to its action reward and $\sigma=1$, giving zero action advantage.


\paragraph{Trajectory-level advantage.}
Terminal outcomes provide complementary supervision for task completion. Let $\mu_{\mathrm{term}}$ be a task-group baseline computed from the sampled terminal outcomes and $\sigma_{\mathrm{term}}>0$ a normalization scale. Following GRPO~\citep{shao2024deepseekmath}, we define
\begin{equation}
A_{\mathrm{traj}}^{(i)}
:=(r_{\mathrm{term}}^{(i)}-\mu_{\mathrm{term}})/{\sigma_{\mathrm{term}}}.
\label{eq:trajectory_advantage}
\end{equation}
Thus, action-level advantages compare progress from a shared state, whereas trajectory-level advantages compare task-completion outcomes across attempts at the same task. Training configurations are provided in Appendix~\ref{app: detailed experimental setting}.

\paragraph{Combining action- and trajectory-level advantages.}
Combining the two signals gives the policy guidance on individual actions while retaining supervision from overall task completion. Each generated token $k$ in action turn $t$ receives the same complementary advantage, with weights $\alpha_\text{act},\alpha_{\mathrm{traj}}\geq0$:
\begin{equation}
A_{t,k}^{(i)}:=\alpha_\text{act}A_{\mathrm{act},t}^{(i)}
+\alpha_\text{traj}A_{\mathrm{traj}}^{(i)}.
\end{equation}
With binary success as the only reward, both advantages vanish in an all-failure group; action advantages can still distinguish progress in an all-success group (Appendix~\ref{app: exploration}).

\subsection{Policy Optimization}
\label{ssec: policy_update}

Using the complementary advantage, we follow GRPO~\citep{shao2024deepseekmath} to optimize the policy with the clipped PPO objective~\citep{schulman2017proximal}. For each task $x\sim P(X)$, we sample $G$ trajectories with policy $\pi_{\theta_{\mathrm{old}}}$, retaining their generated token sequences. With $o_{t,<k}^{(i)}$ denoting preceding tokens, the importance ratio is
$
\rho_{t,k}^{(i)}:=
\frac{\pi_\theta(o_{t,k}^{(i)}\mid s_t^{(i)},o_{t,<k}^{(i)})}
{\pi_{\theta_{\text{old}}}(o_{t,k}^{(i)}\mid s_t^{(i)},o_{t,<k}^{(i)})}.
$
The objective combines clipped policy improvement with a KL penalty:
\begin{equation}
\begin{aligned}
 J_{\mathrm{RewardFlow}}&(\theta) ={} \mathbb{E}_{\substack{x\sim P(X),\\ \{\tau^{(i)}\}_{i=1}^G\sim\pi_{\theta_{\mathrm{old}}}(\cdot\mid x)}}\Bigg[\frac{1}{G}\sum_{i=1}^G\frac{1}{T_i}\sum_{t=0}^{T_i-1}\frac{1}{|o_t^{(i)}|}\sum_{k=1}^{|o_t^{(i)}|}\\
 &\quad\bigg(\min\!\left[\rho_{t,k}^{(i)} A_{t,k}^{(i)},\operatorname{clip}(\rho_{t,k}^{(i)},1-\varepsilon_{\mathrm{clip}},1+\varepsilon_{\mathrm{clip}}) A_{t,k}^{(i)}\right]\;-\;\beta\mathbb D_{\mathrm{KL}}(\pi_\theta\|\pi_{\mathrm{ref}})\bigg)\Bigg].
\end{aligned}
\label{eq:rewardflow_objective}
\end{equation}
Here $\pi_{\theta_{\mathrm{old}}}(\cdot\mid x)$ denotes the trajectory distribution induced by the behavior policy and environment. The objective averages over trajectories, action turns, and tokens. The parameter $\varepsilon_{\mathrm{clip}}>0$ controls clipping, while $\beta\geq0$ weights the token-context KL penalty relative to the fixed policy $\pi_{\mathrm{ref}}$.

\subsection{Theoretical Analysis}
\label{ssec:theoretical_justification}

We first analyze graph approximation, then its effect on state potentials and action-level rewards, and finally trajectory-level reward ordering. For graph approximation, assume sufficient POMDP information states and a fixed map $f$ into a countable shared-state space that preserves admissible actions, terminal status and outcomes, and possible successor classes. Let $\mathcal G_{\mathrm{env}}^f$ contain reachable shared states and their admissible state transitions, and let $d_{\mathrm{env}}^*(s)$ denote its shortest hop distance from $f(s)$ to success, with distance $\infty$ when no such path exists. Full proofs are in Appendix~\ref{app:proofs}.

\begin{proposition}[POMDP State-Transition Graph Approximation]
\label{prop:graph_approximation}
\label{rem:partial_coverage}
Under these conditions, trajectory sampling yields $\mathcal G\subseteq\mathcal G_{\mathrm{env}}^f$ and $d_{\mathcal G}(f(s))\geq d_{\mathrm{env}}^*(s)$ for every represented $s$, with probability one. Distances agree once a shortest path and its successful endpoint are sampled.
\end{proposition}

This relation provides the reference for analyzing distance-based state potentials.
\begin{remark}[POMDP-Induced State Potential]
\label{rem:value_topology}
Under Prop.~\ref{prop:graph_approximation}, let $\Phi_{\mathrm{env}}(s):=\gamma^{d_{\mathrm{env}}^*(s)}$, with $\gamma\in(0,1)$ and $\gamma^\infty=0$. This potential measures shortest-path proximity to success in $\mathcal G_{\mathrm{env}}^f$. For stochastic state transitions with successful-terminal value $1$ and zero non-terminal rewards, the optimal expected discounted terminal value over admissible actions satisfies $V^*(s)\leq\Phi_{\mathrm{env}}(s)$.
\looseness=-1
\end{remark}

Missing paths can make the sampled distance to success longer than the environment's shortest distance. Since the potential decays with distance, the sampled graph then assigns a lower potential. The following result expresses this difference and identifies when it vanishes.
\begin{proposition}[State-Potential Approximation]
  \label{prop:value_approximation}
  Under Rem.~\ref{rem:value_topology}, the potential error at any represented information state $s$ is
  \[
  \delta(s):=\Phi_{\mathrm{env}}(s)-\Phi_{\mathcal G}(f(s))
  =\gamma^{d_{\mathrm{env}}^*(s)}-\gamma^{d_{\mathcal G}(f(s))}\geq0.
  \]
  The error vanishes once a shortest environment path and its successful endpoint are sampled, and converges almost surely to zero under the cumulative-sampling conditions in Appendix~\ref{app:graph-approximation-proof}.
\end{proposition}

Action rewards are differences between the potentials at the two endpoints of a state transition. Their approximation error therefore depends on the difference between the endpoint potential errors. For a retained state transition, write $\tilde r(s,a,s'):=\Phi_{\mathcal G}(f(s'))-\Phi_{\mathcal G}(f(s))$.

\begin{corollary}[Action-Level Reward Approximation]
\label{cor:value_progress}
For retained $(s,a,s')$ and $\delta$ from Prop.~\ref{prop:value_approximation},
\[
\tilde r(s,a,s')
=\bigl(\Phi_{\mathrm{env}}(s')-\Phi_{\mathrm{env}}(s)\bigr)
-\bigl(\delta(s')-\delta(s)\bigr).
\]
The error relative to the reference potential change is at most $\max\{\delta(s),\delta(s')\}$. As both coverage errors vanish, the graph-derived action reward recovers this POMDP-induced progress measure.
\end{corollary}

Finally, summing raw action rewards along a trajectory whose state transitions are all retained cancels intermediate potentials, leaving the final potential minus the initial potential. For such trajectories starting at the same graph node, their total raw action rewards are ordered by their endpoint potentials even when graph coverage is incomplete.
\begin{samepage}
\begin{proposition}[Trajectory-Level Reward Ordering]
\label{prop:ordering_invariance}
For a successful trajectory $\tau^{(+)}$ and a failed trajectory $\tau^{(-)}$ whose state transitions are all retained in $\mathcal G$, let $T_+$ and $T_-$ be their action counts. If both start at the same canonical node $f(s_0^{(+)})=f(s_0^{(-)})=\hat s_{\mathrm{init}}$, their raw action rewards satisfy $\sum_{t=0}^{T_+-1}\tilde r_t^{(+)}-\sum_{t=0}^{T_--1}\tilde r_t^{(-)}=\phi_{T_+}^{(+)}-\phi_{T_-}^{(-)}$, regardless of graph coverage. Since $\phi_{T_+}^{(+)}=1$, the successful trajectory receives a strictly larger total reward whenever the failed endpoint is not merged with a success node, so $\phi_{T_-}^{(-)}<1$.
\end{proposition}
\end{samepage}

\section{Experiments}
\label{sec: experiments}

We evaluate RewardFlow both in environments with well-defined state structures and in search tasks, where state identification requires an additional semantic definition. We first compare task performance, then examine graph-derived reward estimation and the complementarity of action- and trajectory-level advantages, and finally evaluate robustness and training efficiency.

\subsection{Experimental Setup}

Experiments cover text-based household tasks in \textbf{ALFWorld}~\citep{shridhar2021alfworld}, web shopping in \textbf{WebShop}~\citep{yao2022webshop}, visual puzzles in \textbf{Sokoban}~\citep{SchraderSokoban2018}, and retrieval-based question answering in \textbf{search tasks}~\citep{jin2025search}. In ALFWorld, WebShop, and Sokoban, environment-defined object states, page configurations, and board layouts provide a concrete basis for state identification, with interaction history retained where needed. Search tasks instead require an additional definition of shared states: we group observations by semantic similarity to connect related stages of information gathering. We train Qwen2.5-1.5B/3B/7B-Instruct on ALFWorld and WebShop, Qwen2.5-3B/7B-Instruct on search tasks, and Qwen2.5-VL-3B/7B-Instruct on Sokoban.

RLOO, GRPO, and GiGPO are the RL baselines for ALFWorld, WebShop, and Sokoban. For search tasks, we follow GiGPO's setup and also compare R1-Instruct, Search-R1~\citep{jin2025search}, ZeroSearch~\citep{sun2025zerosearch}, and StepSearch~\citep{wang2025stepsearch}. GiGPO and RewardFlow use the same text encoder for semantic grouping (Appendix~\ref{app: state preprocessing}). Appendix~\ref{app: detailed experimental setting} details the benchmarks, metrics, and training configurations; supplementary controls specify their protocols separately.

\begin{table*}[t!]
\mainresultcolors
\vspace{-8pt}
\centering
\caption{Main results (mean $\pm$ std over three seeds). We report success rates (\%) for ALFWorld (per-subtask and overall average), and both task score and success rate for WebShop and Sokoban. \textbf{Bold} and \underline{underline} denote the best and second-best results per column.}
\label{tab:main}
\vspace{-6pt}
\renewcommand{\arraystretch}{1.20}
\fontsize{7}{8}\selectfont
\setlength{\tabcolsep}{2.5pt}
\begin{tabular}{@{}ll @{\hspace{4pt}} | ccccccc | cc | cc@{}}
\toprule
\multirow{2}{*}{\textbf{Type}} & \multirow{2}{*}{\textbf{Method}}& \multicolumn{7}{c|}{\textbf{ALFWorld}} & \multicolumn{2}{c|}{\textbf{WebShop}} & \multicolumn{2}{c}{\textbf{Sokoban}} \\
\cmidrule(lr){3-9}\cmidrule(lr){10-11}\cmidrule(lr){12-13}
    & & Pick & Look & Clean & Heat & Cool & Pick2 & All & Score & Succ. & Score & Succ. \\
\midrule
\rowcolor{regimeHelpful}
\multicolumn{13}{@{}l}{\textit{Qwen2.5-1.5B-Instruct}} \\
Prompting & Base & 9.3\textsubscript{\textpm0.2} & 0.0\textsubscript{\textpm0.0} & 0.0\textsubscript{\textpm0.0} & 0.0\textsubscript{\textpm0.0} & 0.0\textsubscript{\textpm0.0} & 10.5\textsubscript{\textpm0.0} & 3.9\textsubscript{\textpm0.0} & 20.8\textsubscript{\textpm1.5} & 4.9\textsubscript{\textpm2.0} & {--} & {--} \\
RL Training & RLOO & \cellbest{73.3\textsubscript{\textpm5.9}} & \cellsec{16.7\textsubscript{\textpm0.0}} & 46.8\textsubscript{\textpm4.0} & \cellsec{74.6\textsubscript{\textpm6.9}} & \cellsec{54.2\textsubscript{\textpm4.2}} & \cellsec{52.6\textsubscript{\textpm9.1}} & \cellsec{56.0\textsubscript{\textpm2.5}} & 80.8\textsubscript{\textpm1.2} & 56.5\textsubscript{\textpm7.4} & {--} & {--} \\
RL Training & GRPO & \cellsec{72.2\textsubscript{\textpm3.2}} & 13.9\textsubscript{\textpm4.8} & 34.2\textsubscript{\textpm10.4} & 42.1\textsubscript{\textpm8.4} & \cellbest{58.3\textsubscript{\textpm7.2}} & 49.1\textsubscript{\textpm12.2} & 49.5\textsubscript{\textpm3.0} & 64.8\textsubscript{\textpm2.0} & 45.1\textsubscript{\textpm0.5} & {--} & {--} \\
RL Training & GiGPO & 65.9\textsubscript{\textpm5.7} & \cellbest{26.7\textsubscript{\textpm17.3}} & \cellsec{54.7\textsubscript{\textpm8.2}} & 33.8\textsubscript{\textpm9.1} & 50.6\textsubscript{\textpm6.2} & 48.6\textsubscript{\textpm10.3} & 50.8\textsubscript{\textpm2.8} & \cellsec{81.7\textsubscript{\textpm5.5}} & \cellsec{60.7\textsubscript{\textpm4.8}} & {--} & {--} \\
RL Training & RewardFlow & 72.1\textsubscript{\textpm3.5} & \cellsec{16.7\textsubscript{\textpm0.0}} & \cellbest{76.0\textsubscript{\textpm7.6}} & \cellbest{95.2\textsubscript{\textpm8.2}} & 27.8\textsubscript{\textpm6.4} & \cellbest{77.2\textsubscript{\textpm6.1}} & \cellbest{62.8\textsubscript{\textpm1.8}} & \cellbest{83.8\textsubscript{\textpm4.8}} & \cellbest{66.4\textsubscript{\textpm5.9}} & {--} & {--} \\
\midrule
\rowcolor{regimeHelpful}
\multicolumn{13}{@{}l}{\textit{Qwen2.5-(VL)-3B-Instruct}} \\
Prompting & Base & 36.2\textsubscript{\textpm15.1} & 13.9\textsubscript{\textpm4.8} & 0.0\textsubscript{\textpm0.0} & 2.4\textsubscript{\textpm4.1} & 5.6\textsubscript{\textpm2.4} & 12.3\textsubscript{\textpm3.0} & 13.5\textsubscript{\textpm3.3} & 6.1\textsubscript{\textpm1.7} & 0.5\textsubscript{\textpm0.9} & 0.8\textsubscript{\textpm0.3} & 18.0\textsubscript{\textpm3.4} \\
RL Training & RLOO & \cellbest{88.7\textsubscript{\textpm1.7}} & 16.7\textsubscript{\textpm0.0} & 55.7\textsubscript{\textpm5.8} & \cellsec{53.7\textsubscript{\textpm6.0}} & 41.7\textsubscript{\textpm4.2} & 42.1\textsubscript{\textpm9.1} & 55.5\textsubscript{\textpm0.0} & 70.3\textsubscript{\textpm6.0} & 49.2\textsubscript{\textpm5.9} & 0.7\textsubscript{\textpm0.1} & 16.7\textsubscript{\textpm0.5} \\
RL Training & GRPO & 83.5\textsubscript{\textpm2.0} & 16.7\textsubscript{\textpm0.0} & 45.6\textsubscript{\textpm1.0} & 34.9\textsubscript{\textpm1.4} & 41.7\textsubscript{\textpm4.2} & 36.8\textsubscript{\textpm15.8} & 49.2\textsubscript{\textpm3.6} & 63.1\textsubscript{\textpm1.9} & 49.0\textsubscript{\textpm4.0} & \cellsec{1.0\textsubscript{\textpm0.2}} & \cellsec{22.9\textsubscript{\textpm3.7}} \\
RL Training & GiGPO & 84.8\textsubscript{\textpm2.9} & \cellsec{41.7\textsubscript{\textpm8.3}} & \cellsec{65.9\textsubscript{\textpm9.6}} & \cellbest{55.9\textsubscript{\textpm2.2}} & \cellsec{53.6\textsubscript{\textpm14.0}} & \cellsec{65.8\textsubscript{\textpm14.7}} & \cellsec{64.6\textsubscript{\textpm0.4}} & \cellsec{77.9\textsubscript{\textpm3.4}} & \cellsec{62.0\textsubscript{\textpm5.9}} & 0.9\textsubscript{\textpm0.3} & 19.8\textsubscript{\textpm3.6} \\
RL Training & RewardFlow & \cellsec{87.0\textsubscript{\textpm7.9}} & \cellbest{56.7\textsubscript{\textpm11.5}} & \cellbest{85.3\textsubscript{\textpm4.6}} & 43.1\textsubscript{\textpm6.8} & \cellbest{73.3\textsubscript{\textpm5.1}} & \cellbest{76.0\textsubscript{\textpm7.4}} & \cellbest{75.5\textsubscript{\textpm3.2}} & \cellbest{78.8\textsubscript{\textpm4.6}} & \cellbest{62.7\textsubscript{\textpm3.2}} & \cellbest{2.1\textsubscript{\textpm0.2}} & \cellbest{49.5\textsubscript{\textpm1.2}} \\
\midrule
\rowcolor{regimeHelpful}
\multicolumn{13}{@{}l}{\textit{Qwen2.5-(VL)-7B-Instruct}} \\
Prompting & Base & 45.3\textsubscript{\textpm6.8} & 16.7\textsubscript{\textpm0.0} & 8.8\textsubscript{\textpm2.0} & 2.4\textsubscript{\textpm4.1} & 6.9\textsubscript{\textpm2.4} & 0.0\textsubscript{\textpm0.0} & 16.4\textsubscript{\textpm2.1} & 3.3\textsubscript{\textpm0.4} & 0.3\textsubscript{\textpm0.5} & 1.3\textsubscript{\textpm0.1} & 25.0\textsubscript{\textpm2.1} \\
RL Training & RLOO & \cellsec{94.8\textsubscript{\textpm1.9}} & 36.1\textsubscript{\textpm4.8} & 78.4\textsubscript{\textpm11.8} & 79.0\textsubscript{\textpm0.8} & 68.1\textsubscript{\textpm2.4} & 75.4\textsubscript{\textpm6.1} & 76.3\textsubscript{\textpm2.7} & 76.6\textsubscript{\textpm2.8} & 65.6\textsubscript{\textpm0.8} & \cellsec{1.4\textsubscript{\textpm0.2}} & 35.7\textsubscript{\textpm3.3} \\
RL Training & GRPO & 90.7\textsubscript{\textpm0.2} & 25.0\textsubscript{\textpm0.0} & 63.4\textsubscript{\textpm10.7} & 72.4\textsubscript{\textpm10.7} & \cellsec{70.8\textsubscript{\textpm0.0}} & 73.7\textsubscript{\textpm9.1} & 70.6\textsubscript{\textpm2.0} & 78.5\textsubscript{\textpm3.2} & 65.1\textsubscript{\textpm3.5} & 1.3\textsubscript{\textpm0.3} & 26.0\textsubscript{\textpm5.1} \\
RL Training & GiGPO & 90.7\textsubscript{\textpm6.2} & \cellsec{43.3\textsubscript{\textpm31.8}} & \cellsec{87.3\textsubscript{\textpm9.6}} & \cellsec{83.3\textsubscript{\textpm10.9}} & \cellbest{73.2\textsubscript{\textpm2.2}} & \cellsec{85.9\textsubscript{\textpm3.1}} & \cellsec{80.7\textsubscript{\textpm2.0}} & \cellsec{83.8\textsubscript{\textpm1.5}} & \cellsec{70.6\textsubscript{\textpm0.9}} & \cellsec{1.4\textsubscript{\textpm0.2}} & \cellsec{37.8\textsubscript{\textpm3.5}} \\
RL Training & RewardFlow & \cellbest{95.9\textsubscript{\textpm3.6}} & \cellbest{62.3\textsubscript{\textpm14.7}} & \cellbest{92.4\textsubscript{\textpm7.7}} & \cellbest{84.2\textsubscript{\textpm2.5}} & 67.5\textsubscript{\textpm12.5} & \cellbest{91.5\textsubscript{\textpm5.6}} & \cellbest{84.6\textsubscript{\textpm4.5}} & \cellbest{85.3\textsubscript{\textpm0.9}} & \cellbest{74.7\textsubscript{\textpm3.0}} & \cellbest{3.5\textsubscript{\textpm0.4}} & \cellbest{67.4\textsubscript{\textpm4.5}} \\
\bottomrule
\end{tabular}%
\vspace{-8pt}
\end{table*}

\begin{table*}[t!]
\centering
\mainresultcolors
\caption{Performance comparison on search tasks (trained on NQ and HotpotQA following GiGPO). Reported as average accuracy (\%). \dag in-distribution, * out-of-distribution.}
\label{tab:deepresearch}
\vspace{-6pt}
\renewcommand{\arraystretch}{1.20}
\fontsize{8}{8}\selectfont
\setlength{\tabcolsep}{4pt}
\begin{tabular}{@{}ll @{\hspace{4pt}} | ccc | cccc | c@{}}
\toprule
 \multirow{2}{*}{\textbf{Type}} & \multirow{2}{*}{\textbf{Method}} & \multicolumn{3}{c|}{\textbf{Single-hop QA}} & \multicolumn{4}{c|}{\textbf{Multi-hop QA}} & \multirow{2}{*}{\textbf{Avg.}}  \\
\cmidrule(lr){3-5}\cmidrule(lr){6-9}
 & & NQ$^\dag$ & TriviaQA* & PopQA* & HotpotQA$^\dag$ & 2Wiki* & Musique* & Bamboogle* & \\
\midrule
\rowcolor{regimeHelpful}
\multicolumn{10}{@{}l}{\textit{Qwen2.5-3B-Instruct}} \\
RL Training & R1-Instruct  & 27.0 & 53.7 & 19.9 & 23.7 & 29.2 & 7.2  & 29.3 & 27.1 \\
RL Training & Search-R1    & 34.1 & 54.5 & 37.8 & 32.4 & 31.9 & 10.3 & 26.4 & 32.5 \\
RL Training & ZeroSearch   & 41.4 & 57.4 & \cellbest{44.8} & 27.4 & 30.0 & 9.8  & 11.1 & 31.7 \\
RL Training & StepSearch   & {--} & {--} & {--} & 34.5 & 32.0 & \cellbest{17.4} & 34.4 & {--} \\
RL Training & GiGPO        & \cellsec{42.0} & \cellsec{59.5} & 42.4 & \cellsec{36.9} & \cellsec{37.0} & 12.6 & \cellbest{64.1} & \cellsec{42.1} \\
RL Training & RewardFlow & \cellbest{44.4} & \cellbest{60.9} & \cellsec{44.1} & \cellbest{40.3} & \cellbest{41.2} & \cellsec{15.2} & \cellsec{63.7} & \cellbest{44.3} \\
\midrule
\rowcolor{regimeHelpful}
\multicolumn{10}{@{}l}{\textit{Qwen2.5-7B-Instruct}} \\
RL Training & R1-Instruct  & 21.0 & 44.9 & 17.1 & 20.8 & 27.5 & 6.0  & 19.2 & 22.4 \\
RL Training & Search-R1    & 39.3 & 61.0 & 39.7 & 37.0 & 40.1 & 14.6 & 36.8 & 38.5 \\
RL Training & ZeroSearch   & 43.6 & 61.8 & \cellbest{51.5} & 34.6 & 35.2 & 18.4 & 27.8 & 39.1 \\
RL Training & StepSearch   & {--} & {--} & {--} & 38.6 & 36.6 & \cellbest{22.6} & 40.0 & {--} \\
RL Training & GiGPO        & \cellsec{46.4} & \cellsec{64.7} & 46.1 & \cellsec{41.6} & \cellsec{43.6} & 18.9 & \cellsec{68.9} & \cellsec{47.2} \\
RL Training & RewardFlow  & \cellbest{47.4} & \cellbest{65.2} & \cellsec{48.1} & \cellbest{44.7} & \cellbest{47.1} & \cellsec{19.9} & \cellbest{71.4} & \cellbest{49.1} \\
\bottomrule
\end{tabular}%
\vspace{-8pt}
\end{table*}

\subsection{Task Performance}
\label{ssec:task-performance}
\label{ssec:prm-comparison}

\paragraph{Agentic tasks.}
As illustrated in Tab.~\ref{tab:main}, RewardFlow achieves the highest overall ALFWorld success rate and the highest WebShop and Sokoban success rates at every evaluated model scale. On ALFWorld, it improves success over GiGPO from $64.6\%$ to $75.5\%$ at 3B and from $80.7\%$ to $84.6\%$ at 7B. It also leads in both task-match score and success rate on WebShop, showing improvements in both product matching and task completion. The improvements extend to visual tasks: on Sokoban, RewardFlow exceeds the strongest RL baseline by $26.6\%$ at 3B and $29.6\%$ at 7B. Together, these results demonstrate the value of graph-derived action rewards across tasks with different observation modalities and action spaces.

\paragraph{Search tasks.}
Search tasks pose a further challenge for graph construction: retrieved passages vary with the query, and related stages of information gathering need not produce identical observations. We therefore, following GiGPO's setting, identify shared states through approximate semantic matching rather than exact observation equality. Despite this challenge, RewardFlow achieves average accuracy of $44.3\%$ at 3B and $49.1\%$ at 7B, outperforming GiGPO by $2.2\%$ and $1.9\%$, respectively (Tab.~\ref{tab:deepresearch}). With both methods using the same state-preprocessing method, these additional improvements suggest that RewardFlow's graph-based reward estimation better captures progress through observed environment interactions than GiGPO's trajectory-local reward estimation.

\paragraph{Comparison with process reward models.}
Baselines using PRM-estimated rewards underperform GRPO alone, whereas RewardFlow's action-level rewards improve RL success from $50.0\%$ to $62.8\%$ without a learned process scorer (Fig.~\ref{fig:rewardflow_teaser}, middle). Full details are provided in Appendix~\ref{app: prm}; we discuss fixed-scorer transfer failures in Appendix~\ref{app: fixed scorer diagnostics}.

\begin{figure}[t]
\centering

\begin{minipage}[t]{0.50\textwidth}
  \mainresultcolors
  \centering
  \captionof{table}{Out-of-distribution evaluation on ALFWorld. The agent solves household tasks with familiar objects in entirely novel environments (different rooms, layouts, furniture). We report the average success (\%); higher is better.}
  \label{tab:ood}
  \vspace{2pt}
    \renewcommand{\arraystretch}{1.20}
    \fontsize{6}{8}\selectfont
    \setlength{\tabcolsep}{0.5pt}
  \begin{tabular}{@{}l | ccccccc@{}}
  \toprule
  \multirow{2}{*}{Method} & \multicolumn{7}{c}{\textbf{ALFWorld}}\\
  \cmidrule(lr){2-8}
  & Pick & Look & Clean & Heat & Cool & Pick2 & All\\
  \midrule
  \rowcolor{regimeHelpful}\multicolumn{8}{@{}l}{\textit{Qwen2.5-1.5B-Instruct}}\\
  Base & 2.4\textsubscript{\textpm4.1} & 0.0\textsubscript{\textpm0.0} & 5.7\textsubscript{\textpm2.8} & 3.4\textsubscript{\textpm3.4} & 0.0\textsubscript{\textpm0.0} & 0.0\textsubscript{\textpm0.0} & 2.6\textsubscript{\textpm1.2} \\
  RLOO & 66.7\textsubscript{\textpm10.9} & 35.1\textsubscript{\textpm3.0} & 37.5\textsubscript{\textpm3.5} & 41.4\textsubscript{\textpm6.9} & 26.6\textsubscript{\textpm4.1} & \cellsec{33.3\textsubscript{\textpm13.3}} & 39.3\textsubscript{\textpm2.3} \\
  GRPO & \cellsec{69.0\textsubscript{\textpm4.1}} & 26.3\textsubscript{\textpm13.9} & 33.7\textsubscript{\textpm4.9} & 34.5\textsubscript{\textpm6.9} & \cellbest{32.5\textsubscript{\textpm17.3}} & 17.8\textsubscript{\textpm3.8} & 34.6\textsubscript{\textpm2.0} \\
  GiGPO & \cellsec{69.0\textsubscript{\textpm4.1}} & \cellsec{57.9\textsubscript{\textpm5.3}} & \cellsec{53.8\textsubscript{\textpm3.0}} & \cellsec{48.3\textsubscript{\textpm6.0}} & \cellsec{28.7\textsubscript{\textpm4.5}} & 13.3\textsubscript{\textpm0.0} & \cellsec{46.9\textsubscript{\textpm2.8}} \\
  RewardFlow & \cellbest{83.3\textsubscript{\textpm8.2}} & \cellbest{84.2\textsubscript{\textpm5.3}} & \cellbest{69.2\textsubscript{\textpm1.2}} & \cellbest{74.7\textsubscript{\textpm10.0}} & 28.3\textsubscript{\textpm11.6} & \cellbest{44.4\textsubscript{\textpm3.8}} & \cellbest{66.1\textsubscript{\textpm1.8}} \\
  \midrule
  \rowcolor{regimeHelpful}\multicolumn{8}{@{}l}{\textit{Qwen2.5-3B-Instruct}}\\
  Base & 47.6\textsubscript{\textpm14.9} & 3.5\textsubscript{\textpm3.0} & 3.9\textsubscript{\textpm1.7} & 5.7\textsubscript{\textpm2.0} & 4.2\textsubscript{\textpm7.2} & 2.2\textsubscript{\textpm3.8} & 8.9\textsubscript{\textpm2.4} \\
  RLOO & 90.5\textsubscript{\textpm4.1} & 10.5\textsubscript{\textpm5.3} & 53.9\textsubscript{\textpm4.6} & \cellbest{44.8\textsubscript{\textpm6.0}} & \cellsec{53.2\textsubscript{\textpm8.2}} & 20.0\textsubscript{\textpm0.0} & 45.3\textsubscript{\textpm2.1} \\
  GRPO & \cellbest{95.2\textsubscript{\textpm4.1}} & 28.1\textsubscript{\textpm8.0} & 46.1\textsubscript{\textpm5.1} & 37.9\textsubscript{\textpm3.4} & 38.7\textsubscript{\textpm6.6} & 44.4\textsubscript{\textpm13.9} & 45.8\textsubscript{\textpm2.5} \\
  GiGPO & \cellsec{92.9\textsubscript{\textpm0.0}} & \cellsec{47.4\textsubscript{\textpm9.1}} & \cellsec{54.8\textsubscript{\textpm5.5}} & \cellsec{39.1\textsubscript{\textpm5.3}} & 40.9\textsubscript{\textpm7.9} & \cellbest{57.8\textsubscript{\textpm3.8}} & \cellsec{52.9\textsubscript{\textpm0.9}} \\
  RewardFlow & \cellsec{92.9\textsubscript{\textpm0.0}} & \cellbest{52.6\textsubscript{\textpm10.5}} & \cellbest{61.6\textsubscript{\textpm9.0}} & 33.3\textsubscript{\textpm5.3} & \cellbest{55.1\textsubscript{\textpm6.5}} & \cellsec{53.3\textsubscript{\textpm6.7}} & \cellbest{55.5\textsubscript{\textpm4.1}} \\
  \midrule
  \rowcolor{regimeHelpful}\multicolumn{8}{@{}l}{\textit{Qwen2.5-7B-Instruct}}\\
  Base & 28.6\textsubscript{\textpm12.4} & 12.3\textsubscript{\textpm3.0} & 16.4\textsubscript{\textpm7.5} & 8.0\textsubscript{\textpm4.0} & 12.1\textsubscript{\textpm10.5} & 0.0\textsubscript{\textpm0.0} & 12.8\textsubscript{\textpm4.3} \\
  RLOO & \cellbest{92.9\textsubscript{\textpm0.0}} & \cellsec{64.9\textsubscript{\textpm3.0}} & \cellbest{86.5\textsubscript{\textpm3.6}} & 62.1\textsubscript{\textpm3.4} & \cellsec{81.6\textsubscript{\textpm0.6}} & \cellbest{84.4\textsubscript{\textpm10.2}} & \cellsec{77.6\textsubscript{\textpm3.0}} \\
  GRPO & 83.3\textsubscript{\textpm4.1} & 47.4\textsubscript{\textpm0.0} & 73.0\textsubscript{\textpm6.3} & 57.5\textsubscript{\textpm5.3} & \cellsec{81.6\textsubscript{\textpm0.6}} & 60.0\textsubscript{\textpm17.6} & 66.4\textsubscript{\textpm4.9} \\
  GiGPO & \cellbest{92.9\textsubscript{\textpm0.0}} & 49.1\textsubscript{\textpm8.0} & \cellsec{80.8\textsubscript{\textpm7.2}} & \cellsec{65.5\textsubscript{\textpm6.0}} & 79.5\textsubscript{\textpm4.0} & 68.9\textsubscript{\textpm15.4} & 72.4\textsubscript{\textpm4.4} \\
  RewardFlow & \cellsec{90.5\textsubscript{\textpm4.1}} & \cellbest{87.7\textsubscript{\textpm16.9}} & 77.9\textsubscript{\textpm1.9} & \cellbest{69.0\textsubscript{\textpm0.0}} & \cellbest{89.8\textsubscript{\textpm3.4}} & \cellsec{80.0\textsubscript{\textpm6.7}} & \cellbest{80.5\textsubscript{\textpm4.1}} \\
  \bottomrule
  \end{tabular}%
\end{minipage}
\hfill
\begin{minipage}[t]{0.46\textwidth}
    \centering
    \begin{minipage}[t]{\textwidth}
        \vspace{2pt}
        \centering
        \includegraphics[width=\linewidth]{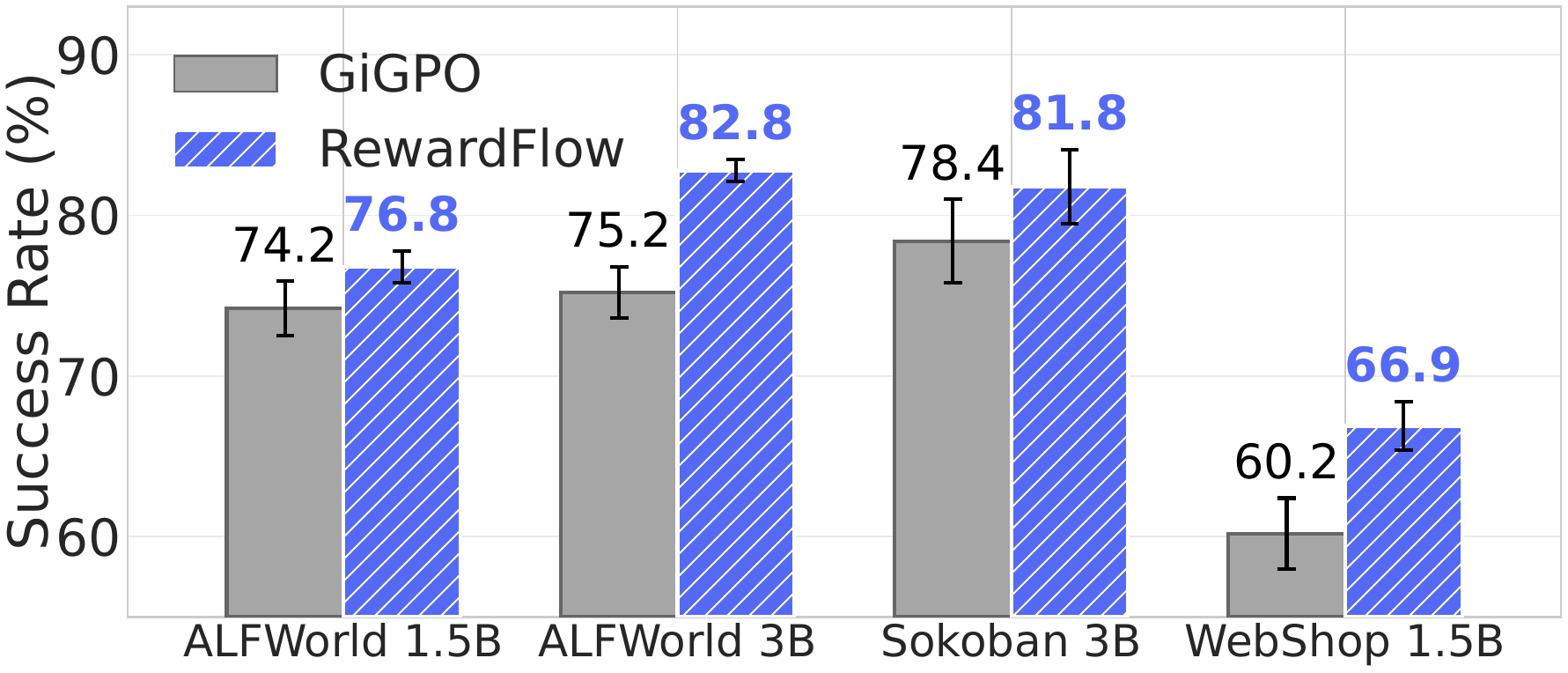}
        \vspace{-18pt}
        \captionof{figure}{GiGPO and RewardFlow with larger training budgets.}
        \label{fig:gigpo-matched}
    \end{minipage}
    \begin{minipage}{\textwidth}
        \centering
        \includegraphics[width=\linewidth]{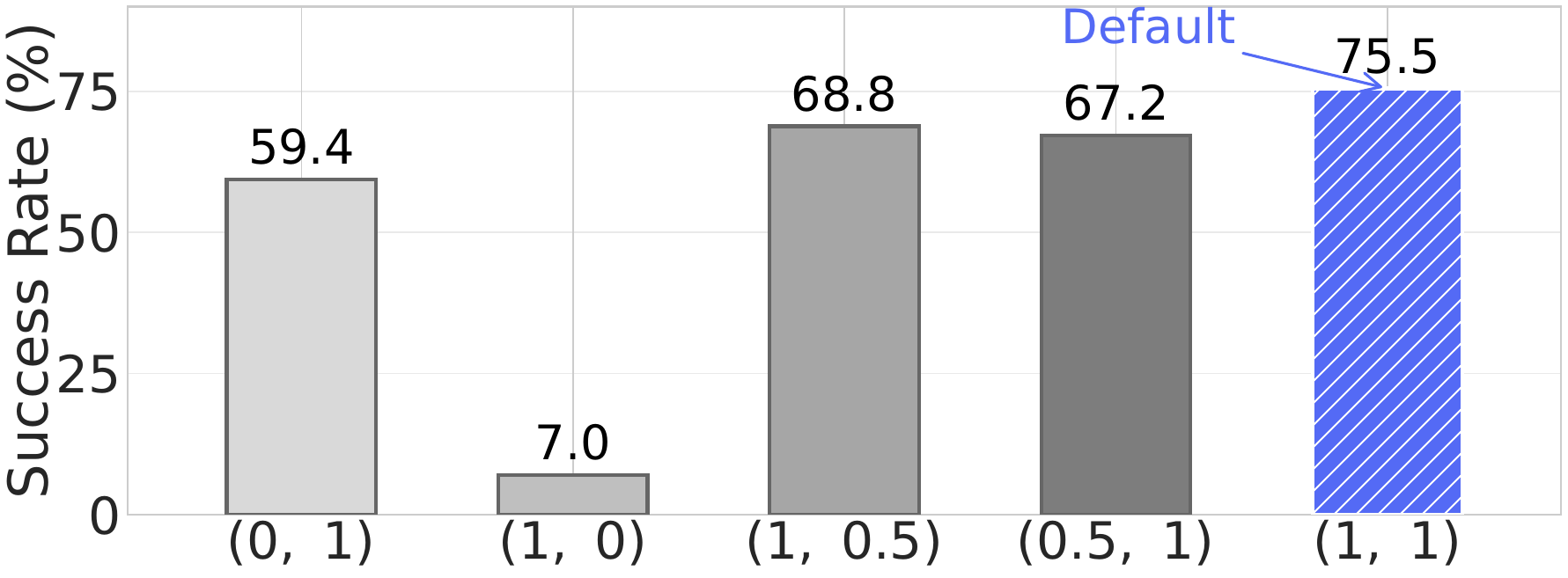}
        \vspace{-18pt}
        \captionof{figure}{Advantage-weight ablation $(\alpha_\text{act},\alpha_{\text{traj}})$ on ALFWorld (Qwen2.5-3B).}
        \label{fig:alpha-ablation}

    \end{minipage}

    \begin{minipage}[t]{\textwidth}
        \centering
        \includegraphics[width=\linewidth]{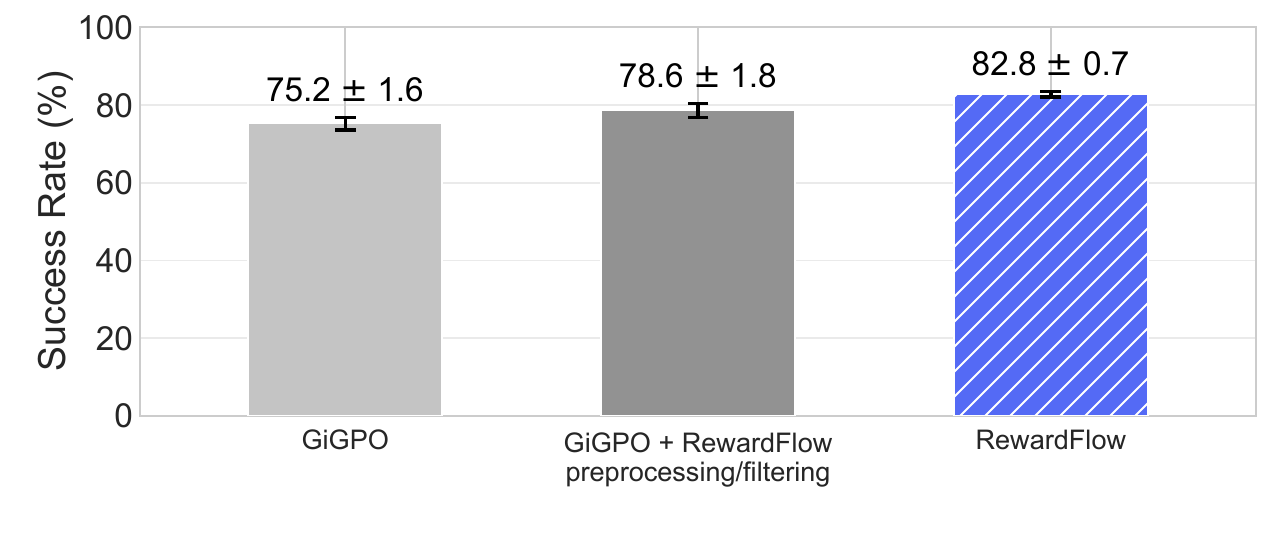}
        \vspace{-24pt}
        \captionof{figure}{Preprocessing control on ALFWorld with Qwen2.5-3B.}
        \label{fig:matched-preprocessing-control}
    \end{minipage}
\end{minipage}
\par\medskip
\begin{minipage}[t]{0.48\textwidth}
    \vspace{0pt}
    \centering
    \includegraphics[width=\linewidth]{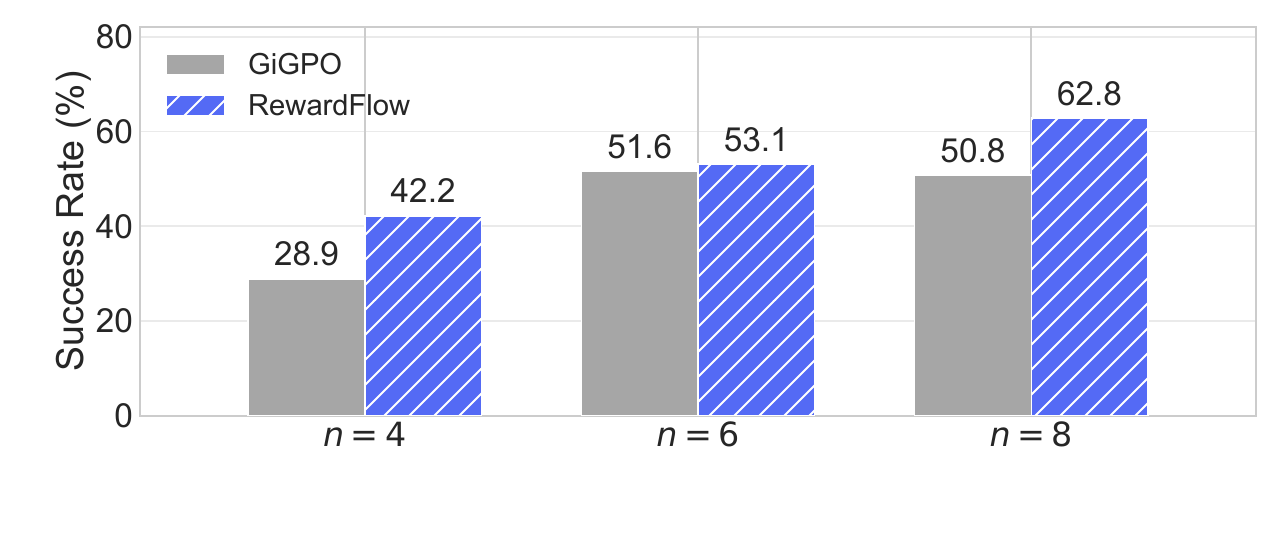}
    \vspace{-24pt}
    \captionof{figure}{GiGPO and RewardFlow with 4-8 rollouts budget on ALFWorld (Qwen2.5-1.5B).}
    \label{fig:exploration-robustness}
\end{minipage}
\hfill
\begin{minipage}[t]{0.48\textwidth}
    \vspace{0pt}
    \centering
    \includegraphics[width=\linewidth]{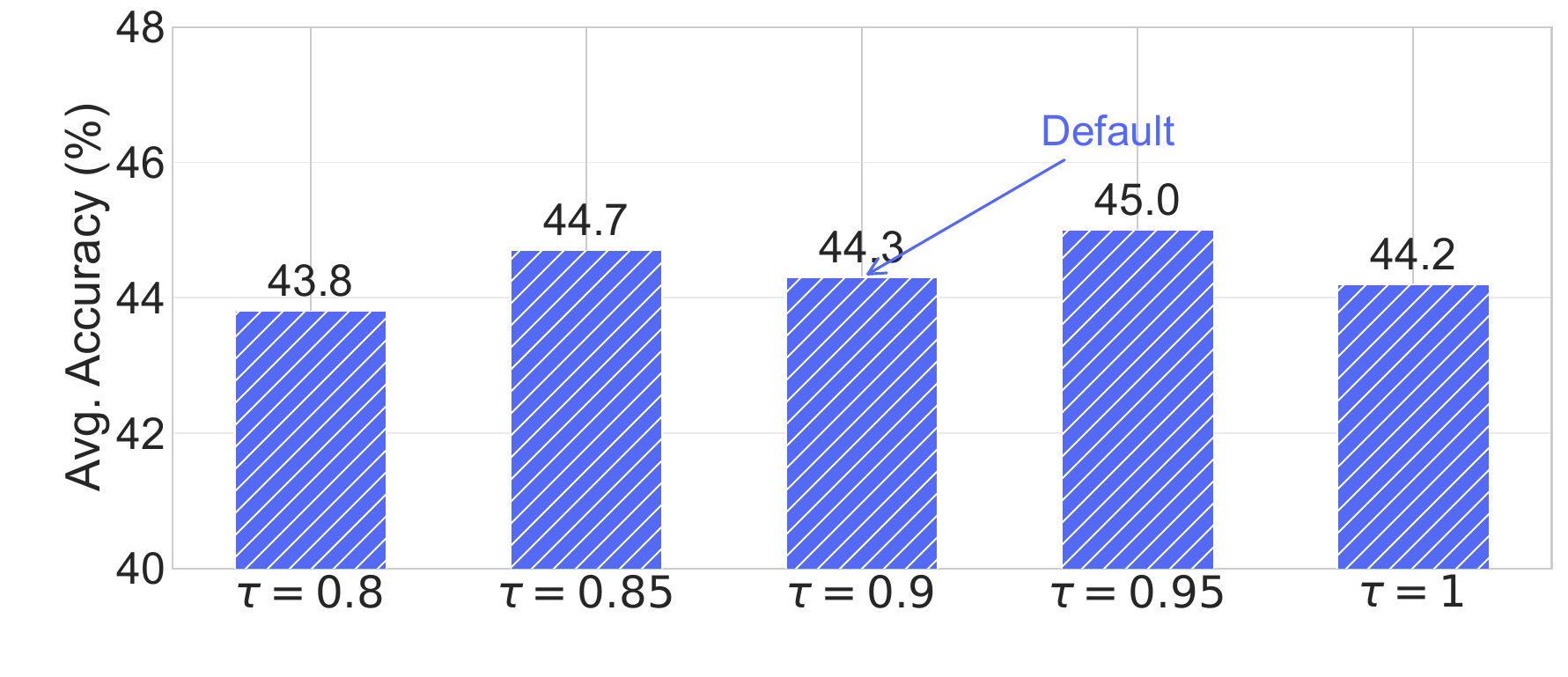}
    \vspace{-24pt}
    \captionof{figure}{Sensitivity to cosine-similarity threshold $\tau$ on search tasks (Qwen2.5-3B).}
    \label{fig:threshold-sensitivity}
\end{minipage}
\vspace{-12pt}
\end{figure}

\subsection{Ablation Studies}
\label{ssec:effectiveness}
\label{ssec:ablations}

\begin{wraptable}{r}{0.47\textwidth}
\vspace{-10pt}
\mainresultcolors
\centering
\captionsetup{font=small}
\caption{Propagation strategy ablation. We report success rates (\%) on ALFWorld with Qwen2.5-1.5B and Sokoban with Qwen2.5-VL-3B. }
\label{tab:propagation-ablation}
\label{tab:ablation-propagation-strategy}
\label{tab:bfs-vs-ppr}
\fontsize{8}{10}\selectfont
\renewcommand{\arraystretch}{1.15}
\setlength{\tabcolsep}{3pt}
\begin{tabularx}{\linewidth}{@{}>{\raggedright\arraybackslash}Xcc@{}}
\toprule
\rowcolor{regimeHelpful}
\textbf{Propagation strategy} & \textbf{ALFWorld} & \textbf{Sokoban} \\
\midrule
Mean aggregation & \cellsec{62.5} & 33.6 \\
Personalized PageRank (PPR) & 52.3 & \cellsec{35.2} \\
Shortest-path (BFS / Max) & \cellbest{62.8} & \cellbest{49.5} \\
\bottomrule
\end{tabularx}
\end{wraptable}

We examine how reward estimation and the two advantage levels contribute to learning:

\begin{itemize}[leftmargin=*]
    \item \textbf{Action-level reward estimation improves performance beyond preprocessing.} GiGPO's success increases from $75.2\%$ to $78.6\%$ using RewardFlow's state preprocessing method and invalid-action filtering. RewardFlow reaches $82.8\%$, a further $4.2\%$ absolute improvement (Fig.~\ref{fig:matched-preprocessing-control}). Full details see Appendix~\ref{app: matched preprocessing control}.
\end{itemize}
\par\WFclear
\begin{itemize}[leftmargin=*,topsep=0pt]
    \item \textbf{Shortest-path propagation achieves the highest success rates.} It outperforms PPR in both environments. Compared with mean aggregation, success is similar on ALFWorld ($62.8\%$ vs.\ $62.5\%$) but increases from $33.6\%$ to $49.5\%$ on Sokoban (Tab.~\ref{tab:propagation-ablation}). Full details are provided in Appendix~\ref{propagation strategy} and Appendix~\ref{app: bfs vs ppr}.
    \item \textbf{Action- and trajectory-level advantages are complementary.} Combining both at unit weight yields $75.5\%$ success, versus $59.4\%$ with trajectory advantages alone and $7.0\%$ with action advantages alone. Halving the action or trajectory weight lowers success to $67.2\%$ or $68.8\%$, respectively (Fig.~\ref{fig:alpha-ablation}). Full details are provided in Appendix~\ref{app: alpha sensitivity}.
\end{itemize}

\subsection{Robustness and Training Efficiency}
\label{ssec:robustness}
\label{ssec: efficiency}

We examine robustness to budget, distribution, and state-grouping changes, and measure runtime:

\begin{itemize}[leftmargin=*]
    \item \textbf{RewardFlow outperforms GiGPO across training and rollout budgets.} RewardFlow stays ahead in all four larger-training-budget comparisons by $2.6$-$7.6\%$ in success rate and across rollout budgets of $4$, $6$, and $8$ per training step, including $42.2\%$ versus $28.9\%$ with $4$ rollouts (Figs.~\ref{fig:gigpo-matched} and~\ref{fig:exploration-robustness}). Full details are provided in Appendix~\ref{app: gigpo-matched} and~\ref{app: exploration-ablation}.
    \item \textbf{Performance improvements extend to unseen environments and datasets.} RewardFlow ranks first in overall success in unseen ALFWorld rooms and layouts at all three scales, including $66.1\%$ versus GiGPO's $46.9\%$ at 1.5B. On OOD search datasets, it exceeds GiGPO on four of five at 3B and all five at 7B (Tabs.~\ref{tab:deepresearch}-\ref{tab:ood}). Full details are provided in Appendix~\ref{app: detailed experimental setting}.
    \item \textbf{Search accuracy remains stable across cosine-similarity thresholds.} Across semantic similarity thresholds $0.80$, $0.85$, $0.90$, $0.95$, and $1.00$ for encoded observations using the embedding model, RewardFlow's accuracy stays within $43.8$-$45.0\%$, with $44.3\%$ at the default $0.90$. Higher thresholds produce finer state groupings (Fig.~\ref{fig:threshold-sensitivity}). Full details are provided in Appendix~\ref{app: aggregation sensitivity}.
    \item \textbf{RewardFlow's computation adds little training overhead.} The introduced components of RewardFlow, state graph construction and action-level reward estimation, take overall $1.87$, $1.47$, and $2.39$ seconds per training step on ALFWorld, WebShop, and Sokoban, respectively, without additional model calls (Fig.~\ref{fig:rewardflow_teaser}, right). Full details are provided in Appendix~\ref{app: training efficiency}.

\end{itemize}

\section{Conclusion}
\label{sec:conclusion}

We present RewardFlow, a method for estimating action-level rewards from state-graph topology. By connecting sampled trajectories through shared states, RewardFlow defines state potentials from shortest-path distances to observed success and uses their differences as action-level rewards. These graph-derived rewards complement trajectory-level supervision to improve performance across household, shopping, visual puzzle, and search tasks without a learned process scorer. Consistent improvements and robustness across evaluation conditions support RewardFlow as a reliable source of action-level supervision. Tool-integrated math reasoning experiments provide preliminary evidence of applicability to open-ended tasks, as detailed in Appendix~\ref{app: math tir supplement}. Future work will develop more reliable state identification across reasoning paths and evaluate this extension more broadly.

\begin{ack}
XF, ZKZ, and BH were supported by Guangdong Basic and Applied Basic Research Foundation No. 2024A1515012399, RGC Research Fellow Scheme No. RFS2627-2S02, RGC Young Collaborative Research Grant No. C2005-24Y, RGC General Research Fund No. 12202026, RGC General Research Fund No. 12200725, NSFC Major Research Plan No. 92570109, NSFC General Program No. 62376235, TCL Collaborative Research Fund, Guangdong and Hong Kong Universities ``1+1+1'' Joint Research Collaboration Scheme, HKBU Faculty Niche Research Areas No. RC-FNRA-IG/25-26/SCI/05, and HKBU CSD Departmental Incentive Scheme.
JCY is supported by National Natural Science Foundation of China (No. 62676245), the Lingang Laboratory (Grant No. LGL-8889-01), STCSM (No. 22DZ2229005) and 111 plan (No. BP0719010).
MKN is supported in part by the Department of Science and Technology of Guangdong Province (GDSTC): Guangdong and Hong Kong Universities ``1+1+1'' Joint Research Collaboration Scheme, the National Key Research and Development Program of China (Grant No. 2024YFE0202900), and Research Grants Council (RGC) of Hong Kong under the General Research Fund (GRF) (Grant No. 12300125).
\end{ack}

\bibliography{bib}
\bibliographystyle{plainnat}
\clearpage
\onecolumn
\raggedbottom
\appendix 

\etocdepthtag.toc{mtappendix}
\etocsettagdepth{mtchapter}{none}
\etocsettagdepth{mtappendix}{subsection}
\renewcommand{\contentsname}{Appendix}
\tableofcontents 
\clearpage
\section{Related Work}
\label{sec:related_work}

\paragraph{RL for LLM Reasoning.}
Policy optimization for LLMs depends both on the feedback available and on how that feedback is converted into an update. PPO~\citep{schulman2017proximal} supplies a clipped policy objective, while GRPO~\citep{shao2024deepseekmath} and RLOO~\citep{ahmadian2024back,kool2019buy} estimate baselines from groups of sampled outputs. Refinements such as DAPO~\citep{yu2025dapo} and Dr.~GRPO~\citep{liu2025understanding} address aspects of the training objective and sampling procedure. CoDaPO~\citep{zhou2026codapo} adapts policy-update weights and question sampling to rollout confidence and empirical difficulty. Preference-based approaches, including DPO~\citep{rafailov2023direct} and SimPO~\citep{meng2024simpo}, provide a different route for learning from comparisons. Self-supervised RL can also derive reward signals from agreement among model-generated answers without annotated solutions~\citep{zhang2026co}. Verifiable feedback also supports policy optimization in scientific domains such as molecular optimization~\citep{li2026repo,li2026learning}. RewardFlow estimates action-level rewards from the state-transition structure revealed across sampled trajectories; its graph-derived advantages are used in a clipped policy update.

\paragraph{RL for Agentic Scenarios.}
Agentic reasoning spans retrieval-based tasks~\citep{jin2025search,song2025r1,zheng2025deepresearcher}, reasoning with code-execution tools~\citep{feng2025retool,li2025torl,jiang2025verltool,zhou2025alphaapollo}, reasoning under incomplete information~\citep{zhou2025from}, and multi-agent debate~\citep{tian2025multi,yi2025from}. Such tasks require the policy to respond to feedback over several turns. Evaluation and analysis tools provide information on task outcomes, reasoning processes, tool usage, and resource consumption~\citep{zhou2026alphadiana,zhou2026landscape}, which can guide RL optimization to improve policy performance. Rollout construction is one way to improve learning: ARPO~\citep{dong2025agentic} guides exploration with entropy, SimpleTIR~\citep{xue2025simpletir} filters void turns, and TSR~\citep{djuhera2026tsr} uses training-time tree search. These choices shape the collected experience, while reward estimation guides learning from it. GiGPO~\citep{feng2025group} and HGPO~\citep{he2026hierarchyofgroups} use repeated states to organize process-level comparisons. Concurrent work GraphGPO~\citep{cheng2026graphgpo} also studies cross-trajectory graphs for reward estimation. RewardFlow estimates action rewards from graph-derived potential differences and analyzes their approximation error under incomplete trajectory coverage.

\paragraph{Process Reward Estimation for LLM Reasoning.}
Process reward models provide supervision below the level of a complete response. Work on mathematical reasoning uses human step labels~\citep{lightman2023let}, rollout-derived labels~\citep{wang2024math,luo2024improve}, or implicit process rewards~\citep{cui2025process}. For agents, judging an intermediate action also requires its environment context and its relation to a future goal. AgentPRM~\citep{xi2026agentprm}, OPRL~\citep{liu2025agentic}, and Web-Shepherd~\citep{chae2025web} investigate process feedback in interactive settings; process rewards can also guide actions at inference time~\citep{hu2025guiding}. RewardFlow recomputes action rewards from current observed state transitions and terminal outcomes without training a separate process scorer.

\section{Further Method Details}
\label{app:further_details}

\subsection{State Preprocessing}
\label{app: state preprocessing}

State preprocessing $f$ assigns graph-node identities to the agent-side representations defined in Sec.~\ref{sec: preliminary}. Advantage normalization refers to reward centering and scaling in Sec.~\ref{ssec: advantage estimation}. Retained history distinguishes observations that look alike but require different actions; embedding-based matching joins related text into shared nodes.

\paragraph{ALFWorld: retaining task-relevant history.}
After an agent cleans an apple, the next observation may still say that it is carrying an apple without recording that change. We retain inventory, held-object status, transformative actions such as cleaning or heating, and placement progress in pick-two tasks. These features distinguish interaction histories that share the same current observation.

\paragraph{Search tasks: semantic state grouping.}
Different search results can describe related information without matching as strings. We use E5 text embeddings~\citep{wang2022text} to aggregate observations that meet a cosine-similarity threshold (default 0.9). Fig.~\ref{fig:threshold-sensitivity} and Appendix~\ref{app: aggregation sensitivity} evaluate sensitivity to this threshold.

Within each task's trajectory group, retrieved observations are processed in input order. Each state group has a fixed representative observation. An observation joins the first existing group whose representative has embedding cosine similarity at least $\tau$ to it. If no representative meets the threshold, the observation starts a new group and becomes its representative. Members of a group share the same graph-node identity.

The following schematic summarizes these task-specific grouping rules:
\begin{equation}
f(s) =
\begin{cases}
\text{cluster representative}(s) & \text{if } s \text{ meets the semantic-similarity criterion}, \\
s \oplus \text{critical features} & \text{if } s \text{ is ambiguous}, \\
s & \text{otherwise}.
\end{cases}
\label{eq:state_preprocessing}
\end{equation}
The right-hand sides specify canonical node keys; critical features are extracted from the observation and history already retained in $s$. Occurrences with the same key map to one node $\hat s_v$. For the identity interface $f(s)=s$, the representation itself is the key; node identifiers remain independent of trajectory and time indices.

\label{app:state_preprocessing}

\paragraph{Task-specific interfaces, shared graph operations.}
The map $f(\cdot)$ reflects the information exposed by each benchmark. Sokoban uses exact board configurations, $f(s)=s$, while text environments can require retained history or approximate matching. Graph construction, distance computation, and advantage formulas are shared; validity handling and training configurations are specified for each environment.

Invalid-action handling is a parallel interface choice. ALFWorld's ``Nothing happens'' reports an ineffective command. Filtering prevents this generic feedback from becoming a shared node across unrelated situations. Invalid actions in WebShop and Sokoban can induce state-transition self-loops, which preserve BFS distance and have zero potential difference.

\begin{table}[htbp]
\centering
\small
\caption{State-construction interfaces for the four main benchmarks and the supplementary Math-TIR evaluation.}
\label{tab:state-construction-interfaces}
\begin{tabular}{p{0.16\linewidth}p{0.43\linewidth}p{0.30\linewidth}}
\toprule
Environment & Node identity / retained information & Role in graph construction \\
\midrule
ALFWorld & Observation plus inventory, object status, and prior placement progress & Retains task-relevant changes across turns. \\
WebShop & Page observation with persistent selected attributes & Distinguishes product configurations. \\
Sokoban & Exact board configuration, $f(s)=s$ & Joins identical board layouts. \\
Search tasks & E5 similarity aggregation of observations & Connects semantically related retrieved information. \\
Math-TIR & Exact interaction histories at reasoning/tool-action boundaries & Shares nodes along common interaction prefixes. \\
\bottomrule
\end{tabular}
\end{table}

Here $s$ in $f(s)$ is the state representation defined in Sec.~\ref{sec: preliminary}, including the observation and retained interaction history. Within each task's trajectory group, the pipeline maps these representations to nodes, adds valid directed state transitions, marks observed successful endpoints, runs reverse BFS, and looks up node potentials for each retained trajectory occurrence.

\paragraph{Node construction for tool-integrated reasoning.}
In mathematical tool-integrated reasoning (Math-TIR), a model alternates written reasoning with Python execution. We call the supplementary node-construction procedure \emph{Structured-Decoding}: generation branches at reasoning and tool-action boundaries, and the exact interaction history up to each boundary identifies a node. Shared prefixes connect the sampled branches. Executed Python actions and the resulting observed states supply directed edges; execution requires tool state consistent with the interaction history. RewardFlow and Graph-GRPO use the same branching rules, and RewardFlow applies its graph-potential formula to the resulting graph. A frozen LLM provides a separate state-construction option by mapping histories to shared node identities.

\paragraph{Recurrence and graph diagnostics.}
\label{app:graph-diagnostics}

The method assigns rewards from paths to observed success. These diagnostics describe node recurrence, connectivity, and retained state transitions.

\paragraph{Repeated visits create opportunities for sharing.}
Fig.~\ref{fig: total_vs_unique} shows repeated states and actions across sampled trajectories, providing opportunities to connect their paths.

\paragraph{Connectivity and reachability.}
Fig.~\ref{fig: invalid_action} (left) illustrates state-transition filtering. Fig.~\ref{fig: graph_analysis} (right) shows success reachability and the incoming and outgoing state transitions at graph nodes. Reachability supplies the distances used to define state potentials; degree describes how trajectories connect.

\paragraph{Coverage conditions.}
Prop.~\ref{prop:graph_approximation} in Sec.~\ref{ssec:theoretical_justification} relates sampled graph distances to environment distances. Appendix~\ref{app:graph-approximation-proof} gives the cumulative-sampling conditions used by the potential analysis.

\subsection{Filtering Invalid Actions}
\label{app: invalid action filtering}

A generated response can fail at two different levels: it may not parse as an action, or its parsed action may be inadmissible in the current environment. The retention predicate $\textsc{valid}(s_t^{(i)},a_t^{(i)},s_{t+1}^{(i)})$ checks the state transition using environment-specific validity and state-change feedback. The set $\mathcal I_{\mathrm{valid}}$ in Sec.~\ref{ssec: state_graph_construction} contains exactly the indices for which this predicate equals $1$.

In ALFWorld, an ineffective command can return generic feedback without creating a new environment state. Treating that string as node identity would merge unrelated situations and permit spurious paths. In the example below, ``invalid state'' labels generic feedback excluded from node construction.

\begin{tcolorbox}[colback=grey!10, colupper=black, colframe=grey!40!black, title=An example of invalid actions in ALFWorld, fontupper=\small, width=\textwidth]

STEP: 1\\
STATE: -= Welcome to TextWorld, ALFRED! =-\\

You are in the middle of a room. Looking quickly around you, you see a cabinet 10, a cabinet 9, a cabinet 8, a cabinet 7, a cabinet 6, a cabinet 5, a cabinet 4, a cabinet 3, a cabinet 2, a cabinet 1, a coffeemachine 1, a countertop 1, a diningtable 1, a drawer 2, a drawer 1, a fridge 1, a garbagecan 1, a microwave 1, a sinkbasin 1, a stoveburner 4, a stoveburner 3, a stoveburner 2, a stoveburner 1, and a toaster 1.\\

Your task is to: put a cool apple in microwave.\\
ACTION: go to cabinet 1\\
---\\
STEP: 2\\
STATE: You arrive at cabinet 1. The cabinet 1 is closed.\\
ACTION: go to cabinet 1 [INVALID ACTION]\\
---\\
STEP: 3\\
STATE: Nothing happens. [INVALID STATE]\\
ACTION: open cabinet 1\\
---\\
STEP: 4\\
STATE: You open the cabinet 1. The cabinet 1 is open. In it, you see nothing.\\

\end{tcolorbox}

We exclude this generic feedback from effective propagation. The filtering ablation in Tab.~\ref{tab:alfworld-ablation} reports its contribution to task performance alongside the other interface components.

\subsection{From Trajectory States to Token-Level Advantages}
\label{app:notation-chain}

This section uses the same notation as the main text throughout the state-to-token computation. Trajectory indices are superscripts $(i)$; action turns and generated tokens are subscripts $t$ and $k$. A graph node has an independent identifier $v$. Terminal outcomes, node potentials, occurrence-level lookups, and action rewards are distinct quantities. The formulas use binary terminal outcomes.

In Fig.~\ref{fig:rewardflow_overview}, $j$ indexes action occurrences grouped at the same graph node $\hat s$. Their rewards are written as $\tilde r_{\hat s,j}$, with group mean $\mu(\hat s)$ and normalization scale $\sigma(\hat s)$; the resulting action advantage is $A_{\mathrm{act},\hat s,j}$.

\begin{longtable}{p{0.35\linewidth}p{0.59\linewidth}}
\caption{Unified notation for states, graph nodes, and the state-to-token computation.}\label{tab:notation-chain}\\
\toprule
Symbol & Domain and meaning \\
\midrule
\endfirsthead
\toprule
Symbol & Domain and meaning (continued) \\
\midrule
\endhead
\bottomrule
\endfoot
$s_t^{(i)}\in\mathcal S$; $a_t^{(i)}\in\mathcal A$ & State-representation and action occurrences at turn $t$ of trajectory $i$. \\
$\hat s_v\in\mathcal V$; $f(s_t^{(i)})$ & Canonical graph node and the node assigned to occurrence $(i,t)$. The identifier $v$ is not a time index. \\
$\mathcal G=(\mathcal V,\mathcal A_{\mathcal G},\mathcal E)$ & Sampled graph: nodes, observed action labels, and retained directed state transitions. \\
$\mathcal V_{\mathrm{succ}}$ & Observed successful terminal nodes in $\mathcal G$. \\
$d_{\mathcal G}(\hat s_v)$ & Shortest directed hop distance to $\mathcal V_{\mathrm{succ}}$, or $\infty$ if unreachable. \\
$d_{\mathrm{env}}^*(s)$ & Shortest hop distance from $f(s)$ to a successful class in the reference graph $\mathcal G_{\mathrm{env}}^f$. \\
$r_{\mathrm{env}}(s,a,s')$ & Environment reward for an observed state transition. \\
$r_{\mathrm{term}}^{(i)}\in\{0,1\}$ & Observed binary terminal success of trajectory $i$. \\
$\Phi_{\mathcal G}(\hat s_v)$ & Graph-node potential, computed from sampled paths. \\
$\phi_t^{(i)}$ & Node potential looked up at occurrence $(i,t)$. \\
$\tilde r_t^{(i)}$ & Raw graph-derived reward of the retained action at $(i,t)$. \\
$\mathcal I_v$ & Indices of retained action occurrences whose source maps to $\hat s_v$. \\
$\mu(\hat s_v)$; $\sigma(\hat s_v)$ & State-group mean reward and positive normalization scale; singleton groups use their sole reward as the mean and scale $1$. \\
$\mu_{\mathrm{term}}$; $\sigma_{\mathrm{term}}$ & Task-group baseline computed from sampled terminal outcomes and positive normalization scale. \\
$A_{\mathrm{act},t}^{(i)}$; $A_{\mathrm{traj}}^{(i)}$ & State-conditioned action advantage and trajectory advantage. Neither is separately estimated per token. \\
$o_t^{(i)}$; $o_{t,k}^{(i)}$ & Generated token sequence at turn $t$ and its $k$-th token. \\
$A_{t,k}^{(i)}$; $\rho_{t,k}^{(i)}$ & Complementary advantage broadcast to token $k$ and its policy importance ratio. \\
$\varepsilon_{\mathrm{norm}}$; $\varepsilon_{\mathrm{clip}}$ & Normalization stabilizer and PPO clipping parameter, respectively. \\
\end{longtable}

\paragraph{Graph nodes and occurrence lookup.}
Several occurrences can map to one node: $f(s_1^{(1)})=f(s_2^{(2)})=f(s_t^{(i)})=\hat s_v$. Each occurrence retains its original $(i,t)$ identity for reward assignment. For a fixed state graph $\mathcal G$, define
\begin{equation}
\Phi_{\mathcal G}(\hat s_v)=
\begin{cases}
\gamma^{d_{\mathcal G}(\hat s_v)}, & d_{\mathcal G}(\hat s_v)<\infty,\\
0, & d_{\mathcal G}(\hat s_v)=\infty.
\end{cases}
\label{eq:app-graph-potential}
\end{equation}
Then $\phi_t^{(i)}=\Phi_{\mathcal G}(f(s_t^{(i)}))$ and $\tilde r_t^{(i)}=\phi_{t+1}^{(i)}-\phi_t^{(i)}$. This is a lookup through $f$, not a single-valued inverse. The preimage of a node is a set of representations; $\mathcal I_v$ additionally distinguishes repeated occurrences of those representations. Only sampled-graph distances determine these rewards. For a generic observed state transition, $\tilde r(s,a,s')=\Phi_{\mathcal G}(f(s'))-\Phi_{\mathcal G}(f(s))$ is the same action-reward definition with the successor explicit.

\paragraph{Occurrence groups and normalization.}
Let
\[
\mathcal I_v=\{(i,t):f(s_t^{(i)})=\hat s_v,\ (i,t)\text{ is retained}\}.
\]
Repeated visits are counted separately, rather than deduplicating action strings. For non-singleton groups, $\mu(\hat s_v)$ is the mean action reward and $\sigma(\hat s_v)>0$ is the positive factor obtained by applying $F_{\mathrm{norm}}$ to their reward multiset, as defined in Sec.~\ref{ssec: advantage estimation}. For singleton groups, $\mu$ equals the sole action reward and $\sigma=1$, so the action advantage is zero.

\begin{samepage}
With the same nonzero-denominator convention for terminal-outcome normalization,
\begin{align}
A_{\mathrm{act},t}^{(i)}&=\frac{\tilde r_t^{(i)}-\mu(f(s_t^{(i)}))}{\sigma(f(s_t^{(i)}))},&
A_{\mathrm{traj}}^{(i)}&=\frac{r_{\mathrm{term}}^{(i)}-\mu_{\mathrm{term}}}{\sigma_{\mathrm{term}}},\\
A_{t,k}^{(i)}&=\alpha_\text{act}A_{\mathrm{act},t}^{(i)}+\alpha_{\mathrm{traj}}A_{\mathrm{traj}}^{(i)}.
\label{eq:app-token-advantage}
\end{align}
\end{samepage}
A constant-outcome group contributes zero trajectory advantage. The complementary advantage is broadcast to the valid generated tokens of the action, whereas $\rho_{t,k}^{(i)}$ depends on the conditional probability of each token under the current and behavior policies. Thus only the final broadcast and probability ratio carry a token index; graph construction and action-level reward estimation operate at action-turn boundaries.

\subsection{Prompt Templates}
\label{app:prompts}
These are the prompts used to sample trajectories during RL training. The no-history variant starts a trajectory; the history variant supplies context for subsequent turns. Prompt contents are reproduced without modification.

\begin{tcolorbox}[colback=grey!10, colupper=black, colframe=grey!40!black, title=Prompt of ALFWorld without history, fontupper=\small, width=\textwidth]

You are an expert agent operating in the ALFRED Embodied Environment.\\
Your current observation is: \{current\_observation\}\\
Your admissible actions of the current situation are: [\{admissible\_actions\}].
\\\\
Now it's your turn to take an action.\\
You should first reason step-by-step about the current situation. This reasoning process MUST be enclosed within  $<$think$>$  $<$/think$>$  tags. \\
Once you've finished your reasoning, you should choose an admissible action for current step and present it within  $<$action$>$  $<$/action$>$  tags.

\end{tcolorbox}

\begin{tcolorbox}[colback=grey!10, colupper=black, colframe=grey!40!black, title=Prompt of ALFWorld with history, fontupper=\small, width=\textwidth]

You are an expert agent operating in the ALFRED Embodied Environment. Your task is
to: \{task\_description\}
\\
Prior to this step, you have already taken \{step\_count\} step(s). Below are the most
recent \{history\_length\} observations and the corresponding actions you took:
\{action\_history\}
\\
You are now at step \{current\_step\} and your current observation is:
\{current\_observation\}
\\
Your admissible actions of the current situation are: [\{admissible\_actions\}].
\\\\
Now it’s your turn to take an action.\\
You should first reason step-by-step about the current situation. This reasoning
process MUST be enclosed within $<$think$>$ $<$/think$>$ tags.\\
Once you’ve finished your reasoning, you should choose an admissible action for
current step and present it within $<$action$>$ $<$/action$>$ tags.

\end{tcolorbox}

\begin{tcolorbox}[colback=grey!10, colupper=black, colframe=grey!40!black, title=Prompt of WebShop without history, fontupper=\small, width=\textwidth]

You are an expert autonomous agent operating in the WebShop e-commerce environment. \\
Your task is to: \{task\_description\}.\\
Your current observation is: \{current\_observation\}.\\
Your admissible actions of the current situation are: \\
\verb|[|\\
\{available\_actions\}\\
\verb|]|. 
\\\\
Now it's your turn to take one action for the current step.\\
You should first reason step-by-step about the current situation, then think carefully which admissible action best advances the shopping goal. This reasoning process MUST be enclosed within $<$think$>$ $<$/think$>$ tags. \\
Once you've finished your reasoning, you should choose an admissible action for current step and present it within $<$action$>$ $<$/action$>$ tags.

\end{tcolorbox}

\begin{tcolorbox}[colback=grey!10, colupper=black, colframe=grey!40!black, title=Prompt of WebShop with history, fontupper=\small, width=\textwidth]

You are an expert autonomous agent operating in the WebShop e-commerce environment.\\
Your task is to: \{task\_description\}.\\
Prior to this step, you have already taken \{step\_count\} step(s). Below are the most recent \{history\_length\} observations and the corresponding actions you took: \{action\_history\}\\
You are now at step \{current\_step\} and your current observation is: \{current\_observation\}.\\
Your admissible actions of the current situation are: 
[
\{available\_actions\}
].
\\\\
Now it's your turn to take one action for the current step.\\
You should first reason step-by-step about the current situation, then think carefully which admissible action best advances the shopping goal. This reasoning process MUST be enclosed within $<$think$>$ $<$/think$>$ tags.\\ 
Once you've finished your reasoning, you should choose an admissible action for current step and present it within $<$action$>$ $<$/action$>$ tags.

\end{tcolorbox}

\begin{tcolorbox}[colback=grey!10, colupper=black, colframe=grey!40!black, title=Prompt of Sokoban, fontupper=\small, width=\textwidth]

You are an expert agent operating in the Sokoban environment. Your goal is to push all the boxes onto the target spots. Once all boxes are on the targets, you win!\\

\# Rules\\
You can only push boxes. You can't pull them, so plan ahead to avoid getting stuck.\\
You can't walk through or push boxes into walls.\\
To avoid traps, do not push boxes into corners or against walls where they can't be moved again.\\

\# Visual Elements in the Image:\\
Character: A small, green alien-like figure with two antennae and black eyes. It represents you.\\
Box: A yellow crate marked with an orange "X" across its front. It is the box you need to push.\\
Target: A black tile outlined in red, with a small red diamond shape in the center. It marks the destination where a box should be pushed.\\

\# Current Step\\
Your current observation is shown in the image: $<$image$>$\\
Your admissible actions are ["up", "down", "left", "right"].\\

Now it's your turn to make a move (choose ONE action only for the current step).\\
You should first reason step-by-step about the current situation --- observe the positions of boxes and targets, plan a path to push a box toward a target, and avoid traps like corners or walls. This reasoning process MUST be enclosed within $<$think$>$ $<$/think$>$ tags. \\
Once you've finished your reasoning, you should choose an admissible action for current step and present it within $<$action$>$ $<$/action$>$ tags.

\end{tcolorbox}

\begin{tcolorbox}[colback=grey!10, colupper=black, colframe=grey!40!black, title=Prompt for search tasks without history, fontupper=\small, width=\textwidth]
You are an expert agent tasked with answering the given question step-by-step.\\
Your question: \{task\_description\}\\

Now it's your turn to respond for the current step.\\
You should first conduct reasoning process. This process MUST be enclosed within $<$think$>$ $<$/think$>$ tags. \\
After completing your reasoning, choose only one of the following actions (do not perform both):\\
(1) If you find you lack some knowledge, you can call a search engine to get more external information using format: $<$search$>$ your query $<$/search$>$.\\
(2) If you have enough knowledge to answer the question confidently, provide your final answer within $<$answer$>$ $<$/answer$>$ tags, without detailed illustrations. For example, $<$answer$>$Beijing$<$/answer$>$.
\end{tcolorbox}

\begin{tcolorbox}[colback=grey!10, colupper=black, colframe=grey!40!black, title=Prompt for search tasks with history, fontupper=\small, width=\textwidth]

You are an expert agent tasked with answering the given question step-by-step.\\
Your question: \{task\_description\}\\

Prior to this step, you have already taken \{step\_count\} step(s). Below is the interaction history where $<$search$>$ $<$/search$>$ wrapped your past search queries and $<$information$>$ $<$/information$>$ wrapped the corresponding search results returned by the external search engine. History:\\
\{memory\_context\}\\

Now it's your turn to respond for the current step.\\
You should first conduct reasoning process. This process MUST be enclosed within $<$think$>$ $<$/think$>$ tags. \\
After completing your reasoning, choose only one of the following actions (do not perform both):\\
(1) If you find you lack some knowledge, you can call a search engine to get more external information using format: $<$search$>$ your query $<$/search$>$.\\
(2) If you have enough knowledge to answer the question confidently, provide your final answer within $<$answer$>$ $<$/answer$>$ tags, without detailed illustrations. For example, $<$answer$>$Beijing$<$/answer$>$.

\end{tcolorbox}

\section{Proofs}
\label{app:proofs}

We first establish how sampled state transitions approximate the POMDP's state-transition structure, then derive support-potential and action-reward guarantees that permit stochastic state transitions. We also distinguish this potential from expected discounted value. Throughout, $s$ is an agent-side information state, $\hat s_v=f(s)$ is its shared graph state, and $\gamma^{\infty}=0$.

\subsection{Information States and State-Transition-Preserving Grouping}

Fix a task $x$. Let $y_t$ denote the observation at turn $t$ and $h_t=(x,y_0,a_0,\ldots,y_t)$ its observation-action history. Write the agent-side representation from Sec.~\ref{sec: preliminary} as $s_t=\psi(h_t)$, where $\psi$ retains the current observation and task-relevant history. The policy-environment interaction induces a distribution $P_{\mathrm{env}}(\cdot\mid h_t,a_t)$ over the next representation $s_{t+1}=\psi(h_{t+1})$. A representation $s_t$ is a \emph{sufficient information state} if this conditional distribution depends on the history only through $s_t$ and the chosen action:
\[
\Pr(s_{t+1}\in B\mid h_t,a_t)=K(B\mid s_t,a_t)
\]
for every measurable set $B$. Admissible actions and terminal status, including success or failure, must also be determined by $s_t$. Relevant task memory and any horizon affecting future outcomes are part of this information. A belief over hidden states, augmented with such task information, is a standard POMDP information state~\citep{kaelbling1998planning}.

Let $f$ be a fixed measurable map into a finite or countable shared-state set. The information-state space itself may be continuous. For a shared state $\hat s_w$, define its successor-class probability by
\[
K_f(\hat s_w\mid s,a)
:=\int_{f^{-1}(\hat s_w)}K(\mathrm d u\mid s,a).
\]
We call $f$ \emph{state-transition-preserving} when states with the same image have the same admissible actions and terminal status/outcome, and, for every common admissible action $a$,
\[
f(s)=f(u)\quad\Longrightarrow\quad
\{\hat s_w:K_f(\hat s_w\mid s,a)>0\}
=\{\hat s_w:K_f(\hat s_w\mid u,a)>0\}.
\]
This condition equates successor supports while allowing different state-transition probabilities. The vertices of $\mathcal G_{\mathrm{env}}^f$ are the shared states reachable from the task's initial distribution under the full interaction process, including feedback after invalid actions. Its action-labelled edges $(\hat s_v,a,\hat s_w)$ include only admissible actions with $K_f(\hat s_w\mid s,a)>0$ for $f(s)=\hat s_v$. Thus every sampled state belongs to the reference vertex set, while invalid state transitions are excluded from both graphs. Successful terminal classes form its success set. Each edge counts as one hop. We write $d_{\mathrm{env}}^*(s)$ for the shortest distance from $f(s)$ to this set, with infinite distance when no such path exists. The reference graph represents the POMDP's information-state transition support.

\subsection{Proof of Proposition~\ref{prop:graph_approximation} and Cumulative-Sampling Convergence}
\label{app:graph-approximation-proof}

\begin{proof}
Every sampled vertex is reachable under the full interaction process, and every retained edge comes from an observed, environment-admissible state transition. Because the shared-state partition is countable, its observed successor class has positive conditional probability almost surely. Hence every vertex and edge lies in $\mathcal G_{\mathrm{env}}^f$; every observed successful endpoint has the same terminal label in that graph. Consequently $\mathcal G\subseteq\mathcal G_{\mathrm{env}}^f$ almost surely, with a subset of the reference success states.

The state-transition-preserving condition allows combining paths at shared states. Given a finite path $(\hat s_{v_0},a_0,\hat s_{v_1},\ldots,a_{\ell-1},\hat s_{v_\ell})$, its first state transition has positive probability from every information state in $f^{-1}(\hat s_{v_0})$. The same property holds at each subsequent class. Inductively, following the listed actions has positive probability of realizing the entire class sequence. Terminal-label preservation ensures that a path ending at a successful class ends in success. Thus each abstract path represents a possible environment path, even when its edges were sampled in different trajectories.

Since all paths in $\mathcal G$ are also in the reference graph and its observed success set is a subset of the full success set,
\[
d_{\mathcal G}(f(s))\geq d_{\mathrm{env}}^*(s).
\]
If a shortest reference path and its successful endpoint are represented, the reverse inequality holds as well, giving equality.

For convergence, suppose $\mathcal G_{\mathrm{env}}^f$ is finite and trajectories are independent draws from a fixed sampler. Let $\mathcal C$ contain its vertices, action-labelled edges, and successful-terminal labels. For every component $c\in\mathcal C$, assume the probability $p_c$ that one sampled trajectory contributes $c$ is positive. This requires sufficient exploration and rollout support for each component. Let $\mathcal G_n$ be the graph formed from the first $n$ sampled trajectories, retaining their observed terminal labels. Independence gives a missing-component probability $(1-p_c)^n$. Therefore
\[
\Pr\!\left(\mathcal G_n\neq\mathcal G_{\mathrm{env}}^f\right)
\leq\sum_{c\in\mathcal C}(1-p_c)^n
\leq |\mathcal C|e^{-np_{\min}},
\qquad p_{\min}:=\min_{c\in\mathcal C}p_c>0.
\]
Each component has a finite first-appearance time almost surely; there are finitely many components. Their maximum is therefore finite almost surely, after which $\mathcal G_n=\mathcal G_{\mathrm{env}}^f$. All represented shortest distances then agree with their reference distances. This proves almost-sure convergence.
\end{proof}

The convergence statement concerns cumulative samples with a fixed state map and sampling distribution. During training, RewardFlow constructs a separate graph from each current batch. Environment-invalid state transitions are excluded before paths are formed; success reachability is then determined within the resulting graph. Stochastic decoding randomizes the policy's action choices, whereas $K(\cdot\mid s,a)$ describes the next information state conditional on an action; both sources of randomness are allowed.

\subsection{State-Transition-Support Potential in a POMDP}

For Rem.~\ref{rem:value_topology}, define the reference potential $\Phi_{\mathrm{env}}(s):=\gamma^{d_{\mathrm{env}}^*(s)}$ from the POMDP's state-transition-support graph. It is $1$ at successful terminals and $0$ when success is unreachable. Under state-transition-preserving grouping, every finite reference path has positive probability from each information state in its starting class, as shown above. Thus $\Phi_{\mathrm{env}}$ measures the discounted length of a shortest possible route to success, whether the information-state transitions are deterministic or stochastic.

To relate this potential to expected value, let $\tau_{\mathrm{succ}}$ be the number of state transitions until successful termination, taking value $\infty$ when success is never reached. For policies choosing admissible actions, use the terminal-value convention
\[
V^\pi(s):=\mathbb E_s^\pi[\gamma^{\tau_{\mathrm{succ}}}\mathbf 1\{\tau_{\mathrm{succ}}<\infty\}],\qquad
V^*(s):=\sup_\pi V^\pi(s).
\]
A successful terminal has value $1$. Almost surely, every finite successful realization follows edges of $\mathcal G_{\mathrm{env}}^f$, so $\tau_{\mathrm{succ}}\geq d_{\mathrm{env}}^*(s)$. Its discounted return is therefore at most $\Phi_{\mathrm{env}}(s)$, giving $V^*(s)\leq\Phi_{\mathrm{env}}(s)$ after taking expectations and the supremum over policies. If no success path exists, both quantities are zero.

Equality holds when a policy reaches success in exactly $d_{\mathrm{env}}^*(s)$ state transitions almost surely. Deterministic information-state transitions are one sufficient special case: a shortest supported path can then be followed with certainty. In general, the only action may reach success with probability $p\in(0,1)$ and terminal failure otherwise, giving $\Phi_{\mathrm{env}}(s)=\gamma$ but $V^*(s)=p\gamma$. The approximation results below compare $\Phi_{\mathcal G}$ with the state-transition-support potential $\Phi_{\mathrm{env}}$.

\subsection{Proof of Proposition~\ref{prop:value_approximation} (State-Potential Approximation)}

\begin{proof}
Under the POMDP information-state conditions of Rem.~\ref{rem:value_topology}, we compare $\Phi_{\mathcal G}(f(s))=\gamma^{d_{\mathcal G}(f(s))}$ with $\Phi_{\mathrm{env}}(s)=\gamma^{d_{\mathrm{env}}^*(s)}$.

Prop.~\ref{prop:graph_approximation} gives $\mathcal G\subseteq\mathcal G_{\mathrm{env}}^f$ and hence
\[
d_{\mathcal G}(f(s))\geq d_{\mathrm{env}}^*(s).
\]
Since $\gamma\in(0,1)$, the function $\gamma^x$ is strictly decreasing, so
\[
\Phi_{\mathcal G}(f(s))=\gamma^{d_{\mathcal G}(f(s))}
\leq\gamma^{d_{\mathrm{env}}^*(s)}=\Phi_{\mathrm{env}}(s).
\]
The coverage error is
\begin{align*}
\delta(s):=\bigl|\Phi_{\mathcal G}(f(s))-\Phi_{\mathrm{env}}(s)\bigr|
&=\Phi_{\mathrm{env}}(s)-\Phi_{\mathcal G}(f(s))\\
&=\gamma^{d_{\mathrm{env}}^*(s)}-\gamma^{d_{\mathcal G}(f(s))}\geq0.
\end{align*}
For a success-reachable state, the error vanishes once the graph contains a shortest reference path from $f(s)$ to success and its successful endpoint; for an unreachable state, both potentials are already zero. Under the cumulative-sampling conditions in Appendix~\ref{app:graph-approximation-proof}, the distance and potential eventually equal $d_{\mathrm{env}}^*(s)$ and $\Phi_{\mathrm{env}}(s)$ almost surely. These conclusions use only state-transition support and thus hold with stochastic state transitions.
\end{proof}

\subsection{Proof of Corollary~\ref{cor:value_progress} (Action-Level Reward Approximation)}

\begin{proof}
Since $\Phi_{\mathcal G}(f(s))=\Phi_{\mathrm{env}}(s)-\delta(s)$ by definition of $\delta$, we expand directly:
\begin{align*}
\tilde r(s,a,s')
&=\Phi_{\mathcal G}(f(s'))-\Phi_{\mathcal G}(f(s))\\
&=\bigl(\Phi_{\mathrm{env}}(s')-\delta(s')\bigr)-\bigl(\Phi_{\mathrm{env}}(s)-\delta(s)\bigr)\\
&=\bigl(\Phi_{\mathrm{env}}(s')-\Phi_{\mathrm{env}}(s)\bigr)-\bigl(\delta(s')-\delta(s)\bigr).
\end{align*}
Nonnegativity of the two coverage errors gives
\[
\bigl|\tilde r(s,a,s')-\bigl(\Phi_{\mathrm{env}}(s')-\Phi_{\mathrm{env}}(s)\bigr)\bigr|
=|\delta(s')-\delta(s)|\leq\max\{\delta(s),\delta(s')\}.
\]
By Prop.~\ref{prop:value_approximation}, both errors vanish when the relevant shortest paths and their successful endpoints are covered, so the action reward recovers $\Phi_{\mathrm{env}}(s')-\Phi_{\mathrm{env}}(s)$. The same conclusion holds almost surely under cumulative sampling.

Since $\tilde r(s,a,s')=\gamma^{d_{\mathcal G}(f(s'))}-\gamma^{d_{\mathcal G}(f(s))}$ and $\gamma^x$ is strictly decreasing, the sign of the raw action reward identifies progress through the two sampled-graph distances.
\end{proof}

\subsection{Proof of Proposition~\ref{prop:ordering_invariance} (Trajectory-Level Reward Ordering)}

\begin{proof}
For a trajectory whose state transitions are all retained, $\tilde r_t^{(i)}=\phi_{t+1}^{(i)}-\phi_t^{(i)}$ at every turn, so the sum of raw action rewards telescopes:
\[
\sum_{t=0}^{T_i-1}\tilde r_t^{(i)}=\phi_{T_i}^{(i)}-\phi_0^{(i)}.
\]
The successful and failed trajectories share the same canonical initial node $\hat s_{\mathrm{init}}$, so $\phi_0^{(+)}=\phi_0^{(-)}=\Phi_{\mathcal G}(\hat s_{\mathrm{init}})$ cancels:
\begin{align*}
\sum_{t=0}^{T_+-1}\tilde r_t^{(+)}-\sum_{t=0}^{T_--1}\tilde r_t^{(-)}
&=\bigl(\phi_{T_+}^{(+)}-\Phi_{\mathcal G}(\hat s_{\mathrm{init}})\bigr)-\bigl(\phi_{T_-}^{(-)}-\Phi_{\mathcal G}(\hat s_{\mathrm{init}})\bigr)\\
&=\phi_{T_+}^{(+)}-\phi_{T_-}^{(-)}.
\end{align*}
This holds for any state graph, regardless of coverage or approximation quality. With $\phi_{T_+}^{(+)}=1$, strict ordering holds if the failed endpoint is not mapped to a successful node, so $\phi_{T_-}^{(-)}<1$. The gap equals $1$ when the failed endpoint has no path to observed success and thus has potential zero. A timeout can end at a success-reachable nonterminal with positive potential. This identity applies to raw action rewards before state-group normalization.
\end{proof}



\section{Further Discussion}
\label{app: further discussion}

\subsection{Relation to Potential-Based Reward Shaping and Reward Machines}
\label{app: pbrs}

RewardFlow uses $\tilde r(s,a,s')=\Phi_{\mathcal G}(f(s'))-\Phi_{\mathcal G}(f(s))$ to turn a node potential into an action reward. Scalar-potential differences are established in control and RL~\citep{porteus1975bounds,mataric1994reward,randlov1998learning}; the question here is how current trajectories supply the potential. RewardFlow constructs $\Phi_{\mathcal G}$ from state transitions and successful endpoints observed across trajectories, without a predefined task graph or a learned process scorer.

Potential-based reward shaping (PBRS)~\citep{ng1999policy} uses $F_{\mathrm{PBRS}}(s,s')=\gamma\Psi(s')-\Psi(s)$ for a scalar potential $\Psi$ in the discounted setting, with its applicable terminal conditions. RewardFlow uses the consecutive difference of its graph potential. Its finite-trajectory sum is $\phi_{T_i}^{(i)}-\phi_0^{(i)}$. This identity concerns raw action rewards before state-group normalization and their combination with trajectory-level advantages. Appendix~\ref{app:proofs} derives the corresponding reward-ordering result.

Reward Machines~\citep{icarte2018using, icarte2022reward} and the finite-state automaton (FSA) induction work of
\citet{furelos2021induction} likewise structure rewards using finite automata, but these
methods rely on either hand-specified or learned symbolic abstractions. RewardFlow
constructs nodes from preprocessed observations and rebuilds its state graph from each batch's trajectories. Tab.~\ref{tab:pbrs-comparison} compares how these methods obtain and use their reward structures.

\begin{table}[h]
\centering
\caption{Conceptual roles of reward-shaping paradigms. PBRS specifies a reward transformation rather than a unique procedure for obtaining its potential.}
\label{tab:pbrs-comparison}
\fontsize{8}{10}\selectfont
\setlength\tabcolsep{4pt}{
\begin{tabular}{p{0.18\linewidth}p{0.23\linewidth}p{0.25\linewidth}p{0.24\linewidth}}
\toprule
 & \textbf{PBRS} & \textbf{Reward Machines / learned FSAs} & \textbf{RewardFlow} \\
\midrule
Required input & A scalar potential & Symbolic events and a specified or learned automaton & Canonical observations, observed state transitions, and outcomes \\
Main operation & Transform rewards given a potential & Track symbolic reward-relevant structure & Estimate potentials over the state graph \\
Structure source & Not prescribed by the shaping theorem & Specification or automaton induction & Current on-policy trajectory group \\
\bottomrule
\end{tabular}}
\vspace{-8pt}
\end{table}

RewardFlow joins observed paths across trajectories before assigning action rewards. The following example shows why that changes the information available compared with computing returns along each trajectory separately.

\subsection{Shared-State Grouping versus Graph-Based Reward Estimation}
\label{app:gigpo-topology-example}

GiGPO~\citep{feng2025group} already supplies action-level rewards: it compares trajectory-local discounted returns among actions sampled at the same anchor state. RewardFlow uses shared states additionally as graph nodes, merges their observed state transitions, and propagates over the combined topology before computing state-conditioned action advantages. Plain trajectory propagation, state grouping, and graph propagation are therefore different operations.

\paragraph{A minimal deterministic example.}
Assume exact state identity, terminal rewards $0/1$, and $\gamma=0.9$. Let $s_{\mathrm A},s_{\mathrm B},s_{\mathrm C}$ name three nonterminal states, and $s_{\mathrm F},s_{\mathrm S}$ failed and successful terminals. Consider the two observed trajectories
\[
\begin{aligned}
\tau^{(1)}&:s_{\mathrm A}\xrightarrow{a}s_{\mathrm C}\xrightarrow{b}s_{\mathrm F},\\
\tau^{(2)}&:s_{\mathrm A}\xrightarrow{c}s_{\mathrm B}\xrightarrow{d}s_{\mathrm C}\xrightarrow{e}s_{\mathrm S}.
\end{aligned}
\]
Each named state maps to its correspondingly named canonical node, for example $f(s_{\mathrm C})=\hat s_{\mathrm C}$. For trajectory returns, the terminal reward is attached to the final state transition. The merged graph instead has terminal potential $\Phi_{\mathcal G}(\hat s_{\mathrm S})=1$ and node potentials
$\Phi_{\mathcal G}(\hat s_{\mathrm F})=0$, $\Phi_{\mathcal G}(\hat s_{\mathrm C})=0.9$, $\Phi_{\mathcal G}(\hat s_{\mathrm B})=0.81$, and $\Phi_{\mathcal G}(\hat s_{\mathrm A})=0.81$.
Tab.~\ref{tab:gigpo-topology-example} illustrates the resulting advantages by subtracting each state group's mean reward for both methods.

\begin{table}[htbp]
\centering
\small
\caption{Illustrative raw signals and centered advantages at shared nodes $\hat s_{\mathrm A}$ (actions $a,c$) and $\hat s_{\mathrm C}$ (actions $b,e$). Entries are ordered pairs.}
\label{tab:gigpo-topology-example}
\begin{tabular}{lcccc}
\toprule
Method & Raw $(a,c)$ & Centered $(a,c)$ & Raw $(b,e)$ & Centered $(b,e)$ \\
\midrule
GiGPO & $(0,0.81)$ & $(-0.405,0.405)$ & $(0,1)$ & $(-0.5,0.5)$ \\
RewardFlow & $(0.09,0)$ & $(0.045,-0.045)$ & $(-0.9,0.1)$ & $(-0.5,0.5)$ \\
\bottomrule
\end{tabular}
\end{table}

Both methods distinguish $b$ from $e$ at $\hat s_{\mathrm C}$. GiGPO assigns $a$ the return recorded from that action onward in the failed trajectory. RewardFlow rewards $a$ for reaching $\hat s_{\mathrm C}$, where another trajectory supplies the successful action $e$, while penalizing $b$. The graph path $\hat s_{\mathrm A}\xrightarrow{a}\hat s_{\mathrm C}\xrightarrow{e}\hat s_{\mathrm S}$ composes existing edges without adding a state transition.

\subsection{State Potentials and Directed State Transitions}
\label{app: reward contradictions}

\paragraph{Potentials at shared nodes.} Reverse BFS assigns one distance and hence one potential to each canonical node, including nodes with several observed successors. Under the state-transition-preserving assumptions in Sec.~\ref{ssec:theoretical_justification}, shared information states admit the same actions and possible successor classes, while allowing stochastic state transitions. Each action occurrence then receives the potential difference along its observed state transition.

\paragraph{Directed state transitions and dead ends.} Every retained edge records an observed directed state transition. Shortest-hop distance measures proximity to success through the support of these state transitions; Appendix~\ref{app:proofs} relates the resulting potential to expected value under stochastic state transitions. Reverse BFS follows incoming graph edges to compute distances, while environment state transitions retain their original direction. A state transition from a success-reachable node to an unreachable dead end has negative potential difference.

\subsection{Exploration and Graph Coverage}
\label{app: exploration}

\paragraph{Graph construction from sampled experience.} RewardFlow constructs a graph from the current task's sampled trajectories. After node construction, reverse BFS costs $O(|\mathcal V|+|\mathcal E|)$ in the observed nodes and edges. Rollout generation and state preprocessing have their own computational costs. Successful endpoints seed propagation through the observed state transitions, and additional trajectories can expose further paths to success.

\paragraph{No successful trajectory.} With no successful endpoint in a trajectory group, $\mathcal V_{\mathrm{succ}}=\varnothing$, every graph potential and consecutive difference is zero, giving zero graph-derived action advantages. In the binary-success-only case, the trajectory advantage is also zero under the nonzero-denominator convention in Appendix~\ref{app:notation-chain}, leaving only the KL term. In the experimental recipe, environment-return differences can also contribute to the advantages (Appendix~\ref{app: detailed experimental setting}). Conversely, if every trajectory succeeds and binary success is the only reward, equal outcomes give zero trajectory advantage, while potential differences can supply nonzero action advantages.

\paragraph{Graph budget in search tasks.} The four-step, five-trajectory search budget supplies at most twenty state transitions per group before merging, fewer than the structured-task budgets. Cross-trajectory paths are formed within this sampled graph. Tab.~\ref{tab:deepresearch} reports improvements over GiGPO at both model scales, and Appendix~\ref{app: state preprocessing} describes the state interfaces used in search tasks and supplementary mathematical reasoning.

\paragraph{Rollout-budget sensitivity.} Tab.~\ref{tab:exploration-ablation} reports higher success than GiGPO at rollout budgets of 4, 6, and 8 on ALFWorld.

Appendix~\ref{app: training efficiency} reports how entropy, graph size, and success vary during training. The state-graph cases in Appendix~\ref{app: state graph cases} illustrate connectivity across one or more sampled trajectories.

\subsection{Graph-Based Rewards versus Learned Process Scorers}
\label{app:prm-discussion}

Process supervision has been studied in mathematical reasoning~\citep{lightman2023let} and agentic interaction. For long-horizon tasks, a process scorer must account for actions, observations, and their changing relationship to task completion. Comparisons should distinguish learned scorers from methods that derive action rewards from outcomes: GiGPO~\citep{feng2025group}, for example, already combines action and trajectory advantages without a separate learned PRM. RewardFlow additionally uses the topology of observed state transitions across trajectories, as illustrated in Appendix~\ref{app:gigpo-topology-example}. It does not require human step labels or an auxiliary learned process-reward scorer.

Fixed per-step scorers can become misaligned with a changing policy. RewardFlow instead reconstructs its graph from current trajectories and observed outcomes to update action rewards. Appendix~\ref{app: prm} provides the PRM comparison results and training protocol; Appendix~\ref{app: fixed scorer diagnostics} examines frozen-scorer transfer to new policies.

\subsection{Learned Graph Models for Reward Estimation}
\label{app: gnn propagation}

A graph neural network (GNN) could use state and action features as well as connectivity. Such an extension would require a training objective, supervision source, and update schedule, along with additional optimization and inference.

RewardFlow uses BFS to compute the distance to observed success directly. Its cost is linear in the constructed graph after node processing, and the resulting distances determine potentials without fitting another model.

\section{Further Experiments}
\label{app:further_experiments}

\subsection{Detailed Experimental Setting}
\label{app: detailed experimental setting}

\paragraph{Evaluation benchmarks.} The four main benchmarks expose different observations and actions:
\begin{itemize}[leftmargin=*]
\vspace{-6pt}
    \item \textbf{ALFWorld}~\citep{shridhar2021alfworld} is a synthetic text-based simulator aligned with the embodied benchmark ALFRED~\citep{shridhar2020alfred}, comprising 3{,}827 household tasks spanning six types: \textit{Pick \& Place} (Pick), \textit{Examine in Light} (Look), \textit{Clean \& Place} (Clean), \textit{Heat \& Place} (Heat), \textit{Cool \& Place} (Cool), and \textit{Pick Two \& Place} (Pick2).

    \item \textbf{WebShop}~\citep{yao2022webshop} is a large-scale benchmark where the agent must fulfill natural-language shopping instructions by issuing text commands such as ``search'' and ``click'' to web pages with over 1.18 million real Amazon products.

    \item \textbf{Sokoban}~\citep{SchraderSokoban2018} is a classic puzzle game in which agents must observe and analyze the grid layout and push boxes onto target locations, taxing spatial understanding and long-horizon planning. In this work, we randomly generate \(6\times 6\) Sokoban boards for training and evaluation.

    \item For \textbf{search tasks}~\citep{jin2025search}, agents answer knowledge-intensive questions using a search engine to retrieve documents. The environment state consists of retrieved documents, and the agent's actions issue search queries or submit a final answer. Retrieved passages can vary between queries, making exact observation matching insufficient for sharing research progress. We evaluate on single-hop tasks: Natural Questions (NQ)~\citep{kwiatkowski2019natural}, TriviaQA~\citep{joshi2017triviaqa}, PopQA~\citep{mallen2023not}, and multi-hop tasks: HotpotQA~\citep{yang2018hotpotqa}, 2WikiMultiHopQA (2Wiki)~\citep{ho2020constructing}, Musique~\citep{trivedi2022musique}, and Bamboogle~\citep{press2023measuring}. Training uses NQ and HotpotQA; the other five datasets evaluate out-of-distribution transfer.
\vspace{-6pt}
\end{itemize}

\paragraph{ALFWorld out-of-distribution evaluation.}
The evaluation uses familiar objects in unseen rooms and layouts (Tab.~\ref{tab:ood}).

\paragraph{Metrics and evaluation repeats.}
Success rate is the fraction of evaluated episodes that satisfy the environment's task-completion check. For ALFWorld, All pools task-completion outcomes across evaluated episodes, while the six subtask columns report rates within each task type. WebShop and Sokoban additionally report their task scores. WebShop's task score measures how well the purchased product matches the requested attributes and is displayed on a 0-100 scale; Sokoban's task score measures progress in solving the puzzle. For search tasks, we report answer accuracy, with an unweighted mean across the listed datasets. Rate improvements are reported as absolute differences in \%.

Tabs.~\ref{tab:main} and~\ref{tab:ood} report variation across three inference seeds at fixed checkpoints. Appendix~\ref{app: matched preprocessing control} specifies the separate 128-task in-distribution evaluation used for the ALFWorld-3B larger-budget and preprocessing comparisons.

\paragraph{Training configuration.}

Tab.~\ref{tab: configuration} records the training and algorithm settings for the main experiments, implemented with Verl-Agent~\citep{feng2025group}. Supplementary controls state their differing configurations in their respective subsections.

\begin{table}[H]
\centering
\caption{Details of training configuration of experiments.}
\label{tab: configuration}
\fontsize{6}{6}\selectfont
        \setlength\tabcolsep{8pt}{
\begin{tabular}{l|l|c|c|c|c}
\toprule
Type & Hyperparameter & ALFWorld & WebShop & Sokoban & Search tasks\\
\midrule
\multicolumn{6}{c}{\textit{Qwen2.5-(VL)-1.5B/3B-Instruct}}\\
\midrule
Learning & Training batch size & 16 & 16 & 16 & 256 \\
Learning & Minibatch size for policy update & 256 & 64 & 256 & 512 \\
Learning & Max action turns per trajectory & 25 & 15 & 15 & 4 \\
Learning & Max prompt length & 2048 & 4096 & 1024 & 4096 \\
Learning & Max response length & 512 & 512 & 512 & 512 \\
Learning & Training steps & 100 & 100 & 100 & 200 \\
Learning & Temperature & 0.4 & 0.4 & 0.4 & 0.4 \\
Learning & Learning rate & 1e-6 & 1e-6 & 1e-6 & 1e-6 \\
Learning & Invalid action penalty & 0.1 & 0.1 & 0.1 & 0.01 \\
Learning & Computational cost & 2x(A100 \& h20) & 2x(A100 \& h20) & 2x(A100 \& h20) & 2x(A100 \& h20) \\
\midrule
\multicolumn{6}{c}{\textit{Qwen2.5-(VL)-7B-Instruct}}\\
\midrule
Learning & Training batch size & 16 & 16 & 16 & 256 \\
Learning & Minibatch size for policy update & 256 & 64 & 256 & 512 \\
Learning & Max action turns per trajectory & 25 & 15 & 15 & 4 \\
Learning & Max prompt length & 2048 & 4096 & 1024 & 4096 \\
Learning & Max response length & 512 & 512 & 512 & 512 \\
Learning & Training steps & 100 & 100 & 100 & 200 \\
Learning & Temperature & 0.4 & 0.4 & 0.4 & 0.4 \\
Learning & Learning rate & 1e-6 & 1e-6 & 1e-6 & 1e-6 \\
Learning & Invalid action penalty & 0.1 & 0.1 & 0.1 & 0.01 \\
Learning & Computational cost & 4x(A100 \& h20) & 4x(A100 \& h20) & 4x(A100 \& h20) & 4x(A100 \& h20) \\
\midrule
\multicolumn{6}{c}{\textit{GRPO}}\\
\midrule
Algorithm & Group size & 8 & 8 & 8 & 5 \\
Algorithm & KL-divergence loss coefficient & 0.01 & 0.01 & 0.01 & 0.001 \\
\midrule
\multicolumn{6}{c}{\textit{RLOO}}\\
\midrule
Algorithm & Group size & 8 & 8 & 8 & 5 \\
Algorithm & KL-divergence loss coefficient & 0.01 & 0.01 & 0.01 & 0.001 \\
\midrule
\multicolumn{6}{c}{\textit{GiGPO}}\\
\midrule
Algorithm & Discount factor $\gamma$ & 0.95 & 0.95 & 0.95 & 0.95 \\
Algorithm & weighting coefficient $w$ & 1 & 1 & 1 & 1 \\
Algorithm & Group size & 8 & 8 & 8 & 5 \\
Algorithm & KL-divergence loss coefficient & 0.01 & 0.01 & 0.01 & 0.001 \\
\midrule
\multicolumn{6}{c}{RewardFlow}\\
\midrule
Algorithm & Discount factor $\gamma$ & 0.90 & 0.90 & 0.90 & 0.90 \\
Algorithm & Max graph propagation iteration & 1000 & 1000 & 1000 & 1000 \\
Algorithm & Action-level advantage weight $\alpha_\text{act}$ & 1 & 1 & 1 & 1 \\
Algorithm & Trajectory-level advantage weight $\alpha_{\text{traj}}$ & 1 & 1 & 1 & 1 \\
Algorithm & Group size & 8 & 8 & 8 & 5 \\
Algorithm & KL-divergence loss coefficient & 0.01 & 0.01 & 0.01 & 0.001 \\
\bottomrule
\end{tabular}
}
\vspace{-6pt}
\end{table}

\paragraph{Reward assignment.}
The training recipe uses environment episode returns for outcome-based supervision. Graph cleanup selects state transitions for propagation, while all active generated responses remain in the policy-update batch. Graph-derived rewards are mapped back to their response positions; missing node potentials are zero, and response rows without a mapped graph reward retain their environment reward for the corresponding state transition as the action-reward entry. Each response receives its trajectory's episode return as its outcome score.

\paragraph{Advantage computation and token masks.}
Action-level rewards are grouped by task and source state, including active responses excluded from graph propagation. Singleton groups receive zero action-level advantage. Task-group baselines are computed by averaging outcome scores over active responses. The action- and trajectory-level advantages are combined and broadcast to generated response tokens; padding is masked.


\subsection{Comparison with Process Reward Models}
\label{app: prm}

Fig.~\ref{fig:rewardflow_teaser} reports comparisons against three baselines using PRM-estimated rewards
(LLM-as-PRM with Qwen2.5-7B-Instruct, ArmoRM-8B, and Selene-8B) and plain GRPO on ALFWorld with Qwen2.5-1.5B-Instruct.

\paragraph{PRM comparison protocol.}
All three comparisons use the ALFWorld Qwen2.5-1.5B-Instruct setting in Appendix~\ref{app: detailed experimental setting}: 100 training steps, task batch size 16, eight rollouts per task, a maximum of 25 action turns, sampling temperature 0.4, and learning rate $10^{-6}$. Each PRM supplies the action-level reward for the generated response in its interaction context. Actions are grouped using GiGPO's state-grouping procedure; their PRM rewards provide the process-level advantage, which is combined with the outcome-based trajectory advantage. The policy, environment, training budget, and evaluation procedure follow the ALFWorld-1.5B setup, while the action-reward source changes to the corresponding PRM. Plain GRPO supplies the outcome-only reference at $50.0\%$ success.

\paragraph{Comparison results.}
RewardFlow achieves
62.8\% success, with absolute improvements over the strongest PRM baseline (GRPO + Selene-8B at 48.4\%) of
14.4\% and over the no-PRM GRPO baseline (50.0\%) of 12.8\%.

\subsection{Fixed-Scorer Transfer Diagnostics}
\label{app: fixed scorer diagnostics}

Complementing the PRM training comparison in Appendix~\ref{app: prm}, the following diagnostics examine transfer with two additional agent-oriented scorers.

We examine released AgentPRM\footnote{Here AgentPRM refers to Choudhury, \emph{Process Reward Models for LLM Agents: Practical Framework and Directions} (2025), \url{https://arxiv.org/abs/2502.10325}; released implementation: \url{https://github.com/sanjibanc/agent_prm}.} and Web-Shepherd~\citep{chae2025web} scorers as frozen process-reward sources for Qwen2.5-1.5B policies in ALFWorld and WebShop, respectively. The AgentPRM integration scores responses from the Qwen policy with a fixed reward model. The Web-Shepherd integration uses a single-judgment probability adapter.

\paragraph{AgentPRM: action admissibility versus a high score.}
The run was stopped near step 57, with checkpoint 50 retained for diagnosis. Four sampled trajectories contain 100 action turns. All 100 parsed action strings are outside the pre-action admissible-action lists, yet their mean raw scorer logit is $97.785$; none of the four trajectories succeeds. On these same 100 states, responses from the released AgentPRM policy are all admissible and have mean logit $0.191$. The transferred Qwen policy frequently combines several commands into one action. The integration checked output syntax but did not enforce candidate membership before assigning the process reward. The reported scorer outputs are unbounded raw logits.

\paragraph{Web-Shepherd: proxy/outcome divergence.}
The one-judgment run completed 100 training steps with batch size 16, eight trajectories per task, horizon 15, and training seed 0. Comparing training-window means over steps 1-10 and 91-100, the raw process score rises from $0.1931$ to $0.9111$, while environment success rate decreases from $3.89\%$ to $0.56\%$. The checkpoint-100 evaluation on 128 tasks per inference seed ($0,1,2$) yields success rates $1.5625\%$, $1.5625\%$, and $0.7813\%$: mean $1.3021\%$ with population standard deviation $0.3683\%$. The adapter's historical trajectory input contains past thoughts/actions but not past page observations; inspected traces narrate unsupported progress. The increasing process score and declining task success show proxy/outcome divergence in this adapter-policy pairing.

\paragraph{Implication for graph-based reward estimation.}
These diagnostics show divergence between transferred process scores and environment outcomes. RewardFlow instead computes action rewards from observed state transitions and terminal outcomes.

\subsection{State Preprocessing and Invalid-Action Filtering}
\label{app: normalization ablation}

Tab.~\ref{tab:alfworld-ablation} shows the contribution of state preprocessing and invalid-action filtering in RewardFlow on ALFWorld (Qwen2.5-3B-Instruct). Removing either state preprocessing or invalid-action filtering lowers success in this configuration. Methodological details of state preprocessing and invalid-action filtering are provided in Appendices~\ref{app: state preprocessing} and~\ref{app: invalid action filtering} respectively.

\begin{table}[h]
\centering
\caption{Performance of RewardFlow with or without each state/action interface component on ALFWorld using Qwen2.5-3B-Instruct. Success rate (\%).}
\label{tab:alfworld-ablation}
\fontsize{8}{10}\selectfont
\setlength\tabcolsep{7pt}{
\begin{tabular}{lc}
\toprule
\textbf{Method} & \textbf{Success Rate (\%)} \\
\midrule
Base & 13.5 \\
RewardFlow (w/o state preprocessing) & 53.9 \\
RewardFlow (w/o invalid-action filtering) & 60.2 \\
RewardFlow & \textbf{75.5} \\
\bottomrule
\end{tabular}}
\vspace{-8pt}
\end{table}

\subsection{Separating Reward Estimation from Preprocessing and Filtering}
\label{app: matched preprocessing control}

To distinguish the state preprocessing and invalid-action filtering from graph-based action-level reward estimation, we additionally evaluate GiGPO with RewardFlow's state preprocessing and invalid-action filtering enabled, but with RewardFlow graph construction and propagation disabled. Tab.~\ref{tab:matched-preprocessing-control} presents this control alongside the ALFWorld-3B rows from Tab.~\ref{tab:gigpo-matched}. Each method uses one training run (seed 0), evaluated at step 150 on the same 128-task in-distribution split, with one sampled trajectory per task, horizon 50, and temperature 0.4. The three repeats vary inference seeds $0,1,2$; error bars are population standard deviations of these three success rates. The additional control has success counts $104/128$, $99/128$, and $99/128$.

\begin{table}[htbp]
\centering
\small
\caption{Preprocessing control on ALFWorld-3B. Success rate (\%), mean $\pm$ population std across three inference seeds at fixed checkpoints after 150 training steps. The outer rows also appear in Table~\ref{tab:gigpo-matched}.}
\label{tab:matched-preprocessing-control}
\begin{tabular}{lc}
\toprule
Method & Success rate (\%) \\
\midrule
GiGPO & $75.2\pm1.6$ \\
GiGPO + RewardFlow preprocessing/filtering & $78.6\pm1.8$ \\
RewardFlow & $82.8\pm0.7$ \\
\bottomrule
\end{tabular}
\end{table}

Adding RewardFlow's state preprocessing and invalid-action filtering improves GiGPO by $3.4\%$ in absolute success rate. The complete RewardFlow estimator improves success by a further $4.2\%$. Tab.~\ref{tab:alfworld-ablation} separately reports single-component removals.

\subsection{Semantic-Aggregation Threshold Sensitivity}
\label{app: aggregation sensitivity}

Tab.~\ref{tab:aggregation-threshold-sweep} gives the values underlying the Qwen2.5-3B threshold analysis on search tasks in Fig.~\ref{fig:threshold-sensitivity}. Lower cosine thresholds permit broader merging; higher thresholds retain finer distinctions.

\begin{table}[htbp]
\centering
\small
\caption{Average accuracy (\%) on search tasks across semantic-aggregation thresholds.}
\label{tab:aggregation-threshold-sweep}
\begin{tabular}{lccccc}
\toprule
Threshold $\tau$ & 0.80 & 0.85 & 0.90 & 0.95 & 1.00 \\
\midrule
Average accuracy & 43.8 & 44.7 & 44.3 & 45.0 & 44.2 \\
\bottomrule
\end{tabular}
\end{table}

Accuracy ranges from $43.8\%$ to $45.0\%$ across cosine-similarity thresholds $0.80$-$1.00$, an absolute spread of $1.2\%$.

\subsection{Shortest-Path versus Mean Propagation}
\label{propagation strategy}
We compare maximum and mean aggregation in decay-based graph propagation. Using the unit-success potential convention of Sec.~\ref{ssec: process_reward_shaping}, potentials start at $1$ for successful terminals and $0$ elsewhere. At each iteration, maximum aggregation takes the largest value among a node's current potential and its successors' discounted potentials. This recovers the shortest-hop potential in Eq.~\eqref{eq:reward_propagation}, since shorter distances give larger discounted values. The mean variant averages these same candidates. Writing $\mathcal N^+(\hat s_v)$ for the distinct successor nodes, its update is
\[
\Phi_{\ell+1}(\hat s_v)=
\frac{\Phi_\ell(\hat s_v)+\gamma\sum_{\hat s_w\in\mathcal N^+(\hat s_v)}\Phi_\ell(\hat s_w)}
{1+|\mathcal N^+(\hat s_v)|}.
\]
Both variants update all nodes synchronously for the configured 1{,}000 iterations. Terminal sinks retain their initial potentials. Nodes with no path to success remain at zero; with no successful terminal, all potentials remain zero. Mean aggregation combines successor values even when the graph has only one successful terminal, whereas maximum aggregation selects the strongest propagated value. Maximum aggregation performs better in both comparisons, including a $15.9\%$ absolute success-rate improvement on Sokoban (Tab.~\ref{tab:ablation-propagation-strategy}).

Appendix~\ref{app: bfs vs ppr} additionally compares BFS with personalized PageRank (PPR).

\subsection{BFS vs.\ Personalized PageRank (PPR)}
\label{app: bfs vs ppr}

BFS propagation yields absolute success-rate improvements over personalized PageRank of $10.5\%$ on ALFWorld and $14.3\%$ on Sokoban (Tab.~\ref{tab:propagation-ablation}). BFS uses hop distance to define $\Phi_{\mathcal G}(\hat s_v)=\gamma^{d_{\mathcal G}(\hat s_v)}$ for success-reachable nodes, whereas PPR assigns mass through a random-walk computation. BFS achieves higher success in both comparisons.

\subsection{Sensitivity to \texorpdfstring{$\alpha_\text{act}$ and $\alpha_{\text{traj}}$}{action and trajectory advantage weights}}
\label{app: alpha sensitivity}

We vary the action- and trajectory-level advantage weights to examine their complementary roles.

\begin{table}[h]
\centering
\fontsize{8}{10}\selectfont
\caption{Ablation on the action-level and trajectory-level advantage weights $(\alpha_\text{act}, \alpha_{\text{traj}})$ on ALFWorld using Qwen2.5-3B-Instruct. Success rate (\%); higher is better. Default $(1,1)$ performs best among the tested weight pairs.}
\setlength\tabcolsep{8pt}
\begin{tabular}{ccc}
\toprule
$\alpha_\text{act}$ & $\alpha_{\text{traj}}$ & \textbf{ALFWorld 3B (\%)} \\
\midrule
0   & 1   & 59.4 \\
1   & 0   & 7.0  \\
1   & 0.5 & 68.8 \\
0.5 & 1   & 67.2 \\
1   & 1   & \textbf{75.5} \\
\bottomrule
\end{tabular}
\vspace{-8pt}
\end{table}

Using only the action advantage obtains a $7.0\%$ success rate. Combining action and trajectory advantages at $(\alpha_\text{act},\alpha_{\text{traj}})=(1,1)$ yields a $16.1\%$ absolute success-rate improvement over using only the trajectory advantage ($75.5\%$ vs.\ $59.4\%$), supporting the complementary roles of both signals.

\subsection{Longer Training}
\label{app: training until convergence}

We continue the main Sokoban checkpoints for an additional 100 steps (total 200 steps) with the same hyperparameters.

\begin{table}[h]
\centering
\caption{Extended training on Sokoban (Qwen2.5-VL-3B-Instruct), comparing success after 100 and 200 steps.}
\label{tab:convergence-training}
\fontsize{8}{10}\selectfont
\setlength\tabcolsep{10pt}{
\begin{tabular}{lccc}
\toprule
\textbf{Method} & \textbf{Success @100 steps (\%)} & \textbf{Success @200 steps (\%)} & \textbf{Absolute $\Delta$ (\%)} \\
\midrule
RLOO & 16.7 & 41.4 & \textbf{+24.7} \\
GRPO & 22.9 & 39.1 & +16.2 \\
GiGPO & 19.8 & 33.6 & +13.8 \\
RewardFlow (ours) & \textbf{49.5} & \textbf{70.3} & +20.8 \\
\bottomrule
\end{tabular}}
\vspace{-10pt}
\end{table}

RewardFlow retains its lead at the later checkpoint: its absolute success-rate gap over the strongest baseline is $26.6\%$ at 100 steps and $28.9\%$ at 200 steps (Tab.~\ref{tab:convergence-training}). The benefit therefore persists under this extension.

\subsection{Comparison under GiGPO's Original Training Settings}
\label{app: gigpo-matched}

We compare GiGPO and RewardFlow using the larger training settings reported by GiGPO~\citep{feng2025group}: 150 steps for ALFWorld and WebShop, 200 steps for Sokoban, an ALFWorld interaction horizon of 50, and a Sokoban training batch size of 32. Tab.~\ref{tab:gigpo-matched} reports mean $\pm$ standard deviation over 3 inference seeds under this protocol. For ALFWorld-3B, both methods use training seed 0 and checkpoint 150; the same two estimates are reused in the preprocessing control, whose full evaluation protocol is in Appendix~\ref{app: matched preprocessing control}.

Our GiGPO reproduction is below the originally reported performance. We compare both methods using the same specified training settings on NVIDIA H20 and A100 GPUs.

\begin{table}[h]
\centering
\caption{Comparison of GiGPO and RewardFlow under GiGPO's original training settings (150 steps for ALFWorld/WebShop, 200 steps for Sokoban, batch size 32 for Sokoban). Success rate (\%, mean $\pm$ std over 3 inference seeds at fixed checkpoints); \textbf{bold} indicates the better method.}
\label{tab:gigpo-matched}
\fontsize{8}{10}\selectfont
\setlength\tabcolsep{8pt}{
\begin{tabular}{lcccc}
\toprule
\textbf{Method} & \textbf{ALFWorld 1.5B} & \textbf{ALFWorld 3B} & \textbf{Sokoban 3B} & \textbf{WebShop 1.5B} \\
\midrule
GiGPO               & $74.2\pm1.7$ & $75.2\pm1.6$ & $78.4\pm2.6$ & $60.2\pm2.2$ \\
RewardFlow & $\mathbf{76.8\pm1.0}$ & $\mathbf{82.8\pm0.7}$ & $\mathbf{81.8\pm2.3}$ & $\mathbf{66.9\pm1.5}$ \\
\bottomrule
\end{tabular}}
\vspace{-8pt}
\end{table}

RewardFlow exceeds GiGPO in all four larger-training-budget comparisons.

\subsection{Rollout-Budget Sensitivity}
\label{app: exploration-ablation}

We vary the rollout budget while comparing RewardFlow with GiGPO. Tab.~\ref{tab:exploration-ablation} reports the resulting success rates and graph sizes. RewardFlow exceeds GiGPO at budgets of 4, 6, and 8, including an absolute success-rate improvement of $13.3\%$ with 4 trajectories.

\begin{table}[h]
\centering
\caption{Ablation on number of trajectories per training task using Qwen2.5-1.5B-Instruct in ALFWorld. Success rate (\%) reported; RewardFlow results include average graph statistics.}
\label{tab:exploration-ablation}
\fontsize{8}{10}\selectfont
\setlength\tabcolsep{3pt}{
\begin{tabular}{cccccc}
\toprule
\textbf{Method} & \textbf{Trajectories} & \makecell[c]{\textbf{Avg.} \\ \textbf{Nodes}} & \makecell[c]{\textbf{Avg.} \\ \textbf{Edges}} & \makecell[c]{\textbf{Success} \\ \textbf{Rate (\%)}} \\
\midrule
\multirow{3}{*}{GiGPO} & 4 & - & - & 28.9 \\
 & 6 & - & - & 51.6 \\
 & 8 & - & - & 50.8 \\
\midrule
\multirow{3}{*}{RewardFlow} & 4 & 17.1 & 28.8 & \textbf{42.2} \\
 & 6 & 22.9 & 41.3 & \textbf{53.1} \\
 & 8 & 28.8 & 55.9 & \textbf{62.8} \\
\bottomrule
\end{tabular}}
\vspace{-8pt}
\end{table}

\subsection{Graph Size and Runtime}
\label{app: training efficiency}

Tab.~\ref{tab:graph-size} reports graph sizes for 8 sampled trajectories per task. ALFWorld and WebShop use Qwen2.5-1.5B-Instruct with horizons of 25 and 15 interactive steps, respectively; Sokoban uses Qwen2.5-VL-3B-Instruct with horizon 15. The separate runtime profile in Fig.~\ref{fig:rewardflow_teaser} (right) reports graph construction and reward propagation times for Qwen2.5-3B-Instruct on ALFWorld and WebShop and Qwen2.5-VL-3B-Instruct on Sokoban. These graph operations use no additional model calls. LLM-based state construction has a separate inference cost.

\begin{table}[h!]
\centering
\caption{Average and maximum size of induced state graphs from 8 on-policy trajectories (Qwen2.5-1.5B/VL-3B-Instruct). Maxima refer to the sampled graphs in these configurations.}
\label{tab:graph-size}
\fontsize{8}{10}\selectfont
\setlength\tabcolsep{10pt}{
\begin{tabular}{lcccc}
\toprule
\textbf{Environment} & \textbf{Avg. Nodes} & \textbf{Avg. Edges} & \textbf{Max Nodes} & \textbf{Max Edges} \\
\midrule
ALFWorld & 28.8 & 55.9 & 42 & 77 \\
WebShop & 36.8 & 48.2 & 54 & 65 \\
Sokoban & 14.0 & 21.0 & 27 & 39 \\
\bottomrule
\end{tabular}}
\vspace{-8pt}
\end{table}

\paragraph{Entropy and graph size during training.} We examine how policy entropy, retained graph size, and success vary across checkpoints. We measure policy entropy as $H(p)=-\sum_{z\in\mathcal W}p(z)\log p(z)$, where $\mathcal W$ is the token vocabulary and $z$ a vocabulary token; $\mathcal W$ is distinct from the graph-node set $\mathcal V$. To study the interplay between entropy and graph structure, we evaluate three checkpoints of the Qwen2.5-1.5B-Instruct model during RewardFlow training on ALFWorld: step 0 (initial supervised model), step 50, and step 100. For each checkpoint, we collect $G=8$ trajectories per task using group sampling, compute the entropy, and record graph statistics along with validation performance.

\begin{table}[h!]
\centering
\caption{Training progress of Qwen2.5-1.5B-Instruct on the target task. Entropy, invalid rates, graph size, and success rate across training steps.}
\label{tab:qwen-training-progress}
\fontsize{8}{8}\selectfont
\setlength\tabcolsep{6pt}
\begin{tabular}{cccccccc}
\toprule
\textbf{Training Step} & \textbf{Entropy} & \makecell{\textbf{Invalid State} \\ \textbf{Rate (\%)}} & \makecell{\textbf{Invalid Action} \\ \textbf{Rate (\%)}} & \textbf{Avg. Node} & \textbf{Avg. Edge} & \makecell{\textbf{Success} \\ \textbf{Rate (\%)}} \\
\midrule
0 & 1.099 & 33.1 & 41.0 & 24.1 & 46.9 & 4.1 \\
50 & 0.555 & 0.3 & 0 & 33.1 & 55.9 & 44.5 \\
100 & 0.295 & 0.1 & 0 & 19.9 & 30.6 & 62.8 \\
\bottomrule
\end{tabular}
\vspace{-8pt}
\end{table}

At step 0, the policy has the highest entropy among these checkpoints (1.099) and generates many invalid actions, which are removed during graph construction.
At step 50, entropy is lower (0.555), invalid actions are rare, and the observed graphs are largest among these checkpoints (33.1 nodes and 55.9 edges on average). At step 100, entropy is lowest (0.295), the graphs are smaller (19.9 nodes and 30.6 edges), and validation success is highest (62.8\%). Across these checkpoints, success increases as invalid actions decrease, while graph size first increases and then decreases.

\subsection{Offline Synthesis of a State Normalizer}
\label{app: learned normalizer}

We test whether an executable ALFWorld normalizer can be synthesized without exposing the handwritten normalizer or its reward outputs to the proposer and validation judge. The GPT-5.6 proposer/judge procedure uses reward-blind interaction traces and feedback on state-transition consistency and merge consistency. Its collection contains 1,024 trajectories over 256 task identities, split by task into 161 development, 49 validation, and 46 final-test identities. The reported R04 candidate uses merge supervision.

After freezing the candidate, we compare its induced action rewards with the handwritten interface on a separate fixed comparison dump of 128 trajectories (16 groups of eight, 2,407 state transitions), using potential discount factor $0.95$. This post-hoc agreement assessment applies the same trajectory cleaner to both normalizers.

\begin{table}[htbp]
\centering
\small
\caption{Post-hoc action-reward agreement with the handwritten normalizer on the same fixed trajectories. Exact agreement compares reward values; direction agreement compares positive, zero, or negative signs.}
\label{tab:learned-normalizer-agreement}
\begin{tabular}{lccc}
\toprule
Normalizer & Exact (\%) & Direction (\%) & Nonzero-support exact (\%) \\
\midrule
Identity & 84.67 & 90.78 & 34.11 \\
Frozen merge-supervised R04 & 95.26 & 96.88 & 77.69 \\
\bottomrule
\end{tabular}
\end{table}

The absolute increases in all-action exact and direction agreement are $10.59\%$ and $6.10\%$. Because both interfaces assign zero reward to many actions, we also report exact agreement over the union of their nonzero-reward supports; the R04 figure is $397/511=77.69\%$, and its full-dump reward correlation is $0.601$. The synthesized normalizer more closely reproduces the reference action rewards on this comparison dump.

\subsection{Supplementary Math-TIR Evaluation}
\label{app: math tir supplement}

Mathematical tool-integrated reasoning (Math-TIR) is a setting in which a model solves math problems by alternating written reasoning with Python execution. The Structured-Decoding variant identifies nodes by exact interaction histories and branches generation at reasoning/tool-action boundaries (Appendix~\ref{app: state preprocessing}). Its baseline, Graph-GRPO, runs GRPO with the same branching rules. Both use Qwen2.5-3B-Instruct, training seed 0, and checkpoints after 200 updates. We also report a RewardFlow variant that uses a frozen Qwen2.5-14B state normalizer as its state-construction interface.

\begin{table}[htbp]
\centering
\small
\caption{Supplementary Math-TIR results (\%). Offline audited accuracy on MATH-500 (500 problems) and AIME24/25/26 (30 each), with eight samples per problem and inference seed 0.}
\label{tab:math-tir-supplement}
\begin{tabular}{llrrr}
\toprule
Benchmark & Method / state interface & mean@8 & pass@8 & maj@8 \\
\midrule
MATH-500 & Graph-GRPO / exact prefix & 62.250 & 82.600 & 71.400 \\
 & RewardFlow / Structured-Decoding & 62.775 & 82.000 & 70.600 \\
 & RewardFlow / frozen LLM normalization & 59.80 & 81.60 & 70.20 \\
\midrule
AIME24 & Graph-GRPO / exact prefix & 7.500 & 20.000 & 13.333 \\
 & RewardFlow / Structured-Decoding & 6.667 & 20.000 & 10.000 \\
 & RewardFlow / frozen LLM normalization & 7.08 & 16.67 & 10.00 \\
\midrule
AIME25 & Graph-GRPO / exact prefix & 5.000 & 16.667 & 10.000 \\
 & RewardFlow / Structured-Decoding & 2.500 & 13.333 & 3.333 \\
 & RewardFlow / frozen LLM normalization & 3.33 & 10.00 & 6.67 \\
\midrule
AIME26 & Graph-GRPO / exact prefix & 4.583 & 13.333 & 10.000 \\
 & RewardFlow / Structured-Decoding & 6.667 & 26.667 & 10.000 \\
 & RewardFlow / frozen LLM normalization & 6.67 & 16.67 & 13.33 \\
\bottomrule
\end{tabular}
\end{table}

Here mean@8 is average sample accuracy, pass@8 is the fraction of problems with at least one correct sample, and maj@8 is majority-vote accuracy. Structured-Decoding improves mean@8 on MATH-500 and AIME26 and pass@8 on AIME26; its AIME24 and AIME25 results are lower or tied across metrics. These experiments extend graph-based action-level reward estimation to mathematical reasoning through two state-construction interfaces. The frozen-LLM entries retain the source report's two-decimal percentage precision.

\paragraph{Scoring audit.}
The reported metrics use offline \texttt{qwen\_math} scoring of the completed evaluations. Audits found 19 online false negatives for Structured-Decoding and 16 for Graph-GRPO among 4,720 samples, all on MATH-500, due to a symbolic-equivalence fallback that fails in worker threads. One additional Structured-Decoding sample was initially over-counted by offline reconstruction from a stale earlier boxed answer; the audited MATH-500 mean is corrected to $2511/4000=62.775\%$. The LLM-normalized evaluation records 15 analogous online/offline discrepancies among 4,000 MATH-500 samples; independent offline rescoring matches its reported metrics.

\clearpage
\subsection{State Graph Cases}
\label{app: state graph cases}

We show the original state-graph visualizations constructed from 1, 2, and 3 sampled rollouts in ALFWorld, WebShop, and Sokoban. The WebShop and Sokoban rollouts were sampled from checkpoints after 100 training steps. The figures retain their original node and edge color scales and illustrate the state transitions shared across rollouts. Node 0 is the initial state, and $\bigstar$ marks task completion. The accompanying lists retain the observations and action strings, including different states reached at the same location.

\paragraph{ALFWorld.} These cases connect alternative routes through shared household states. The accompanying observations track the kettle's status during cleaning and placement.

\begingroup
\setlength{\intextsep}{6pt}
\begin{figure}[h!]
    \centering
    \includegraphics[width=\linewidth]{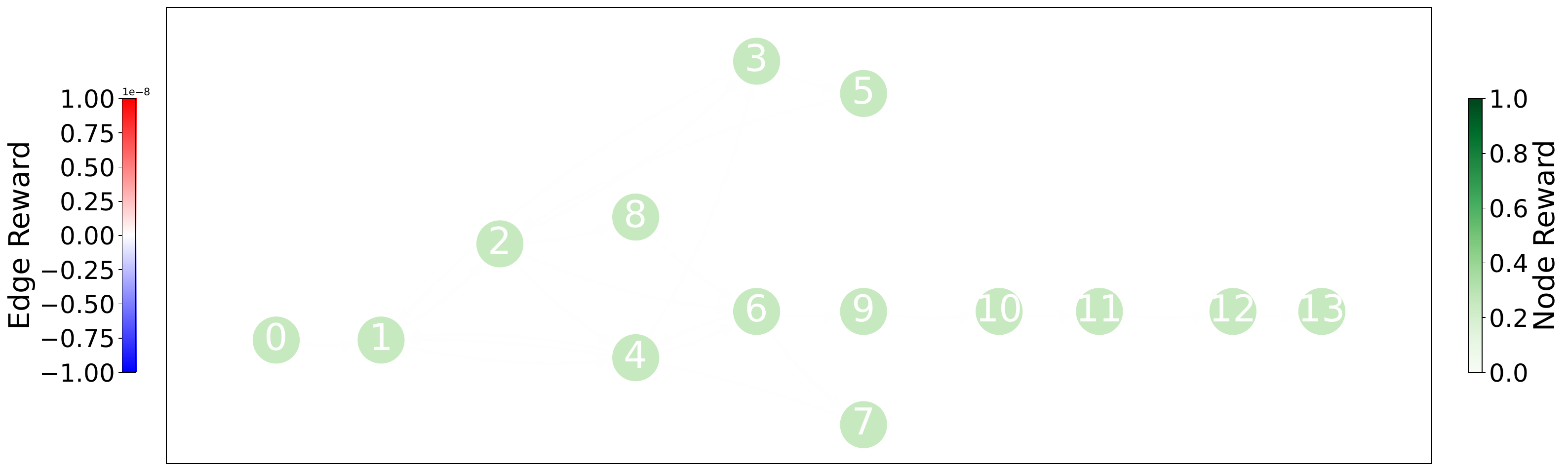}
    \caption{The original state graph constructed from 1 sampled rollout in ALFWorld. Darker nodes indicate higher values on the source plot's color scale. Node 0 is the initial state, and $\bigstar$ is the successful terminal state. Node and action descriptions appear in Figures~\ref{fig:node_list_alfworld} and~\ref{fig:edge_list_alfworld}.}
    \label{fig:state_graph_alfworld_1}
\end{figure}

\begin{figure}[h!]
    \centering
    \includegraphics[width=\linewidth]{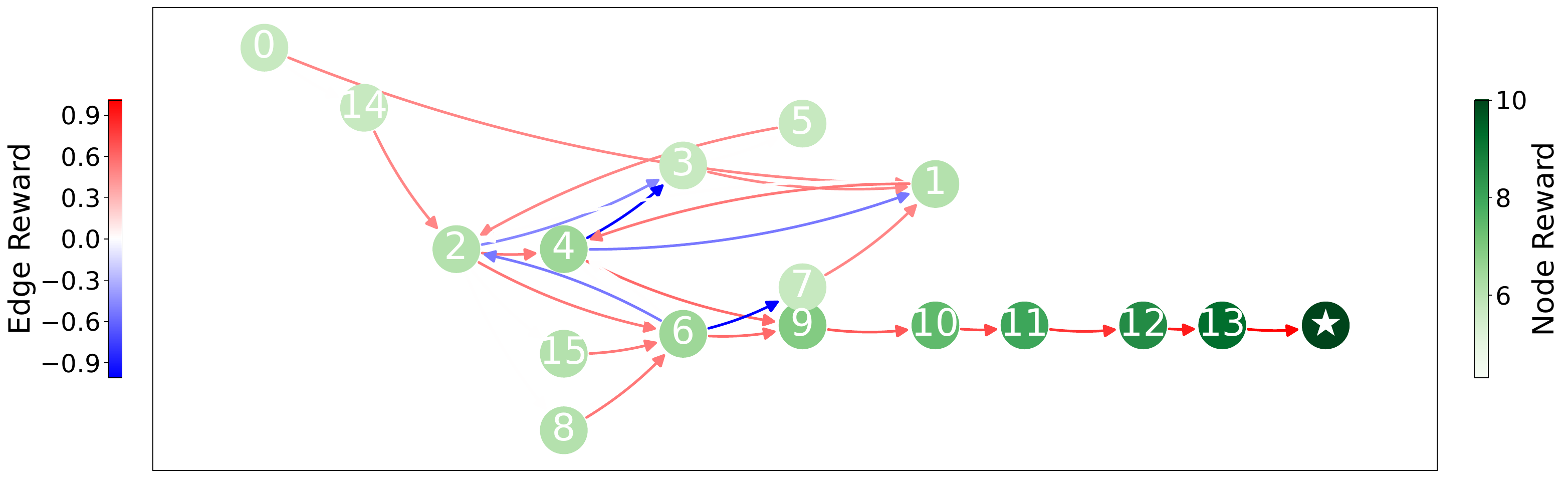}
    \caption{The original state graph constructed from 2 sampled rollouts in ALFWorld. Darker nodes indicate higher values on the source plot's color scale. Node 0 is the initial state, and $\bigstar$ is the successful terminal state. Node and action descriptions appear in Figures~\ref{fig:node_list_alfworld} and~\ref{fig:edge_list_alfworld}.}
    \label{fig:state_graph_alfworld_2}
\end{figure}

\begin{figure}[h!]
    \centering
    \includegraphics[width=\linewidth]{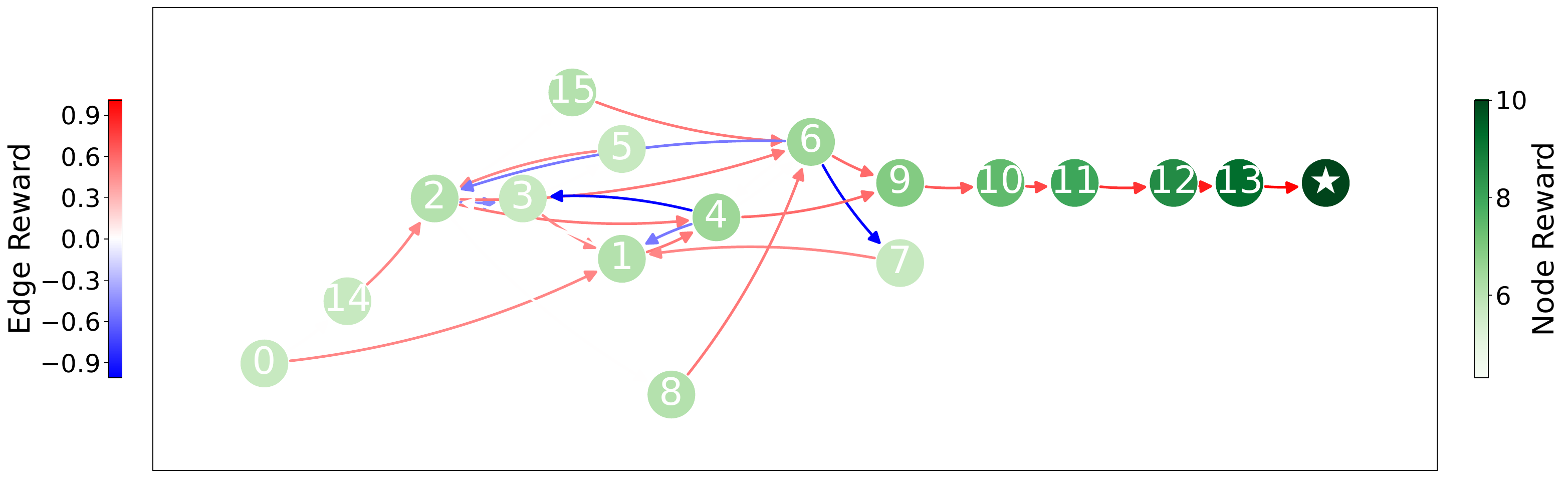}
    \caption{The original state graph constructed from 3 sampled rollouts in ALFWorld. Darker nodes indicate higher values on the source plot's color scale. Node 0 is the initial state, and $\bigstar$ is the successful terminal state. Node and action descriptions appear in Figures~\ref{fig:node_list_alfworld} and~\ref{fig:edge_list_alfworld}.}
    \label{fig:state_graph_alfworld_3}
    \label{fig:state_graph_alfworld}
\end{figure}
\endgroup

\clearpage

\begin{figure}
    \centering

\begin{tcolorbox}[colback=blue!10, colupper=black, colframe=blue!40!black, title=Node list of cases in ALFWorld, fontupper=\scriptsize, width=\textwidth, label=lst:node_alfworld]

\textbf{Node 0:} -= Welcome to TextWorld, ALFRED! =-

You are in the middle of a room. Looking quickly around you, you see a cabinet 13, a cabinet 12, a cabinet 11, a cabinet 10, a cabinet 9, a cabinet 8, a cabinet 7, a cabinet 6, a cabinet 5, a cabinet 4, a cabinet 3, a cabinet 2, a cabinet 1, a coffeemachine 1, a countertop 2, a countertop 1, a diningtable 1, a drawer 4, a drawer 3, a drawer 2, a drawer 1, a fridge 1, a garbagecan 1, a microwave 1, a shelf 3, a shelf 2, a shelf 1, a sinkbasin 1, a stoveburner 4, a stoveburner 3, a stoveburner 2, a stoveburner 1, and a toaster 1.

Your task is to: clean some kettle and put it in shelf. You are not carrying anything.

\textbf{Node 1:} You arrive at coffeemachine 1. On the coffeemachine 1, you see nothing. You are not carrying anything.

\textbf{Node 2:} You arrive at countertop 1. On the countertop 1, you see a fork 1, and a lettuce 1. You are not carrying anything.

\textbf{Node 3:} You arrive at stoveburner 1. On the stoveburner 1, you see nothing. You are not carrying anything.

\textbf{Node 4:} You arrive at countertop 2. On the countertop 2, you see a book 2, a book 1, a butterknife 1, a dishsponge 2, a fork 2, a kettle 1, a lettuce 2, a plate 2, a sink 1, a statue 2, and a tomato 1. You are not carrying anything.

\textbf{Node 5:} You pick up the kettle 1 from the countertop 2. You are carrying: a kettle 1.

\textbf{Node 6:} You arrive at sinkbasin 1. On the sinkbasin 1, you see nothing. You are carrying: a kettle 1.

\textbf{Node 7:} You clean the kettle 1 using the sinkbasin 1. You are carrying: kettle 1 (cleaned).

\textbf{Node 8:} You arrive at shelf 1. On the shelf 1, you see nothing. You are carrying: kettle 1 (cleaned).

\textbf{Node $\bigstar$:} You Won!

\textbf{Node 10:} You arrive at fridge 1. The fridge 1 is closed. You are not carrying anything.

\textbf{Node 11:} You open the fridge 1. The fridge 1 is open. In it, you see a apple 2, a bowl 1, a cup 1, a egg 1, a plate 1, and a potato 1. You are not carrying anything.

\textbf{Node 12:} You arrive at shelf 1. On the shelf 1, you see nothing. You are not carrying anything.

\textbf{Node 13:} You arrive at fridge 1. The fridge 1 is open. In it, you see a apple 2, a bowl 1, a cup 1, a egg 1, a plate 1, and a potato 1. You are not carrying anything.

\textbf{Node 14:} You arrive at diningtable 1. On the diningtable 1, you see a apple 1, a bread 2, a bread 1, a houseplant 1, a knife 1, a saltshaker 2, a soapbottle 1, a spatula 2, a spatula 1, a statue 1, a tomato 3, and a tomato 2. You are not carrying anything.

\textbf{Node 15:} You arrive at stoveburner 2. On the stoveburner 2, you see a pan 1. You are not carrying anything.

\textbf{Node 16:} You arrive at sinkbasin 1. On the sinkbasin 1, you see nothing. You are not carrying anything.

\end{tcolorbox}
    \caption{The node list of cases in ALFWorld in Figures~\ref{fig:state_graph_alfworld_1},~\ref{fig:state_graph_alfworld_2}, and~\ref{fig:state_graph_alfworld_3}.}
    \label{fig:node_list_alfworld}
\end{figure}

\begin{figure}
    \centering

\begin{tcolorbox}[colback=red!5!white, colupper=black, colframe=red!40!black, title=Edge list of cases in ALFWorld, fontupper=\scriptsize, width=\textwidth]

\textbf{Node 0 $->$ Node 1:} go to coffeemachine 1

\textbf{Node 0 $->$ Node 14:} go to diningtable 1

\textbf{Node 1 $->$ Node 2:} go to countertop 1

\textbf{Node 2 $->$ Node 13:} go to fridge 1

\textbf{Node 2 $->$ Node 3:} go to stoveburner 1

\textbf{Node 2 $->$ Node 16:} go to sinkbasin 1

\textbf{Node 2 $->$ Node 12:} go to shelf 1

\textbf{Node 2 $->$ Node 10:} go to fridge 1

\textbf{Node 3 $->$ Node 13:} go to fridge 1

\textbf{Node 3 $->$ Node 14:} go to diningtable 1

\textbf{Node 3 $->$ Node 10:} go to fridge 1

\textbf{Node 3 $->$ Node 4:} go to countertop 2

\textbf{Node 4 $->$ Node 5:} take kettle 1 from countertop 2

\textbf{Node 5 $->$ Node 6:} go to sinkbasin 1

\textbf{Node 6 $->$ Node 7:} clean kettle 1 with sinkbasin 1

\textbf{Node 7 $->$ Node 8:} go to shelf 1

\textbf{Node 8 $->$ Node $\bigstar$:} move kettle 1 to shelf 1

\textbf{Node 10 $->$ Node 11:} open fridge 1

\textbf{Node 10 $->$ Node 14:} go to diningtable 1

\textbf{Node 11 $->$ Node 2:} go to countertop 1

\textbf{Node 12 $->$ Node 13:} go to fridge 1

\textbf{Node 13 $->$ Node 2:} go to countertop 1

\textbf{Node 13 $->$ Node 4:} go to countertop 2

\textbf{Node 13 $->$ Node 15:} go to stoveburner 2

\textbf{Node 13 $->$ Node 3:} go to stoveburner 1

\textbf{Node 14 $->$ Node 3:} go to stoveburner 1

\textbf{Node 14 $->$ Node 2:} go to countertop 1

\textbf{Node 15 $->$ Node 14:} go to diningtable 1

\textbf{Node 16 $->$ Node 13:} go to fridge 1

\end{tcolorbox}

    \caption{The edge list of cases in ALFWorld in Figures~\ref{fig:state_graph_alfworld_1},~\ref{fig:state_graph_alfworld_2}, and~\ref{fig:state_graph_alfworld_3}.}
    \label{fig:edge_list_alfworld}
\end{figure}


\clearpage
\paragraph{WebShop.} Different search queries can reach the same product page, joining purchase paths that share the subsequent attribute-selection actions.

\begin{figure}[h!]
    \centering
    \includegraphics[width=\linewidth]{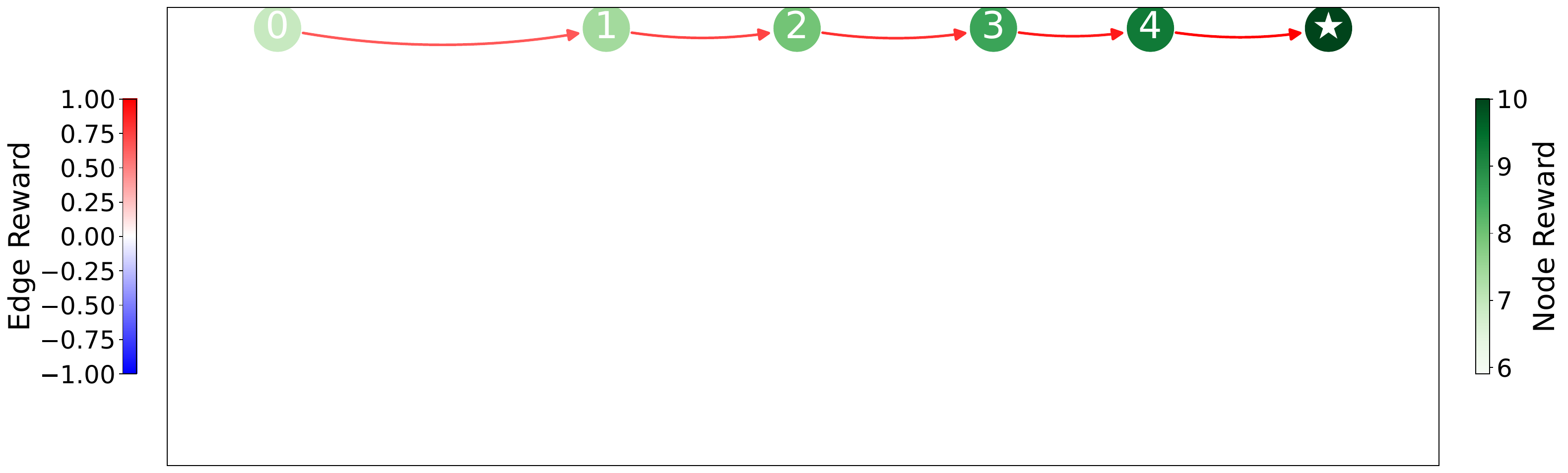}
    \caption{The original state graph constructed from 1 sampled rollout in WebShop. Darker nodes indicate higher values on the source plot's color scale. Node 0 is the initial state, and $\bigstar$ is the successful terminal state. Node and action descriptions appear in Figures~\ref{fig:node_list_webshop} and~\ref{fig:edge_list_webshop}.}
    \label{fig:state_graph_webshop_1}
\end{figure}

\begin{figure}[h!]
    \centering
    \includegraphics[width=\linewidth]{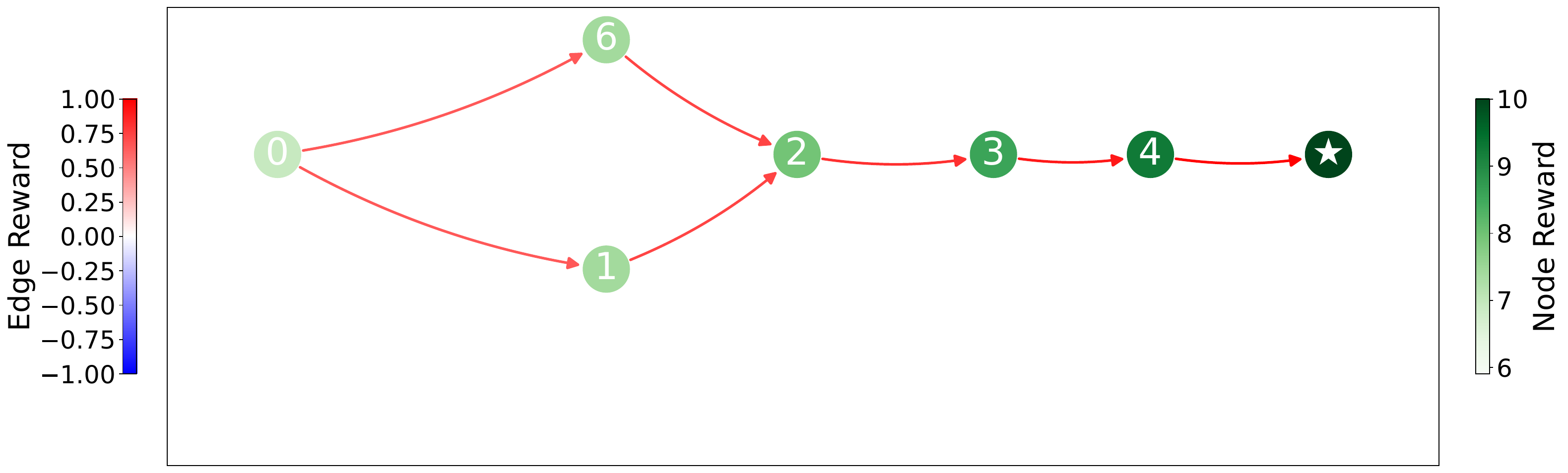}
    \caption{The original state graph constructed from 2 sampled rollouts in WebShop. Darker nodes indicate higher values on the source plot's color scale. Node 0 is the initial state, and $\bigstar$ is the successful terminal state. Node and action descriptions appear in Figures~\ref{fig:node_list_webshop} and~\ref{fig:edge_list_webshop}.}
    \label{fig:state_graph_webshop_2}
\end{figure}

\begin{figure}[h!]
    \centering
    \includegraphics[width=\linewidth]{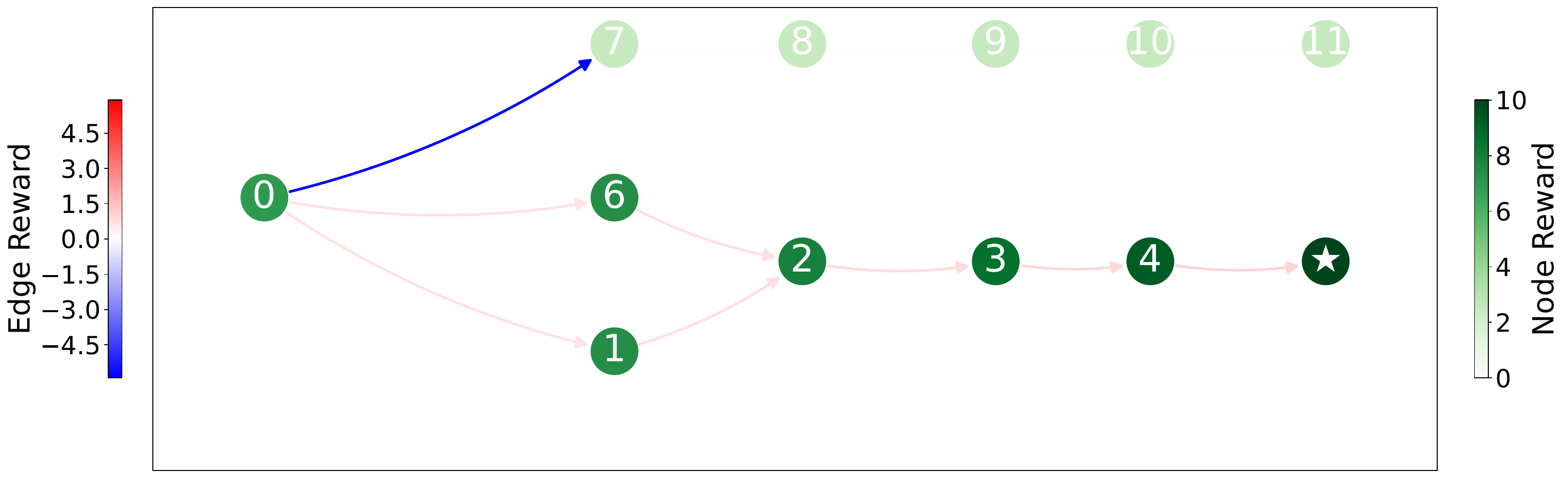}
    \caption{The original state graph constructed from 3 sampled rollouts in WebShop. Darker nodes indicate higher values on the source plot's color scale. Node 0 is the initial state, and $\bigstar$ is the successful terminal state. Node and action descriptions appear in Figures~\ref{fig:node_list_webshop} and~\ref{fig:edge_list_webshop}.}
    \label{fig:state_graph_webshop_3}
    \label{fig:state_graph_webshop}
\end{figure}

\begin{figure}
    \centering

\begin{tcolorbox}[colback=blue!10, colupper=black, colframe=blue!40!black, title=Node list of examples in WebShop, fontupper=\tiny, width=\textwidth]

\textbf{Node 0:} 'Search'

\textbf{Node 1:} 'Back to Search' [SEP] 'Page 1 (Total results: 50)' [SEP] 'Next >' [SEP] 'B07S7HDC88' [SEP] 'VanciLin Mens Casual Leather Fashion Slip-on Loafers' [SEP] '\textdollar33.99 to \textdollar35.99' [SEP] 'B00ITCBD3Y' [SEP] 'Laredo Womens Spellbound Studded Square Toe Dress Boots Mid Calf Low Heel 1-2" - Black' [SEP] '\textdollar132.0 to \textdollar150.26' [SEP] 'B09N9T673Q' [SEP] 'Women's Heeled Sandals Sexy Minimalistic Cut out Rear Zipper Wedge Roman Sandals Casual Solid Stiletto Pump for Night Party Comfy Lightweight Summer Shoes' [SEP] '\textdollar100.0' [SEP] 'B07KK7G5YJ' [SEP] 'Laeacco Steampunk Clock Backdrops 6.5x6.5ft Vinyl Photography Background Shabby Clock Photo Backdrop Old Clock Punk Grunge Culture Adults Brave Boy Portraits Photo Studio Props' [SEP] '\textdollar24.59' [SEP] 'B09P1HJ32G' [SEP] 'VDSOIUTYHFV Infrared Digital Night Vision Device, High-Definition Day and Night Dual-use, Outdoor Photos and Videos are All Black, for Adults Bird Watching, Traveling, Hiking' [SEP] '\textdollar840.99' [SEP] 'B089KGBJGW' [SEP] 'ZINUS 10 Inch Cooling Copper ADAPTIVE Pocket Spring Hybrid Mattress / Moisture Wicking Cover / Cooling Foam / Pocket Innersprings for Motion Isolation / Mattress-in-a-Box, Full' [SEP] '\textdollar289.0' [SEP] 'B09P57ZKGP' [SEP] '10 Pack 512MB USB Flash Drives Gift Stick Business Bidding Thumb Drive USB 2.0 high Speed Pen Drives (512MB, Black)' [SEP] '\textdollar100.0' [SEP] 'B001C6JXJU' [SEP] 'Premier HiDef Grey Electric Projection Screen with Low Voltage Motor Viewing Area: 50" H x 80" W' [SEP] '\textdollar2721.34' [SEP] 'B085HGD5Z3' [SEP] 'haoricu Womens Platform Sandals Summer Fashion Ladies Bow Tie Flats Rome Casual Basic Slipper' [SEP] '\textdollar14.98 to \textdollar25.98' [SEP] 'B09QKJHF4Q' [SEP] 'LIUEONG Lady Metal Buckle Platform Shoes Booties Comfortable High Top Boots Casual Shoes Comfort Women's Closed Toe Ankle Strap' [SEP] '\textdollar37.99'

\textbf{Node 2:} 'Back to Search' [SEP] '< Prev' [SEP] 'size' [SEP] '6.5' [SEP] '7' [SEP] '8' [SEP] '8.5' [SEP] '9' [SEP] '9.5' [SEP] '10' [SEP] '10.5' [SEP] '11' [SEP] '11.5' [SEP] '12' [SEP] '13' [SEP] '14' [SEP] 'color' [SEP] '1877black' [SEP] '1877blue' [SEP] '228black' [SEP] 'black137' [SEP] 'blue137' [SEP] 'l.brown137' [SEP] 'r.brown137' [SEP] '1877r.brown' [SEP] '228blue' [SEP] 'r.brown-hole228' [SEP] '1877brown' [SEP] '228r.brown-a' [SEP] 'y.brown-hole228' [SEP] 'black-hole228' [SEP] '228l.brown' [SEP] 'r.brown2070' [SEP] '228r.brown' [SEP] 'black2070' [SEP] 'r.brown2626' [SEP] 'r.brown-hole1887' [SEP] 'black2626' [SEP] 'black1901' [SEP] 'r.brown2060' [SEP] 'l.brown2060' [SEP] 'VanciLin Mens Casual Leather Fashion Slip-on Loafers' [SEP] 'Price: \textdollar33.99 to \textdollar35.99' [SEP] 'Rating: N.A.' [SEP] 'Description' [SEP] 'Features' [SEP] 'Reviews' [SEP] 'Buy Now'

\textbf{Node 3:} 'Back to Search' [SEP] '< Prev' [SEP] 'size' [SEP] '6.5' [SEP] '7' [SEP] '8' [SEP] '8.5' [SEP] '9' [SEP] '9.5' [SEP] '10' [SEP] '10.5' [SEP] '11' [SEP] '11.5' [SEP] '12' [SEP] '13' [SEP] '14' [SEP] 'color' [SEP] '1877black' [SEP] '1877blue' [SEP] '228black' [SEP] 'black137' [SEP] 'blue137' [SEP] 'l.brown137' [SEP] 'r.brown137' [SEP] '1877r.brown' [SEP] '228blue' [SEP] 'r.brown-hole228' [SEP] '1877brown' [SEP] '228r.brown-a' [SEP] 'y.brown-hole228' [SEP] 'black-hole228' [SEP] '228l.brown' [SEP] 'r.brown2070' [SEP] '228r.brown' [SEP] 'black2070' [SEP] 'r.brown2626' [SEP] 'r.brown-hole1887' [SEP] 'black2626' [SEP] 'black1901' [SEP] 'r.brown2060' [SEP] 'l.brown2060' [SEP] 'VanciLin Mens Casual Leather Fashion Slip-on Loafers' [SEP] 'Price: \textdollar33.99 to \textdollar35.99' [SEP] 'Rating: N.A.' [SEP] 'Description' [SEP] 'Features' [SEP] 'Reviews' [SEP] 'Buy Now' clicked attributes: [228r.brown-a]

\textbf{Node 4:} 'Back to Search' [SEP] '< Prev' [SEP] 'size' [SEP] '6.5' [SEP] '7' [SEP] '8' [SEP] '8.5' [SEP] '9' [SEP] '9.5' [SEP] '10' [SEP] '10.5' [SEP] '11' [SEP] '11.5' [SEP] '12' [SEP] '13' [SEP] '14' [SEP] 'color' [SEP] '1877black' [SEP] '1877blue' [SEP] '228black' [SEP] 'black137' [SEP] 'blue137' [SEP] 'l.brown137' [SEP] 'r.brown137' [SEP] '1877r.brown' [SEP] '228blue' [SEP] 'r.brown-hole228' [SEP] '1877brown' [SEP] '228r.brown-a' [SEP] 'y.brown-hole228' [SEP] 'black-hole228' [SEP] '228l.brown' [SEP] 'r.brown2070' [SEP] '228r.brown' [SEP] 'black2070' [SEP] 'r.brown2626' [SEP] 'r.brown-hole1887' [SEP] 'black2626' [SEP] 'black1901' [SEP] 'r.brown2060' [SEP] 'l.brown2060' [SEP] 'VanciLin Mens Casual Leather Fashion Slip-on Loafers' [SEP] 'Price: \textdollar33.99 to \textdollar35.99' [SEP] 'Rating: N.A.' [SEP] 'Description' [SEP] 'Features' [SEP] 'Reviews' [SEP] 'Buy Now' clicked attributes: [228r.brown-a, 6.5]

\textbf{Node $\bigstar$:} Thank you for shopping with us! [SEP] Your code: [SEP] None [SEP] (Paste it in your MTurk interface.) [SEP] Purchased [SEP] asin [SEP] B07S7HDC88 [SEP] options [SEP] {"color": "228r.brown-a", "size": "6.5"} [SEP] attrs [SEP] None [SEP] category [SEP] None [SEP] query [SEP] None [SEP] product category [SEP] None [SEP] Target [SEP] asin [SEP] options [SEP] attrs [SEP] price upper [SEP] instuction text [SEP] category [SEP] product category [SEP] query [SEP] Goal [SEP] None [SEP] Reward [SEP] Your score (min 0.0, max 1.0) [SEP] 1.0 [SEP] Reward Details [SEP] None

\textbf{Node 6:} 'Back to Search' [SEP] 'Page 1 (Total results: 23)' [SEP] 'Next >' [SEP] 'B07S7HDC88' [SEP] 'VanciLin Mens Casual Leather Fashion Slip-on Loafers' [SEP] '\textdollar33.99 to \textdollar35.99' [SEP] 'B07YCGBPRD' [SEP] 'OTAO Privacy Screen Protector for iPhone 11 Pro Max/iPhone Xs Max 6.5 Inch True 28°Anti Spy Tempered Glass Full-Coverage (2-pack)' [SEP] '\textdollar9.98' [SEP] 'B085HGD5Z3' [SEP] 'haoricu Womens Platform Sandals Summer Fashion Ladies Bow Tie Flats Rome Casual Basic Slipper' [SEP] '\textdollar14.98 to \textdollar25.98' [SEP] 'B07KK7G5YJ' [SEP] 'Laeacco Steampunk Clock Backdrops 6.5x6.5ft Vinyl Photography Background Shabby Clock Photo Backdrop Old Clock Punk Grunge Culture Adults Brave Boy Portraits Photo Studio Props' [SEP] '\textdollar24.59' [SEP] 'B09J95S478' [SEP] 'My Sanity Question Giraffe Christmas Pattern Black Ugly Wool Christmas Sweater Pullover Long Sleeve Sweater for Men Women, Couple Matching, Friends' [SEP] '\textdollar39.99' [SEP] 'B09N9T673Q' [SEP] 'Women's Heeled Sandals Sexy Minimalistic Cut out Rear Zipper Wedge Roman Sandals Casual Solid Stiletto Pump for Night Party Comfy Lightweight Summer Shoes' [SEP] '\textdollar100.0' [SEP] 'B07TN3Y9H1' [SEP] 'Maxim 32478SWSBRBP Finn Mid-Century Modern Satin White Glass Ball Bath Vanity Wall Mount, 5-Light 200 Total Watts, 7"H x 44"W, Satin Brass/Brushed Platinum' [SEP] '\textdollar219.99' [SEP] 'B09QKJHF4Q' [SEP] 'LIUEONG Lady Metal Buckle Platform Shoes Booties Comfortable High Top Boots Casual Shoes Comfort Women's Closed Toe Ankle Strap' [SEP] '\textdollar37.99' [SEP] 'B0912FQSMT' [SEP] '450 Bluetooth Marine Gauge Receiver and 6.5-Inch Speaker Package' [SEP] '\textdollar224.02' [SEP] 'B088TWNTRB' [SEP] 'MIA Delena' [SEP] '\textdollar21.89 to \textdollar55.08'

\textbf{Node 7:} 'Back to Search' [SEP] 'Page 1 (Total results: 50)' [SEP] 'Next $>$' [SEP] 'B07RPMCNKD' [SEP] 'ULTRAIDEAS Women's Comfy Lightweight Slippers Non-Slip House Shoes for Indoor \& Outdoor' [SEP] '\textdollar19.99 to \textdollar21.99' [SEP] 'B07HP6LVRS' [SEP] 'MCICI Mens Loafers Moccasin Driving Shoes Premium Genuine Leather Casual Slip On Flats Fashion Slipper Breathable Big Size' [SEP] '\textdollar15.99 to \textdollar32.99' [SEP] 'B07S7HDC88' [SEP] 'VanciLin Mens Casual Leather Fashion Slip-on Loafers' [SEP] '\textdollar33.99 to \textdollar35.99' [SEP] 'B09HMCKZQW' [SEP] 'Copercn Dress Ankle Booties For Women Ladies Fashion Casual Chunky Block High Heels Dress Pump Thermal Short Boots Winter Fall Dressy Short Boots For Business Work Wedding Party Dresses' [SEP] '\textdollar17.99 to \textdollar23.99' [SEP] 'B09F2MFG1Q' [SEP] 'Womens Mens house slippers Memory Foam Garden Plush Lining Slip On Rubber sole Indoor Outdoor House Casual Shoes (Black adult-10.5-11)' [SEP] '\textdollar100.0' [SEP] 'B09S6VN97V' [SEP] 'Skechers Women's, Relaxed Fit: Arch Fit - Commute Clog Taupe 11 M' [SEP] '\textdollar74.95' [SEP] 'B008QYOGZ2' [SEP] 'OluKai MEA Ola - Men's Supportive Sandal Charcoal/Dkjava - 14' [SEP] '\textdollar100.0' [SEP] 'B07XDRVVYM' [SEP] 'Clarks Women's Un Adorn Sling Sandal' [SEP] '\textdollar34.99 to \textdollar99.95' [SEP] 'B01LXTLLVG' [SEP] 'PUMA V1.08 Tricks Top Trainer Mens Soccer Sneakers/Boots-Grey-5.5' [SEP] '\textdollar55.5' [SEP] 'B09QXF3V3X' [SEP] 'DEUVOUM Summer Trend Mesh Shoes Men's Sports Shoes Solid Color Lace-Up Sneakers Fashion All-Match Walking Shoes Outdoor Hiking Shoes Non-Slip Shock-Absorbing Casual Sports Shoes' [SEP] '\textdollar100.0'

\textbf{Node 8:} 'Back to Search' [SEP] '< Prev' [SEP] 'size' [SEP] '6.5' [SEP] '7' [SEP] '8' [SEP] '8.5' [SEP] '9.5' [SEP] '10' [SEP] '10.5' [SEP] '11' [SEP] '11.5' [SEP] 'color' [SEP] 'black' [SEP] 'blue' [SEP] 'brown' [SEP] 'white' [SEP] 'yellow' [SEP] 'black-b' [SEP] 'blue-a' [SEP] 'yellow-a' [SEP] 'MCICI Mens Loafers Moccasin Driving Shoes Premium Genuine Leather Casual Slip On Flats Fashion Slipper Breathable Big Size' [SEP] 'Price: \textdollar15.99 to \textdollar32.99' [SEP] 'Rating: N.A.' [SEP] 'Description' [SEP] 'Features' [SEP] 'Reviews' [SEP] 'Buy Now'

\textbf{Node 9:} 'Back to Search' [SEP] '< Prev' [SEP] 'size' [SEP] '6.5' [SEP] '7' [SEP] '8' [SEP] '8.5' [SEP] '9.5' [SEP] '10' [SEP] '10.5' [SEP] '11' [SEP] '11.5' [SEP] 'color' [SEP] 'black' [SEP] 'blue' [SEP] 'brown' [SEP] 'white' [SEP] 'yellow' [SEP] 'black-b' [SEP] 'blue-a' [SEP] 'yellow-a' [SEP] 'MCICI Mens Loafers Moccasin Driving Shoes Premium Genuine Leather Casual Slip On Flats Fashion Slipper Breathable Big Size' [SEP] 'Price: \textdollar15.99 to \textdollar32.99' [SEP] 'Rating: N.A.' [SEP] 'Description' [SEP] 'Features' [SEP] 'Reviews' [SEP] 'Buy Now' clicked attributes: [brown]

\textbf{Node 10:} 'Back to Search' [SEP] '< Prev' [SEP] 'size' [SEP] '6.5' [SEP] '7' [SEP] '8' [SEP] '8.5' [SEP] '9.5' [SEP] '10' [SEP] '10.5' [SEP] '11' [SEP] '11.5' [SEP] 'color' [SEP] 'black' [SEP] 'blue' [SEP] 'brown' [SEP] 'white' [SEP] 'yellow' [SEP] 'black-b' [SEP] 'blue-a' [SEP] 'yellow-a' [SEP] 'MCICI Mens Loafers Moccasin Driving Shoes Premium Genuine Leather Casual Slip On Flats Fashion Slipper Breathable Big Size' [SEP] 'Price: \textdollar15.99 to \textdollar32.99' [SEP] 'Rating: N.A.' [SEP] 'Description' [SEP] 'Features' [SEP] 'Reviews' [SEP] 'Buy Now' clicked attributes: [brown, 6.5]

\textbf{Node 11:} Thank you for shopping with us! [SEP] Your code: [SEP] None [SEP] (Paste it in your MTurk interface.) [SEP] Purchased [SEP] asin [SEP] B07HP6LVRS [SEP] options [SEP] {"color": "brown", "size": "6.5"} [SEP] attrs [SEP] None [SEP] category [SEP] None [SEP] query [SEP] None [SEP] product category [SEP] None [SEP] Target [SEP] asin [SEP] options [SEP] attrs [SEP] price upper [SEP] instuction text [SEP] category [SEP] product category [SEP] query [SEP] Goal [SEP] None [SEP] Reward [SEP] Your score (min 0.0, max 1.0) [SEP] 0.7142857142857143 [SEP] Reward Details [SEP] None

\end{tcolorbox}

\caption{The node list of cases in WebShop in Figures~\ref{fig:state_graph_webshop_1},~\ref{fig:state_graph_webshop_2}, and~\ref{fig:state_graph_webshop_3}.}
    \label{fig:node_list_webshop}
\end{figure}

\begin{figure}
\centering

\begin{tcolorbox}[colback=red!5!white, colupper=black, colframe=red!40!black, title=Edge list of examples in WebShop, fontupper=\scriptsize, width=\textwidth]

\textbf{Node 0 $->$ Node 6:} search[228r.brown-a 6.5]

\textbf{Node 0 $->$ Node 1:} search[228r.brown-a, size 6.5, price lower than 70.00]

\textbf{Node 0 $->$ Node 7:} search[228r.brown-a 6.5 rubber outsole men slipper non slip]

\textbf{Node 1 $->$ Node 2:} click[b07s7hdc88]

\textbf{Node 2 $->$ Node 3:} click[228r.brown-a]

\textbf{Node 3 $->$ Node 4:} click[6.5]

\textbf{Node 4 $->$ Node $\bigstar$:} click[buy now]

\textbf{Node 6 $->$ Node 2:} click[b07s7hdc88]

\textbf{Node 7 $->$ Node 8:} click[b07hp6lvrs]

\textbf{Node 8 $->$ Node 9:} click[brown]

\textbf{Node 9 $->$ Node 10:} click[6.5]

\textbf{Node 10 $->$ Node 11:} click[buy now]

\end{tcolorbox}
    \caption{The edge list of cases in WebShop in Figures~\ref{fig:state_graph_webshop_1},~\ref{fig:state_graph_webshop_2}, and~\ref{fig:state_graph_webshop_3}.}
    \label{fig:edge_list_webshop}
\end{figure}

\clearpage
\paragraph{Sokoban.} Repeated board configurations connect movement sequences, including detours that return to a state on the route to a solved board.

\begin{figure}[h!]
    \centering
    \includegraphics[width=\linewidth]{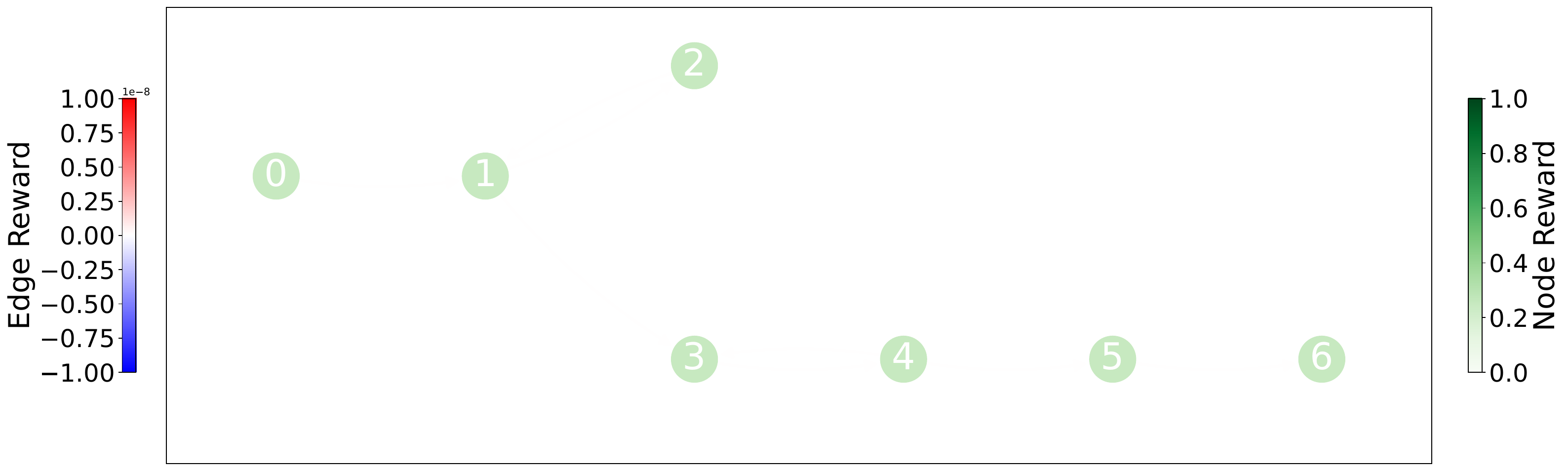}
    \caption{The original state graph constructed from 1 sampled rollout in Sokoban. Darker nodes indicate higher values on the source plot's color scale. Node 0 is the initial state, and $\bigstar$ is the successful terminal state. Node and action descriptions appear in Figures~\ref{fig:node_list_sokoban} and~\ref{fig:edge_list_sokoban}.}
    \label{fig:state_graph_sokoban_1}
\end{figure}

\begin{figure}[h!]
    \centering
    \includegraphics[width=\linewidth]{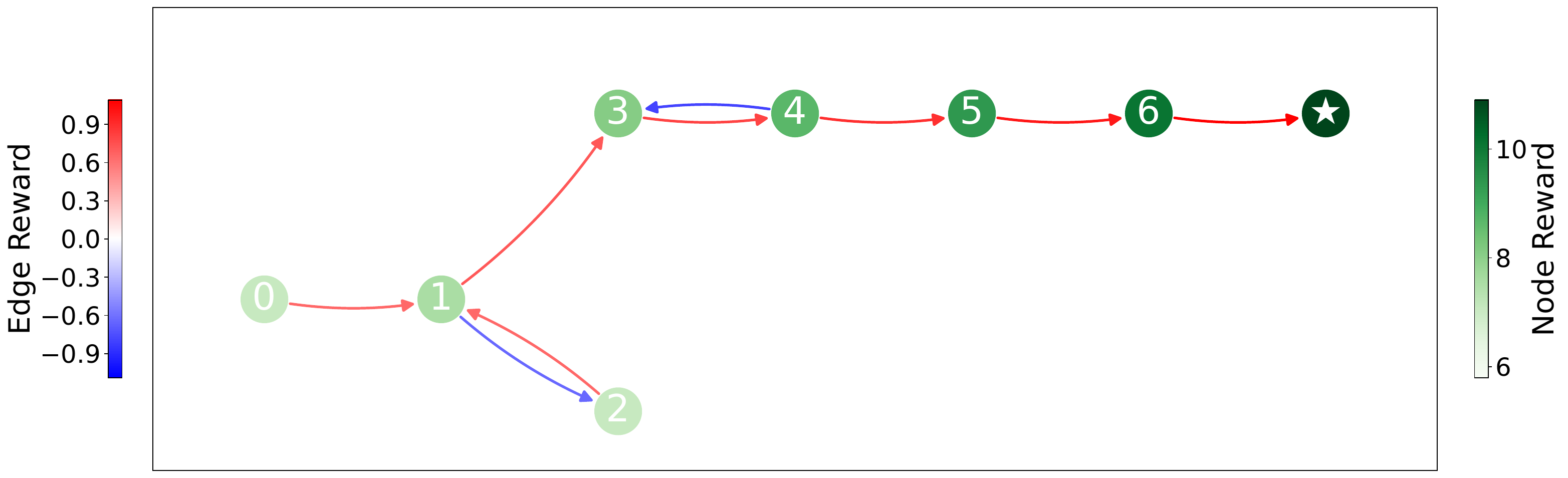}
    \caption{The original state graph constructed from 2 sampled rollouts in Sokoban. Darker nodes indicate higher values on the source plot's color scale. Node 0 is the initial state, and $\bigstar$ is the successful terminal state. Node and action descriptions appear in Figures~\ref{fig:node_list_sokoban} and~\ref{fig:edge_list_sokoban}.}
    \label{fig:state_graph_sokoban_2}
\end{figure}

\begin{figure}[h!]
    \centering
    \includegraphics[width=\linewidth]{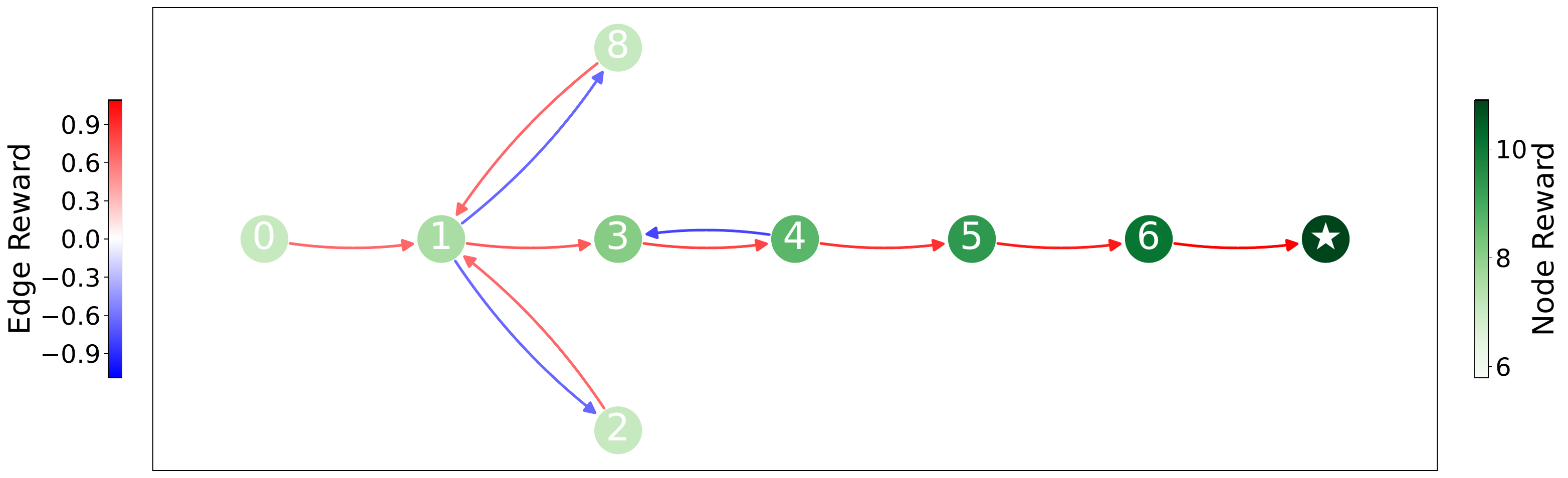}
    \caption{The original state graph constructed from 3 sampled rollouts in Sokoban. Darker nodes indicate higher values on the source plot's color scale. Node 0 is the initial state, and $\bigstar$ is the successful terminal state. Node and action descriptions appear in Figures~\ref{fig:node_list_sokoban} and~\ref{fig:edge_list_sokoban}.}
    \label{fig:state_graph_sokoban_3}
    \label{fig:state_graph_sokoban}
\end{figure}

\begin{figure}
    \centering

\begin{tcolorbox}[colback=blue!10, colupper=black, colframe=blue!40!black, title=Node list of examples in Sokoban, fontupper=\scriptsize, width=\textwidth]

\begin{tabular}{@{} >{\bfseries}p{1.5cm} @{\hspace{8pt}} l @{}}
\textbf{Node 0:}     & \includegraphics[width=1.5cm]{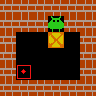} \\[4pt]
\textbf{Node 1:}     & \includegraphics[width=1.5cm]{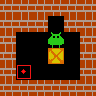} \\[4pt]
\textbf{Node 2:}     & \includegraphics[width=1.5cm]{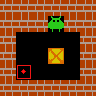} \\[4pt]
\textbf{Node 3:}     & \includegraphics[width=1.5cm]{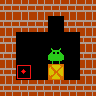} \\[4pt]
\textbf{Node 4:}     & \includegraphics[width=1.5cm]{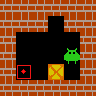} \\[4pt]
\textbf{Node 5:}     & \includegraphics[width=1.5cm]{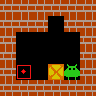} \\[4pt]
\textbf{Node 6: }    & \includegraphics[width=1.5cm]{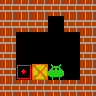} \\[4pt]
\textbf{Node $\bigstar$:} & \includegraphics[width=1.5cm]{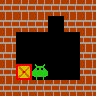} \\[4pt]
\textbf{Node 8:}     & \includegraphics[width=1.5cm]{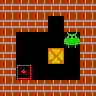} \\
\end{tabular}

\end{tcolorbox}

\caption{The node list of cases in Sokoban in Figures~\ref{fig:state_graph_sokoban_1},~\ref{fig:state_graph_sokoban_2}, and~\ref{fig:state_graph_sokoban_3}.}
    \label{fig:node_list_sokoban}
\end{figure}

\begin{figure}
\centering

\begin{tcolorbox}[colback=red!5!white, colupper=black, colframe=red!40!black, title=Edge list of examples in Sokoban, fontupper=\scriptsize, width=\textwidth]

\textbf{Node 0 $->$ Node 1:} down

\textbf{Node 1 $->$ Node 3:} down

\textbf{Node 1 $->$ Node 2:} up

\textbf{Node 1 $->$ Node 8:} right

\textbf{Node 2 $->$ Node 1:} down

\textbf{Node 3 $->$ Node 4:} right

\textbf{Node 4 $->$ Node 5:} down

\textbf{Node 4 $->$ Node 3:} left

\textbf{Node 5 $->$ Node 6:} left

\textbf{Node 6 $->$ Node $\bigstar$:} left

\textbf{Node 8 $->$ Node 1:} left

\end{tcolorbox}
    \caption{The edge list of cases in Sokoban in Figures~\ref{fig:state_graph_sokoban_1},~\ref{fig:state_graph_sokoban_2}, and~\ref{fig:state_graph_sokoban_3}.}
    \label{fig:edge_list_sokoban}
\end{figure}

\clearpage
\section*{NeurIPS Paper Checklist}

\begin{enumerate}

\item {\bf Claims}
    \item[] Question: Do the main claims made in the abstract and introduction accurately reflect the paper's contributions and scope?
    \item[] Answer: \answerYes{}
    \item[] Justification: The abstract and Sec.~\ref{sec:introduction} describe topology-based action-level reward estimation and its use with trajectory-level supervision. Sections~\ref{sec: graph_modeling} and~\ref{sec: method} define the method and theoretical results, and Sec.~\ref{sec: experiments} evaluates performance, robustness, and computational overhead.
    \item[] Guidelines:
    \begin{itemize}
        \item The answer \answerNA{} means that the abstract and introduction do not include the claims made in the paper.
        \item The abstract and/or introduction should clearly state the claims made, including the contributions made in the paper and important assumptions and limitations. A \answerNo{} or \answerNA{} answer to this question will not be perceived well by the reviewers.
        \item The claims made should match theoretical and experimental results, and reflect how much the results can be expected to generalize to other settings.
        \item It is fine to include aspirational goals as motivation as long as it is clear that these goals are not attained by the paper.
    \end{itemize}

\item {\bf Limitations}
    \item[] Question: Does the paper discuss the limitations of the work performed by the authors?
    \item[] Answer: \answerYes{}
    \item[] Justification: Section~\ref{ssec:theoretical_justification} states the state-identification and graph-coverage conditions for the analysis, while Appendices~\ref{app: exploration} and~\ref{app: aggregation sensitivity} discuss exploration and state-grouping sensitivity. Section~\ref{sec:conclusion} identifies more reliable state identification and broader open-ended evaluation as future directions.
    \item[] Guidelines:
    \begin{itemize}
        \item The answer \answerNA{} means that the paper has no limitation while the answer \answerNo{} means that the paper has limitations, but those are not discussed in the paper.
        \item The authors are encouraged to create a separate ``Limitations'' section in their paper.
        \item The paper should point out any strong assumptions and how robust the results are to violations of these assumptions (e.g., independence assumptions, noiseless settings, model well-specification, asymptotic approximations only holding locally). The authors should reflect on how these assumptions might be violated in practice and what the implications would be.
        \item The authors should reflect on the scope of the claims made, e.g., if the approach was only tested on a few datasets or with a few runs. In general, empirical results often depend on implicit assumptions, which should be articulated.
        \item The authors should reflect on the factors that influence the performance of the approach. For example, a facial recognition algorithm may perform poorly when image resolution is low or images are taken in low lighting. Or a speech-to-text system might not be used reliably to provide closed captions for online lectures because it fails to handle technical jargon.
        \item The authors should discuss the computational efficiency of the proposed algorithms and how they scale with dataset size.
        \item If applicable, the authors should discuss possible limitations of their approach to address problems of privacy and fairness.
        \item While the authors might fear that complete honesty about limitations might be used by reviewers as grounds for rejection, a worse outcome might be that reviewers discover limitations that aren't acknowledged in the paper. The authors should use their best judgment and recognize that individual actions in favor of transparency play an important role in developing norms that preserve the integrity of the community. Reviewers will be specifically instructed to not penalize honesty concerning limitations.
    \end{itemize}

\item {\bf Theory assumptions and proofs}
    \item[] Question: For each theoretical result, does the paper provide the full set of assumptions and a complete (and correct) proof?
    \item[] Answer: \answerYes{}
    \item[] Justification: Section~\ref{ssec:theoretical_justification} states the assumptions and results for graph approximation, state potentials, action-level reward approximation, and trajectory-level reward ordering. Appendix~\ref{app:proofs} provides the formal definitions and proofs, including the cumulative-sampling conditions.
    \item[] Guidelines:
    \begin{itemize}
        \item The answer \answerNA{} means that the paper does not include theoretical results.
        \item All the theorems, formulas, and proofs in the paper should be numbered and cross-referenced.
        \item All assumptions should be clearly stated or referenced in the statement of any theorems.
        \item The proofs can either appear in the main paper or the supplemental material, but if they appear in the supplemental material, the authors are encouraged to provide a short proof sketch to provide intuition.
        \item Inversely, any informal proof provided in the core of the paper should be complemented by formal proofs provided in appendix or supplemental material.
        \item Theorems and Lemmas that the proof relies upon should be properly referenced.
    \end{itemize}

    \item {\bf Experimental result reproducibility}
    \item[] Question: Does the paper fully disclose all the information needed to reproduce the main experimental results of the paper to the extent that it affects the main claims and/or conclusions of the paper (regardless of whether the code and data are provided or not)?
    \item[] Answer: \answerYes{}
    \item[] Justification: Sections~\ref{sec: graph_modeling} and~\ref{sec: method} specify the algorithm; Appendices~\ref{app: state preprocessing}, \ref{app: invalid action filtering}, and~\ref{app:notation-chain} describe state interfaces, retained state transitions, and reward and advantage notation. Appendix~\ref{app: detailed experimental setting} and the supplementary experimental subsections provide training configurations and evaluation protocols, complemented by the public implementation.
    \item[] Guidelines:
    \begin{itemize}
        \item The answer \answerNA{} means that the paper does not include experiments.
        \item If the paper includes experiments, a \answerNo{} answer to this question will not be perceived well by the reviewers: Making the paper reproducible is important, regardless of whether the code and data are provided or not.
        \item If the contribution is a dataset and\slash or model, the authors should describe the steps taken to make their results reproducible or verifiable.
        \item Depending on the contribution, reproducibility can be accomplished in various ways. For example, if the contribution is a novel architecture, describing the architecture fully might suffice, or if the contribution is a specific model and empirical evaluation, it may be necessary to either make it possible for others to replicate the model with the same dataset, or provide access to the model. In general. releasing code and data is often one good way to accomplish this, but reproducibility can also be provided via detailed instructions for how to replicate the results, access to a hosted model (e.g., in the case of a large language model), releasing of a model checkpoint, or other means that are appropriate to the research performed.
        \item While NeurIPS does not require releasing code, the conference does require all submissions to provide some reasonable avenue for reproducibility, which may depend on the nature of the contribution. For example
        \begin{enumerate}
            \item If the contribution is primarily a new algorithm, the paper should make it clear how to reproduce that algorithm.
            \item If the contribution is primarily a new model architecture, the paper should describe the architecture clearly and fully.
            \item If the contribution is a new model (e.g., a large language model), then there should either be a way to access this model for reproducing the results or a way to reproduce the model (e.g., with an open-source dataset or instructions for how to construct the dataset).
            \item We recognize that reproducibility may be tricky in some cases, in which case authors are welcome to describe the particular way they provide for reproducibility. In the case of closed-source models, it may be that access to the model is limited in some way (e.g., to registered users), but it should be possible for other researchers to have some path to reproducing or verifying the results.
        \end{enumerate}
    \end{itemize}

\item {\bf Open access to data and code}
    \item[] Question: Does the paper provide open access to the data and code, with sufficient instructions to faithfully reproduce the main experimental results, as described in supplemental material?
    \item[] Answer: \answerYes{}
    \item[] Justification: Our implementation is publicly available at \url{https://github.com/tmlr-group/RewardFlow}, with installation instructions, environment and data setup, and training scripts for RewardFlow and RL baselines. Appendix~\ref{app: detailed experimental setting} specifies the configurations used in the paper.
    \item[] Guidelines:
    \begin{itemize}
        \item The answer \answerNA{} means that paper does not include experiments requiring code.
        \item Please see the NeurIPS code and data submission guidelines (\url{https://neurips.cc/public/guides/CodeSubmissionPolicy}) for more details.
        \item While we encourage the release of code and data, we understand that this might not be possible, so \answerNo{} is an acceptable answer. Papers cannot be rejected simply for not including code, unless this is central to the contribution (e.g., for a new open-source benchmark).
        \item The instructions should contain the exact command and environment needed to run to reproduce the results. See the NeurIPS code and data submission guidelines (\url{https://neurips.cc/public/guides/CodeSubmissionPolicy}) for more details.
        \item The authors should provide instructions on data access and preparation, including how to access the raw data, preprocessed data, intermediate data, and generated data, etc.
        \item The authors should provide scripts to reproduce all experimental results for the new proposed method and baselines. If only a subset of experiments are reproducible, they should state which ones are omitted from the script and why.
        \item At submission time, to preserve anonymity, the authors should release anonymized versions (if applicable).
        \item Providing as much information as possible in supplemental material (appended to the paper) is recommended, but including URLs to data and code is permitted.
    \end{itemize}

\item {\bf Experimental setting/details}
    \item[] Question: Does the paper specify all the training and test details (e.g., data splits, hyperparameters, how they were chosen, type of optimizer) necessary to understand the results?
    \item[] Answer: \answerYes{}
    \item[] Justification: Section~\ref{sec: experiments} identifies the benchmarks, models, baselines, and metrics, while Appendix~\ref{app: detailed experimental setting} and Tab.~\ref{tab: configuration} provide the training budgets, batch sizes, interaction horizons, sampling temperature, learning rate, and algorithm hyperparameters. Supplementary comparisons specify their differing settings in the corresponding subsections.
    \item[] Guidelines:
    \begin{itemize}
        \item The answer \answerNA{} means that the paper does not include experiments.
        \item The experimental setting should be presented in the core of the paper to a level of detail that is necessary to appreciate the results and make sense of them.
        \item The full details can be provided either with the code, in appendix, or as supplemental material.
    \end{itemize}

\item {\bf Experiment statistical significance}
    \item[] Question: Does the paper report error bars suitably and correctly defined or other appropriate information about the statistical significance of the experiments?
    \item[] Answer: \answerYes{}
    \item[] Justification: Tabs.~\ref{tab:main} and~\ref{tab:ood} report means and standard deviations over three inference seeds at fixed checkpoints, as specified in Appendix~\ref{app: detailed experimental setting}. The preprocessing control additionally specifies the inference seeds and population-standard-deviation calculation in Appendix~\ref{app: matched preprocessing control}.
    \item[] Guidelines:
    \begin{itemize}
        \item The answer \answerNA{} means that the paper does not include experiments.
        \item The authors should answer \answerYes{} if the results are accompanied by error bars, confidence intervals, or statistical significance tests, at least for the experiments that support the main claims of the paper.
        \item The factors of variability that the error bars are capturing should be clearly stated (for example, train/test split, initialization, random drawing of some parameter, or overall run with given experimental conditions).
        \item The method for calculating the error bars should be explained (closed form formula, call to a library function, bootstrap, etc.)
        \item The assumptions made should be given (e.g., Normally distributed errors).
        \item It should be clear whether the error bar is the standard deviation or the standard error of the mean.
        \item It is OK to report 1-sigma error bars, but one should state it. The authors should preferably report a 2-sigma error bar than state that they have a 96\% CI, if the hypothesis of Normality of errors is not verified.
        \item For asymmetric distributions, the authors should be careful not to show in tables or figures symmetric error bars that would yield results that are out of range (e.g., negative error rates).
        \item If error bars are reported in tables or plots, the authors should explain in the text how they were calculated and reference the corresponding figures or tables in the text.
    \end{itemize}

\item {\bf Experiments compute resources}
    \item[] Question: For each experiment, does the paper provide sufficient information on the computer resources (type of compute workers, memory, time of execution) needed to reproduce the experiments?
    \item[] Answer: \answerNo{}
    \item[] Justification: Tab.~\ref{tab: configuration} reports GPU types and counts and training-step budgets; Fig.~\ref{fig:rewardflow_teaser} and Appendix~\ref{app: training efficiency} provide per-step runtime measurements. GPU memory and complete per-run and project-wide compute accounting are not reported.
    \item[] Guidelines:
    \begin{itemize}
        \item The answer \answerNA{} means that the paper does not include experiments.
        \item The paper should indicate the type of compute workers CPU or GPU, internal cluster, or cloud provider, including relevant memory and storage.
        \item The paper should provide the amount of compute required for each of the individual experimental runs as well as estimate the total compute.
        \item The paper should disclose whether the full research project required more compute than the experiments reported in the paper (e.g., preliminary or failed experiments that didn't make it into the paper).
    \end{itemize}

\item {\bf Code of ethics}
    \item[] Question: Does the research conducted in the paper conform, in every respect, with the NeurIPS Code of Ethics \url{https://neurips.cc/public/EthicsGuidelines}?
    \item[] Answer: \answerYes{}
    \item[] Justification: We follow the NeurIPS Code of Ethics and evaluate on public benchmarks using automated metrics, without new human-subject data collection. Potential deployment risks and mitigation measures are discussed in the Broader impacts item below.
    \item[] Guidelines:
    \begin{itemize}
        \item The answer \answerNA{} means that the authors have not reviewed the NeurIPS Code of Ethics.
        \item If the authors answer \answerNo, they should explain the special circumstances that require a deviation from the Code of Ethics.
        \item The authors should make sure to preserve anonymity (e.g., if there is a special consideration due to laws or regulations in their jurisdiction).
    \end{itemize}

\item {\bf Broader impacts}
    \item[] Question: Does the paper discuss both potential positive societal impacts and negative societal impacts of the work performed?
    \item[] Answer: \answerYes{}
    \item[] Justification: RewardFlow can improve the efficiency of training interactive agents, supporting useful automation but also potentially enabling harmful automation or tool misuse when goals, retrieved information, or success criteria are inappropriate. Our evaluation uses benchmark tasks; real-world deployment would require task-specific permission controls, monitoring, and safety evaluation.
    \item[] Guidelines:
    \begin{itemize}
        \item The answer \answerNA{} means that there is no societal impact of the work performed.
        \item If the authors answer \answerNA{} or \answerNo, they should explain why their work has no societal impact or why the paper does not address societal impact.
        \item Examples of negative societal impacts include potential malicious or unintended uses (e.g., disinformation, generating fake profiles, surveillance), fairness considerations (e.g., deployment of technologies that could make decisions that unfairly impact specific groups), privacy considerations, and security considerations.
        \item The conference expects that many papers will be foundational research and not tied to particular applications, let alone deployments. However, if there is a direct path to any negative applications, the authors should point it out. For example, it is legitimate to point out that an improvement in the quality of generative models could be used to generate Deepfakes for disinformation. On the other hand, it is not needed to point out that a generic algorithm for optimizing neural networks could enable people to train models that generate Deepfakes faster.
        \item The authors should consider possible harms that could arise when the technology is being used as intended and functioning correctly, harms that could arise when the technology is being used as intended but gives incorrect results, and harms following from (intentional or unintentional) misuse of the technology.
        \item If there are negative societal impacts, the authors could also discuss possible mitigation strategies (e.g., gated release of models, providing defenses in addition to attacks, mechanisms for monitoring misuse, mechanisms to monitor how a system learns from feedback over time, improving the efficiency and accessibility of ML).
    \end{itemize}

\item {\bf Safeguards}
    \item[] Question: Does the paper describe safeguards that have been put in place for responsible release of data or models that have a high risk for misuse (e.g., pre-trained language models, image generators, or scraped datasets)?
    \item[] Answer: \answerNA{}
    \item[] Justification: We release the RewardFlow implementation for benchmark-based training, rather than a new pretrained model or scraped dataset requiring controlled release. Deployment considerations are addressed in the Broader impacts item above.
    \item[] Guidelines:
    \begin{itemize}
        \item The answer \answerNA{} means that the paper poses no such risks.
        \item Released models that have a high risk for misuse or dual-use should be released with necessary safeguards to allow for controlled use of the model, for example by requiring that users adhere to usage guidelines or restrictions to access the model or implementing safety filters.
        \item Datasets that have been scraped from the Internet could pose safety risks. The authors should describe how they avoided releasing unsafe images.
        \item We recognize that providing effective safeguards is challenging, and many papers do not require this, but we encourage authors to take this into account and make a best faith effort.
    \end{itemize}

\item {\bf Licenses for existing assets}
    \item[] Question: Are the creators or original owners of assets (e.g., code, data, models), used in the paper, properly credited and are the license and terms of use explicitly mentioned and properly respected?
    \item[] Answer: \answerNo{}
    \item[] Justification: The paper credits the benchmarks, models, and training framework used in the experiments, but does not enumerate the license and terms of use for every asset.
    \item[] Guidelines:
    \begin{itemize}
        \item The answer \answerNA{} means that the paper does not use existing assets.
        \item The authors should cite the original paper that produced the code package or dataset.
        \item The authors should state which version of the asset is used and, if possible, include a URL.
        \item The name of the license (e.g., CC-BY 4.0) should be included for each asset.
        \item For scraped data from a particular source (e.g., website), the copyright and terms of service of that source should be provided.
        \item If assets are released, the license, copyright information, and terms of use in the package should be provided. For popular datasets, \url{paperswithcode.com/datasets} has curated licenses for some datasets. Their licensing guide can help determine the license of a dataset.
        \item For existing datasets that are re-packaged, both the original license and the license of the derived asset (if it has changed) should be provided.
        \item If this information is not available online, the authors are encouraged to reach out to the asset's creators.
    \end{itemize}

\item {\bf New assets}
    \item[] Question: Are new assets introduced in the paper well documented and is the documentation provided alongside the assets?
    \item[] Answer: \answerYes{}
    \item[] Justification: The released RewardFlow implementation includes installation instructions, environment setup, dependencies, and example training commands under the Apache-2.0 license at \url{https://github.com/tmlr-group/RewardFlow}.
    \item[] Guidelines:
    \begin{itemize}
        \item The answer \answerNA{} means that the paper does not release new assets.
        \item Researchers should communicate the details of the dataset\slash code\slash model as part of their submissions via structured templates. This includes details about training, license, limitations, etc.
        \item The paper should discuss whether and how consent was obtained from people whose asset is used.
        \item At submission time, remember to anonymize your assets (if applicable). You can either create an anonymized URL or include an anonymized zip file.
    \end{itemize}

\item {\bf Crowdsourcing and research with human subjects}
    \item[] Question: For crowdsourcing experiments and research with human subjects, does the paper include the full text of instructions given to participants and screenshots, if applicable, as well as details about compensation (if any)?
    \item[] Answer: \answerNA{}
    \item[] Justification: This research does not involve crowdsourcing or human-subject experiments; evaluation uses automated benchmark environments and public datasets.
    \item[] Guidelines:
    \begin{itemize}
        \item The answer \answerNA{} means that the paper does not involve crowdsourcing nor research with human subjects.
        \item Including this information in the supplemental material is fine, but if the main contribution of the paper involves human subjects, then as much detail as possible should be included in the main paper.
        \item According to the NeurIPS Code of Ethics, workers involved in data collection, curation, or other labor should be paid at least the minimum wage in the country of the data collector.
    \end{itemize}

\item {\bf Institutional review board (IRB) approvals or equivalent for research with human subjects}
    \item[] Question: Does the paper describe potential risks incurred by study participants, whether such risks were disclosed to the subjects, and whether Institutional Review Board (IRB) approvals (or an equivalent approval/review based on the requirements of your country or institution) were obtained?
    \item[] Answer: \answerNA{}
    \item[] Justification: No human-subject experiments were conducted, so participant risk disclosures and IRB approval are not applicable.
    \item[] Guidelines:
    \begin{itemize}
        \item The answer \answerNA{} means that the paper does not involve crowdsourcing nor research with human subjects.
        \item Depending on the country in which research is conducted, IRB approval (or equivalent) may be required for any human subjects research. If you obtained IRB approval, you should clearly state this in the paper.
        \item We recognize that the procedures for this may vary significantly between institutions and locations, and we expect authors to adhere to the NeurIPS Code of Ethics and the guidelines for their institution.
        \item For initial submissions, do not include any information that would break anonymity (if applicable), such as the institution conducting the review.
    \end{itemize}

\item {\bf Declaration of LLM usage}
    \item[] Question: Does the paper describe the usage of LLMs if it is an important, original, or non-standard component of the core methods in this research? Note that if the LLM is used only for writing, editing, or formatting purposes and does \emph{not} impact the core methodology, scientific rigor, or originality of the research, declaration is not required.
    \item[] Answer: \answerYes{}
    \item[] Justification: Section~\ref{sec: experiments} identifies the LLM policies and PRM baselines used in the experiments. Appendices~\ref{app: learned normalizer} and~\ref{app: math tir supplement} describe LLM-based state-normalizer synthesis and the frozen-model state-normalization variant, respectively.
    \item[] Guidelines:
    \begin{itemize}
        \item The answer \answerNA{} means that the core method development in this research does not involve LLMs as any important, original, or non-standard components.
        \item Please refer to our LLM policy in the NeurIPS handbook for what should or should not be described.
    \end{itemize}

\end{enumerate}

\end{document}